\documentclass[11pt,a4paper]{article}
\usepackage{fullpage}

\usepackage[utf8]{inputenc} 

\usepackage[T1]{fontenc}    
\usepackage{booktabs}       
\usepackage{amsfonts}       
\usepackage{nicefrac}       
\usepackage{microtype}      
\usepackage{xcolor}         
\usepackage{bm}
\usepackage{amsmath}
\usepackage{amssymb}
\usepackage{amsthm}
\usepackage{makecell}
\usepackage{graphicx}
\usepackage{subfigure}
\usepackage{mathtools}
\usepackage{environ}
\usepackage[round]{natbib}

\DeclareUnicodeCharacter{00A0}{~}
\newtheorem{theorem}{Theorem}
\newtheorem{corollary}[theorem]{Corollary}
\newtheorem{lemma}[theorem]{Lemma}

\newtheorem{proposition}[theorem]{Proposition}
\newtheorem{definition}[theorem]{Definition}

\newtheorem{assumption}[theorem]{Assumption}
\newtheorem{problem}[theorem]{Problem}
\newtheorem{remark}{Remark}

\newcommand{\mbf}[1]{\boldsymbol{#1}}
\newcommand{\mres}{\mathbin{\vrule height 1.2ex depth 0pt width
0.13ex\vrule height 0.13ex depth 0pt width 0.9ex}}
\newcommand{\innerp}[2]{\langle #1,#2 \rangle}

\NewEnviron{resizealign}{\sbox0{\let\notag=\relax
    $\begin{matrix}\displaystyle\BODY\end{matrix}$}%
  \sbox1{$(\theequation)$}%
  \sbox2{\parbox{\dimexpr \wd0 + 2\wd1}%
    {\begin{align}\BODY\end{align}}}
  \noindent\resizebox{\columnwidth}{!}{\usebox2}%
}

\newcommand{\br}{\mbf{r}}
\newcommand{\bF}{\mbf{F}}

\newcommand{\bv}{{\mbf{v}}}

\newcommand{\bx}{{\mbf{x}}}
\newcommand{\bX}{\mbf{X}}
\newcommand{\bbX}{\mathbb{X}}
\newcommand{\bY}{\mbf{Y}} 
\newcommand{\by}{\mbf{y}}

\newcommand{\bbY}{\mathbb{Y}}
\newcommand{\bbV}{\mathbb{V}}

\newcommand{\balpha}{\mbf{\alpha}}
\newcommand{\bgamma}{\mbf{\gamma}}

\newcommand{\bZ}{\mbf{Z}}
\newcommand{\bbZ}{\mathbb{Z}}

\newcommand{\bV}{\mbf{V}}

\newcommand{\mK}{{K}}

\newcommand{\mH}{\mathcal{H}_{{K}}}
\newcommand{\mE}{\mathcal{E}}

\newcommand{\rhoL}{\rho_T^L}

\newcommand{\intkernele}{{\intkernel^{}}}

\newcommand{\force}{\mbf{F}}

\newcommand{\forcev}{\force^{\bv}}

\newcommand{\intkernel}{\phi}

\newcommand{\intkernelvar}{\varphi}

\newcommand{\rhsfo}{\mathbf{f}}
\newcommand{\hypspace}{\mathcal{H}}

\newcommand{\phiH}{\intkernel_{\mathcal{H}}}

\newcommand{\E}{\mathbb{E}}
\newcommand{\cov}{\mathrm{Cov}}

\newcommand{\norm}[1]{\left\| #1 \right\|}
\newcommand{\infnorm}[1]{\| #1\|_{\infty}}

\newcommand{\Rhoxnorm}[1]{\| #1\|_{L^2(\rho_{\mbf{X}})}}
\newcommand{\rhotnorm}[1]{\| #1\|_{L^2(\tilde\rho_{T}^L)}}

\newcommand{\Rhoxinnerp}[2]{\langle #1, #2\rangle_{L^2(\rho_{\mbf{X}})}}

\newcommand{\supp}[1]{\text{supp}(#1)}

\newcommand{\real}{\mathbb{R}}
\newcommand{\R}{\real}

\newcommand{\argmin}[1]{\underset{#1}{\operatorname{arg}\operatorname{min}}\;}

\usepackage{ulem}

\begin{document}
\newcommand{\commentout}[1]{}

\title{Learning particle swarming models from data with Gaussian processes}


 \author{Jinchao Feng\thanks{Department of Applied Mathematics and Statistics, Johns Hopkins University. Email: jfeng34@jhu.edu}
 \and
  Charles Kulick
  \thanks{Department of Mathematics,  University of California Santa Barbara. Email:  
charles@math.ucsb.edu}
\and
 Yunxiang Ren \thanks{Department of 
Mathematics, and Physics,  Harvard University. Email: yren@g.harvard.edu }
 \and
 Sui Tang \thanks{Department of Mathematics,  University of California Santa Barbara. Email: suitang@ucsb.edu}
}

\maketitle

\begin{abstract}

{Interacting particle or agent systems that exhibit diverse swarming behaviors are prevalent in science and engineering. Developing effective differential equation models to understand the connection between individual interaction rules and swarming is a fundamental and challenging goal.} In this paper, we study the data-driven discovery of a second-order particle swarming model that describes the evolution of $N$ particles in $\mathbb{R}^d$ under radial interactions. We propose a learning approach that models the latent radial interaction function as Gaussian processes, which can simultaneously fulfill two inference goals: one is the nonparametric inference of {the} interaction  function with  pointwise uncertainty quantification, and the other is the inference of unknown scalar parameters in the non-collective friction forces of the system. We formulate the learning problem as a statistical inverse learning problem and introduce an operator-theoretic framework that provides a detailed analysis of recoverability conditions, establishing that a coercivity condition is sufficient for recoverability. Given data collected from $M$ i.i.d trajectories with independent Gaussian observational noise,   we provide a finite-sample analysis, showing that our posterior mean estimator  converges in a Reproducing Kernel Hilbert Space norm, at an optimal rate in $M$  equal to the one in the classical 1-dimensional Kernel Ridge regression. As a byproduct, we show we can obtain a parametric learning rate in $M$ for the  posterior marginal variance using $L^{\infty}$ norm and that the rate could also involve $N$ and $L$ (the number of observation time instances for each trajectory) depending on the condition number of the inverse problem.  
{Numerical results on systems that exhibit different swarming behaviors demonstrate efficient learning of our approach from scarce noisy trajectory data.}{We provide numerical results on systems exhibiting different swarming behaviors, highlighting the effectiveness of our approach in the scarce, noisy trajectory data regime.}



\end{abstract}

\section{Introduction} \label{sec:intro}

 Swarming  behaviour exhibited by interacting particles is very common, referring to particles of similar size  aggregating together, milling about the same spot, moving en masse, or migrating in some direction. Examples include the aggregation of popular opinion on events, the flocking of birds, the schooling of fish, and the coordinated movement of robots. It is a central subject in various disciplines to reveal the links between swarming behaviors and individual interaction laws.

A common belief in scientific research is that complicated swarming behaviors are the consequences of simple interactions, for instance, the interactions depending on pairwise distances. Inspired by physics, one can write down a second-order ODE system for  $N$  interacting particles $\bx_1,\cdots,\bx_N$ in $\mathbb{R}^d$ as follows: for $i=1,\cdots,N$, the $i$-th equation for $\bx_i$ is 
 \begin{equation}
 \begin{aligned}\label{odes}
m_i\ddot\bx_i(t) = F_i(\bx_i(t), \dot\bx_i(t), \mbf{\alpha}_i)  +\underbrace{ \sum_{i'=1}^N \frac{1}{N} \Big[\intkernel (\norm{\bx_{i'}(t) - \bx_i(t)})(\bx_{i'}(t) - \bx_i(t)) \Big]}_{\text{interaction force}}.
\end{aligned}
 \end{equation}
 The form of the above governing equation is derived from Newton's second law: $m_i$ is the mass of the agent $i$; $\ddot\bx_i$ is the acceleration; $\dot\bx_i$ is the velocity; $F_i$ is a parametric function of the position and velocity, modeling frictions of the particles with the environment; the scalar parameters $\mbf{\alpha}_i$ describe the friction strength; and the interaction force  is the $i$-th component of the derivative of  a potential energy function $\mathcal{U}$ depending on pairwise distances:
 
\begin{align} \label{potentialfunction}
 \mathcal U(\bX(t)):=\sum_{i,i'=1}^{N}\frac{1}{2N} \Phi(\norm{\bx_{i'}(t) - \bx_i(t)}), {\Phi}'(r) = \intkernel(r)r,
 \end{align} 
where $\|\cdot\|$ is the Euclidean norm and $\phi$ is the interaction kernel function $\phi: \mathbb{R}^+ \to \mathbb{R}$.

There are remarkable achievements in the qualitative study of the system \eqref{odes} and its variants. Despite the simple form of the interactions, the asymptotic behavior of the solutions to  \eqref{odes} has proven to reproduce a wide variety of macroscopic collective patterns \citep{d2006self,motsch2014heterophilious,baumann2020laplacian,chuang2007state} which are similar to those observed in practice. System \eqref{odes} and its variants also find various applications in optimization \citep{mei2018mean} and sampling \citep{NIPS2017_17ed8abe} in machine learning.  Despite the impressive progress,
the governing interacting potentials and parametric form of friction force are still far from being \textit{precisely} determined for many systems that arise in biology, ecology, and social science.



The recent rapid advancements in digital imaging and high-resolution lightweight GPS devices have made the trajectory data of interacting particle systems increasingly available. This motivated us to consider the fundamental inverse problem: given the trajectory data generated from \eqref{odes}, can we discover the governing equation? Furthermore, what are effective algorithms with theoretical guarantees?  There are several challenges we face.  The first one is non-linearity. The friction force  often depends  \textit{nonlinearly} with respect to the scalar parameters $\mbf{\alpha}$. Thus, solving the inverse problem  involves nontrivial separations between the friction force and interaction force constrained to dynamics. The second challenge results from little information on the analytic forms of interaction kernels. For example, the Morse type kernels and Lennard-Jones type kernels have very different parametric forms, but they are well-known to reproduce similar collective patterns in particle dynamics. Ideally, we  want to make \textit{minimal} assumptions on their analytic forms and infer them in a nonparametric fashion.  This involves working with large and flexible infinite-dimensional function spaces (e.g, Sobolev spaces).  Thirdly, in practical scenarios, it is possible that only a small amount of data is available, i.e., $M,L$ is small, and the data may have some stochastic effects such as noises.  {In such scenarios, obtaining quantitative predictive uncertainties in estimated interaction kernels is crucial for quantifying the reliability of estimators. This information is useful in designing a data acquisition plan, known as active learning \citep{cohn1996active}, which can be used to optimally enhance our knowledge about the system. In summary, we seek an algorithm that can simultaneously perform inference of $\mbf{\alpha}$, which incorporates the parametric form of $\bF$, and nonparametric inference of $\phi$ while also providing uncertainty quantification of the learned models.}

In machine learning, Gaussian process (GP) based approaches have well-documented merits not only in superior learning of a rich class of nonlinear functions without assumptions on their parametric form in the scarce noisy data regime, but also in quantifying the associated uncertainty. This makes a GP based approach attractive for our learning problem. We  propose a novel approach by modeling $\phi$ as  Gaussian processes and incorporating the GPs into the structure of the whole ODE system \eqref{odes}. The probabilistic framework brought by GPs enables us to perform joint parametric inference of $\mbf{\alpha}$ and nonparametric inference of $\phi$  via the powerful model selection procedure of GPs. The resulting algorithm has superior performance in the scarce noisy data regime and  yields estimators with  uncertainty quantification. We shall show that it is computationally efficient, statistically sound, and effective in benchmark systems.

\subsection{Summary/overview of the proposed algorithm} 

\paragraph{First-order systems} For demonstration purposes, we first summarize
the key ideas of the GP based algorithm for the first order system: for $i=1,\cdots,N$, one has 
\begin{eqnarray}\label{odes1}
\dot \bx_i(t)  = \frac{1}{N} \sum_{i'=1}^N \intkernel \left(\norm{\bx_{i'}(t) - \bx_i(t)}\right)(\bx_{i'}(t) - \bx_i(t)),
\,
\end{eqnarray} 
where $\bx_i(t),\dot{\bx}_i(t) \in \mathbb{R}^d$ are the position and velocity of $i$-th agent at time $t$;  $\phi:\mathbb{R}^{+}\rightarrow \mathbb{R}$ governs the pairwise interactions.  Our observations consist of $\{\bx_i^{(m)}(t_l),\dot \bx_i^{(m)}(t_l)+\mbf{\epsilon}_i^{(m,l)}\}_{i,m,l=1}^{N,M,L}$, with $L$ time instances $0=t_1<t_2\cdots<t_L=T$; $M$ being the number of trajectories, and the initial positions $\{\bx_i^{(m)}(0): 1\leq i\leq N\}$ are drawn i.i.d from an unknown probability measure $\mu_0$ defined on the state space $\mathbb{R}^{dN}$; the Gaussian noise  $\mbf{\epsilon}^{(m,l)}_i \stackrel{i.i.d}{\sim} \mathcal{N}(\bm{0}, \sigma^2 I_{d\times d})$  is independent of $\mu_0$.  The noise model we adapt here can be viewed as a discretization of corresponding Stochastic Differential Equations (SDEs) with homogeneous Brownian noise and is used to model the random effects of the environment on the measurement of velocities or imposing frictions.  We shall see immediately that the Gaussian noise term serves the role of regularization in the proposed GP framework (see \eqref{firstorder:pos} below). In this paper, we are interested in the data regime where $L$ is fixed and $M$ varies. That is to say, we observe data coming from multiple independent trajectories of fixed length.

For ease of presentation, we use a compact notation to represent the ODEs \eqref{odes1}
\begin{align}\label{firstorder:homogeneous}
 \dot\bX(t)&=\rhsfo_{\intkernel}(\bX(t)), 
\end{align} 
where $\bX=[\bx_1^{\top},\cdots,\bx_N^{\top}]^{\top} \in \mathbb{R}^{dN}$ denotes the full state vector, and $\rhsfo_{\intkernel}:\mathbb{R}^{dN}  \rightarrow \mathbb{R}^{dN}$ represents the distance based interactions governed by the interaction kernel $\intkernel$ as in \eqref{odes}. We use the notation $\bbY_{\sigma^2, M}=\{\bbX_{M},  \bbV_{\sigma^2,M}  \}$ to denote the noisy observed trajectory data,  where we introduce two vectors  

\begin{align}\bbX_{M} &= \mathrm{Vec}\big(\{\bX^{(m)}(t_l)\}_{m,l=1}^{M,L}\big) \in \mathbb{R}^{dNML}\\
\bbV_{\sigma^2,M}&= \mathrm{Vec}\big(\{\dot\bX^{(m)}(t_l)+\mbf{\epsilon}^{(m,l)}\}_{m,l=1}^{M,L}\big) \in \mathbb{R}^{dNML}.
\end{align}

Our proposed algorithm consists of three steps. We start by modeling  $\intkernel$ as a Gaussian process \citep{williams2006gaussian}, i.e., consider the prior $\intkernel \sim \mathcal{GP}(0,K_\theta(r,r'))$, with mean zero and covariance kernel function $K_{\theta}$ which depends on hyper-parameters $\theta$. This prior incorporates our prior knowledge about the underlying interaction rule. Secondly, we leverage the powerful training procedure of GP  to  choose a data-driven prior, i.e.,  updating $\theta$ by  maximizing the \textcolor{black}{likelihood} of the observational data. That is equivalent to minimizing the negative log-likelihood function  
\begin{align}
  & \mathrm{argmin}_{\theta,\sigma^2}-\log \mathbb{P}(\bbV_{\sigma^2,M}|\bbX_M,\theta,\sigma^2),
\label{eq:likelihood}
\end{align} 
where we estimate the noise level ($\sigma^2$) of our observations at the same time.
In our setting, $\rhsfo_{\intkernel}$ is linear in $\intkernel$, i.e., $\rhsfo_{\intkernel_1+\intkernel_2}=\rhsfo_{\intkernel_1}+\rhsfo_{\intkernel_2}$, and observational noises are Gaussians which are independent of trajectory data. Therefore, $\mathbb{P}(\bbV_{\sigma^2,M}|\bbX_M,\theta)$ is still Gaussian. One can write its explicit formula and its gradients with respect to hyper-parameters. This allows us to use an efficient variant of the conjugate gradient method to find a minimizer $\hat \theta$. We now denote $\widehat K=K_{\hat\theta}$. 
 
Finally, we use the posterior mean estimator to  predict the value of $\phi$ at a testing location $r^{\ast} \in \mathbb{R}^+$. 
 Leveraging the fact that the joint distribution of $\rhsfo_{\intkernel}$ and $\intkernel$ according to the prior is still Gaussian, we use a conditioning argument to  obtain the following closed-form formula    
 \begin{equation}\label{firstorder:pos}
    \bar{\intkernel}_M(r^{\ast})= \widehat{K}_{\intkernel,\rhsfo_{\intkernel}}(r^\ast,\bbX_M)(\widehat{K}_{\rhsfo_\intkernel}(\bbX_M,\bbX_M) + \sigma^2I)^{-1}\bbV_{\sigma^2,M}, 
\end{equation} 
where the matrices $\widehat{K}_{\intkernel,\rhsfo_{\intkernel}}(r^\ast,\bbX_M) \in \mathbb{R}^{1\times dNML}$ and $\widehat{K}_{\rhsfo_{\intkernel}}(\bbX_M,\bbX_M) \in \mathbb{R}^{dNML\times dNML}$ denote the covariance matrix between $\intkernel (r^{\ast})$ and $\rhsfo_{\intkernel}(\bbX_M)$, and  $\rhsfo_{\intkernel}(\bbX_M)$ and $\rhsfo_{\intkernel}(\bbX_M)$ respectively; $I$ is the identity matrix of compatible size.  In addition, we can also quantify the uncertainty of estimation at $r^*$ by 
\begin{equation}\label{firstorder:var}
\begin{aligned}
    &\mathrm{Var}(    \bar{\intkernel}_M |\bbY_{\sigma^2,M}) = \widehat{K}(r^\ast,r^\ast) - \widehat{K}_{\intkernele,\rhsfo_{\intkernele}}(r^\ast,\bbX_M)(\widehat{K}_{\rhsfo_\intkernele}(\bbX_M,\bbX_M) + \sigma^2I)^{-1}\widehat{K}_{\rhsfo_{\intkernele},\intkernele}(\bbX_M,r^\ast), 
\end{aligned}
\end{equation} 
where $\widehat{K}_{\rhsfo_{\intkernele},\intkernele}(\bbX_M,r^\ast) \in \mathbb{R}^{ dNML \times 1} $ is the transpose of $\widehat{K}_{\intkernel,\rhsfo_{\intkernel}}(r^\ast,\bbX_M) $ (See Corollary \ref{first-order-estimator} for the derivation).

\paragraph{Extension to systems with external forces and second-order systems} The proposed approach can be easily generalized to the variants. For example, consider the first-order system with  unknown external forces
\begin{align}\label{firstorder:force}
 \dot\bX(t)&=\rhsfo_{\intkernel}(\bX(t))+\mbf{F}(\bX(t), \mbf{\alpha}), 
 \end{align} where $\mbf{F}(\bX(t), \mbf{\alpha}):\mathbb{R}^{dN}\rightarrow \mathbb{R}^{dN}$  is a parametric function of \textit{unknown} scalar parameters  $\mbf{\alpha}$ ($\mbf{F}$ can depend nonlinearly on $\mbf{\alpha}$). The parametric form of $\mbf{F}$ encodes the physical constraints of the underlying system. In this case, we can treat both  $\mbf{\alpha}$ and $\sigma^2$ (the noise level) as hyper-parameters and solve 
 \begin{equation*}
      \mathrm{argmin}_{\theta,\mbf{\alpha}, \sigma^2}-\log \mathbb{P}(\bbV_{\sigma^2,M}|\bbX_M,\theta,\mbf{\alpha},\sigma^2).
 \end{equation*}

Even though the above optimization is in general non-convex, our numerical examples show that one can find accurate estimations of $\mbf{\alpha}$  and  $\sigma^2$ (the noise level) using a few iterations ($\approx 50$) from a small set of training data. Finally, we plug these estimates into the model and perform the prediction of $\phi$ using a posterior mean similar to \eqref{firstorder:pos}.  In section \ref{sec:learningapproach}, we provide full technical details of the proposed approach to  general second-order systems with unknown external forces. 

\subsection{Literature review and the novelty of our work}

Many recent works have applied machine learning tools to the discovery of dynamical systems, leading to the formulation of new general principles. The resulting methods can be divided into two main categories: (1) methods based on variants of deep neural networks (DNNs) \citep{long2018pde,raissi2018deep,raissi2018multistep,qin2019data,li2021physics,wang2021learning}; and (2) methods based on kernel methods and Gaussian processes \citep{archambeau2007gaussian,raissi2017machine,heinonen2018learning,yildiz2018learning,mao2019nonlocal,zhao2020state,chen2020gaussian, lee2020coarse,yang2021inference,wang2021explicit,chen2021solving,stepaniants2021learning}. However, methods of type (2) have the potential for considerable
advantages over those of type (1), both in terms of theoretical analysis and numerical
implementation \citep{chen2021solving}. \textcolor{black}{In a nutshell, there is no single method that works best in all settings and the theoretical results are still scarce. It is necessary and requires nontrivial effort to propose and develop a theoretical understanding of learning methodology for a particular type of dynamical system and data regime, as one has to face the unique challenges caused by the underlying physical constraints and the observational data. }

In this paper, we cast the data estimation problem arising in the particle swarm models \eqref{odes} as a statistical inverse learning problem and develop a simple and rigorous kernel/Gaussian process framework for solving it. Below we shall compare our work with the works using Gaussian processes and existing works for particle swarm models.


\paragraph{Novelty of the algorithm} Our method is different from other GP based approaches  introduced to learn ODEs from observations: they either model  $\rhsfo_{\intkernel}:\mathbb{R}^{dN} \rightarrow \mathbb{R}^{dN}$ as a GP, ignoring the interacting structure, and solve a regression problem which would be cursed by the high dimension of the state space of $\bX$, e.g.\citep{heinonen2018learning}, or assume independent GP prior distributions on each  component of $\bX$, and consider  learning a parametric function \citep{mao2019nonlocal,yang2021inference}. 
We  instead model the latent function $\phi$  as a GP  and solve an inverse problem by restricting the GP on a manifold that satisfies the ODE system. In this way, we offer a nonparametric approach, with minimal assumptions on $\phi$, and  build the invariance of the equations under permutation of the agents as well as the radial symmetry of $\intkernel$ into the machine learning model of $\rhsfo_{\intkernel}$, and therefore avoid the curse of dimensionality. The methodology we introduce has the following properties:
\begin{itemize}   \setlength\itemsep{0.1em}

\item  theoretically, the proposed method is  amenable to rigorous
 analysis. We establish a novel operator-theoretical framework, suggesting new research directions to generalize the analysis of
kernel regression methods \citep{williams2006gaussian} and linear inverse problems to interacting particle systems. Under  H\"older type source conditions on $\phi$, we prove the reconstruction error converges at an  upper rate in $M$ (see Theorem \ref{Maintheorem}):  $$\|\bar \phi_M-\phi  \|_{\mH} \lesssim M^{\frac{-\gamma}{2\gamma+2}} $$
where $\mH$ is the underlying Reproducing Kernel Hilbert Space (RKHS) and $0<\gamma\leq \frac{1}{2}$. Based on our best knowledge, there is no prior published  work on the application of GP to particle swarm models \eqref{odes}  in this way and our paper is the first one to obtain the theoretical convergence rates in an RKHS norm. 
We remark this  upper rate in $M$ is statistically optimal for target functions  satisfying certain source conditions
and can not be further improved. One can refer to  \citep{blanchard2018optimal} which established the minimax rates for  classical linear statistical inverse problems; our case corresponds to  $s=0$ (reconstruction error) and $b \rightarrow 1^{+}$ (as we deal with  all Mercer kernels) in their main result.
Using our framework as the bridge, we believe one can obtain more refined rates and bounds in the future. As a byproduct, we also show that a parametric rate in $M$ for the $L^{\infty}$ norm of  the marginal posterior variance can  be obtained, and furthermore, this rate could also involve  the number of particles and the number of observational time instances (see Theorem \ref{marginalpos}).   Last but not least, the reconstruction error bound also yields  bounds on trajectory predictions even if our observation data is \textit{finite} and obtained from discrete time instances. Let $\hat \bX_{[0,T]}$ denote the trajectory generated by the estimator 
$\hat \phi$ over the time interval $[0,T]$, given the same initial condition, then application of  Gr\"onwall's inequality \citep{ames1997inequalities} implies 
$$ \|\hat \bX_{[0,T]}- \bX_{[0,T]}\| \lesssim \|\hat \phi-\phi  \|_{\mH}.  $$ 

Such a trajectory prediction error bound is only available in previous works \citep{lu2019nonparametric,lu2020learning,lu2021learning,miller2020learning} where one has continuous-time observational data. This demonstrates the benefit of using stronger RKHS norms.

\item computationally,  it inherits the complexity of state-of-the-art solvers for  kernel matrices, suggesting new research directions to generalize the work of 
optimal approximate methods for linear regression \citep{quinonero2005unifying,schafer2021sparse}, to the proposed setting
of solving parameter and kernel identification in particle swarm models.
See more discussions in section \ref{computation}. 

\end{itemize}

\paragraph{The existing works on the data-driven discovery of interacting particle systems} Motivated by the broad applications of interacting particle systems in various disciplines, the data-driven discovery of interacting particle systems has become a highly active area of research in recent years. We will first briefly review the relevant works on stochastic interacting particle systems. The most frequently studied approach in recent works is the maximum likelihood approach, which includes parameter estimation \citep{kasonga1990maximum, bishwal2011estimation, gomes2019parameter, chen2021maximum, sharrock2021parameter} and nonparametric estimation of drift in the stochastic McKean-Vlasov equation \citep{genon2022inference, della2022lan, yao2022mean}, as well as radial interaction kernel learning in \citep{lu2021learning}. One can also refer to \citep{messenger2021learning} for the development of the Weak SINDy algorithm that leverages the weak form of the differential equation and sparse parametric regression, with applications to cellular dynamics \citep{messenger2022cells}.

 Our work is on the non-parametric methods for deterministic microscopic interacting particle systems. The theoretical study of {the} least square approach for learning $\phi$ in first-order systems was proposed in \citep{bongini2017inferring}. Later, it has been generalized to second-order systems and heterogeneous systems in \citep{lu2019nonparametric}, with theoretical developments in  \citep{lu2020learning,lu2021learning,miller2020learning}. Compared with previous work that only focused on learning interaction kernels, our proposed method has the following advantages: (1) it can handle  more difficult yet more practical scenarios, i.e., joint inference of  scalar parameters $\mbf{\alpha}$ and $\intkernel$, as both  are often unknown in practical scenarios. Therefore, our method can learn the governing equations \eqref{odes}. (2) It provides uncertainty quantification on estimators.  In the ideal data regime, we provide a rigorous  analysis  and show how it depends on the system parameters. This uncertainty measures the reliability of our estimators, in particular, it can be used to  measure the mismatch between our proposed models with the real-world systems.  (3) It has a powerful training procedure to select a data-driven prior and this overcomes  the drawback of the previous least square algorithms: there is no criterion to select the optimal choice  of function spaces (in terms of both basis and dimensions) for learning so as to minimize the generalization error.   We show in Example \ref{ex: FM} that this yields better performance in trajectory predictions with unseen datasets.

 The theories developed in this paper are related to but significantly depart from  previous work on studying least square estimators \citep{lu2019nonparametric,lu2020learning,lu2021learning,miller2020learning}.
We shall show the posterior mean estimators can be viewed as KRR estimators, whose risk functionals are the \textit{regularized} version of those proposed in previous works by setting the underlying hypothesis space to be an appropriate RKHS space. We go much  further beyond the existing analysis:
 
 \begin{itemize}
     \item  Our new and rigorous operator-theoretic framework formulates this learning problem as a linear statistical inverse  problem. This allows us to refine the  analysis for target functions under source conditions and obtain convergence in the \textit{stronger} Reproducing Kernel Hilbert Space (RKHS) norm. From the perspective of the inverse problem, we analyze the reconstruction error while the previous works analyzed the residual error, where only $L^2$ error bounds were obtained. We remark that the analysis framework presented in \citep{lu2019nonparametric,lu2020learning} can not be extended directly to the RKHS norm and our operator-theoretical framework is significantly different than the previous ones. 

     \item  We study  noisy trajectory data and  provide error bounds on uncertainties that noise brings to the estimation, while the previous works only dealt with  noise-free trajectory data.
 \end{itemize}
 
 To summarize, our contribution can be briefly stated as 
 \begin{itemize}
     \item A novel GP-based algorithm that can solve joint parametric  and nonparametric inference in the particle swarm model.
     
     \item Rigorous analysis on recoverability, quantitative error bounds,  and establishing the statistical optimality of both posterior mean and variance estimators in the framework of linear statistical inverse problems. 
     
     \item Extensive numerical experiments demonstrating the effectiveness and advantages over previous approaches.
     
 \end{itemize}

\subsection{Outline and organization of the paper} Our paper is organized as follows: in section \ref{sec:learningapproach}, we present the algorithm for second-order systems of form \eqref{odes}. In section \ref{sec:learningtheory}, we establish a novel operator-theoretic framework to analyze the performance of the posterior mean estimators and marginal posterior variance. Finally, we test the effectiveness and demonstrate the advantages of the proposed approach on several benchmark systems exhibiting {different types of swarming behaviour}.  

 \subsection{Notation and preliminaries}
\paragraph{Notation} Let $\rho$ be a Borel positive measure on $D$ dimensional Eucliean space $\mathbb{R}^{D}$.  We use $L^2(\mathbb{R}^{D};\rho;\mathbb{R}^{n})$ to denote the set of $L^2(\rho)$-integrable vector-valued functions that map $\mathbb{R}^{D}$ to $\mathbb{R}^{n}$. For a function $\mbf{f}\in L^2(\mathbb{R}^{D};\rho;\mathbb{R}^{n})$, and a vector $\bX=[\bx_1^{\top},\cdots,\bx_m^{\top}]^T \in \mathbb{R}^{mD}$ with $\bx_i \in \mathbb{R}^D$, we use the notation $\mbf{f}(\bX)$ to represent the image of the vector under the function of $\mbf{f}$ componentwisely, namely, $\mbf{f}(\bX)=[\mbf{f}(\bx_1)^{\top},\cdots,\mbf{f}(\bx_m)^{\top}]^{\top}\in \mathbb{R}^{mn}$.
Let $\mathcal{S}_1$ be  a  measurable subset of $\mathbb{R}^{m}$, then the restriction of the measure $\rho$ on $\mathcal{S}_1$, denoted by $\rho\mres \mathcal{S}_1$, is defined as $\rho\mres \mathcal{S}_1(\mathcal{S}_2)=\rho(\mathcal{S}_1 \cap \mathcal{S}_2)$ for any measurable subset $\mathcal{S}_2$ of $\mathbb{R}^{D}$. We used $\mathcal{N}(0,I_{d\times d})$ to denote the standard multivariate Gaussian distribution in $\mathbb{R}^d$.

\paragraph{Preliminaries on GPs (Gaussian Processes) Prior } We say $\phi \sim \mathcal{GP}(u,K)$ to denote our prior on $\phi$. In particular, this means that for any $r \in \mathbb{R}$, the random variable $\phi(r)$ is Gaussian: $\phi(r) \sim \mathcal{N}(u(r), K(r,r))$, where $\mathcal{N}$ denotes the normal  or multivariable normal distributions. Similarly, the joint distribution of $ \begin{bmatrix} \phi(r)\\ \phi(r') \end{bmatrix}$ is multivariate Gaussian: $\begin{bmatrix} \phi(r)\\ \phi(r') \end{bmatrix} \sim \mathcal{N} \left(   \begin{bmatrix} (u(r)\\ (u(r') \end{bmatrix} ,\begin{bmatrix} K ( r, r)& K(r,r')\\ K(r',r)&K(r',r') \end{bmatrix}\right)$. This extends in a natural way to any finite set $(r_1, \dots, r_N) \in \mathbb{R}^N$.

\paragraph{Preliminaries on operator algebras} 


Let $\mathcal{H}_1, \mathcal{H}_2$ be  Hilbert spaces. We use $\langle\cdot,\cdot\rangle_{\mathcal{H}_1}$ to denote the inner product over $\mathcal{H}_1$, and still use $\langle\cdot,\cdot\rangle$ to denote the inner product on the Euclidean space.  We denote by $\mathcal{B}(\mathcal{H}_1,\mathcal{H}_2)$  the set of bounded linear operators mapping $\mathcal{H}_1$ to $\mathcal{H}_2$.
Let $A \in \mathcal{B}(\mathcal{H}_1,\mathcal{H}_2)$, we use $\mathrm{Im}(A)$ to denote its range and $\|A\|$ to denote its operator norm.  $A$ is a compact operator if $A$ maps bounded subsets of $\mathcal{H}_1$ to relatively compact subsets of $\mathcal{H}_2$ (subsets with compact closure in $\mathcal{H}_2$). We use $A^*: \mathcal{H}_2\rightarrow \mathcal{H}_1$ to denote the adjoint operator of $A$, that is, $\forall f \in \mathcal{H}_1$, $g\in\mathcal{H}_2$, $\langle Af,g\rangle_{\mathcal{H}_2}=\langle f,A^*g\rangle_{\mathcal{H}_1}$. $A \in \mathcal{B}(\mathcal{H}_1,\mathcal{H}_1)$ is said to be positive if $A^*=A$ and $\langle Ah,h\rangle_{\mathcal{H}_1} \geq 0$ for all $h \in \mathcal{H}_1$. If $A$ is a real-valued matrix, $A^*=A^{\top}$, the transpose of the matrix. 

If $A \in \mathcal{B}(\mathcal{H}_1,\mathcal{H}_1)$ is a compact positive operator, and $\lambda_n^{}$ represents the $n$-th eigenvalue in decreasing order, then, by the spectral theory of compact operators, 
the eigenfunctions $\{\intkernelvar_n\}_{n=1}^{N}$ (possibly with $N = \infty$) of $A$ form an orthonormal basis for $\mathcal{H}_1$ so that $A^{\tau}\psi=\sum_{n=1}^{N}\lambda_n^{\tau}\langle \intkernelvar_n,\psi\rangle_{\mathcal{H}_1}\intkernelvar_n$ for a real number $\tau$. If $\tau <0$, the domain of $A^{\tau}$ is on the subspace $S_{\tau}$ of $\mathcal{H}_1$ given by 
$S_\tau=\{\sum_{n=1}^{N} a_n\intkernelvar_n|\sum_{n=1}^{N}(a_n\lambda_n^{\tau})^2 \text{ is convergent}\}$. If $h \not\in S_\tau$, then $\|A^{\tau}h\|_{\mathcal{H}_1}=\infty.$

Let $\mathcal{H}$ be a Hilbert space, and $A,B \in \mathcal{B}(\mathcal{H},\mathcal{H})$.  For two self-adjoint operators $A, B$, that is, $A^*=A$ and $B^*=B$,  we say that $A\geq B$ if $A-B$ is a positive operator, i.e. $\langle (A-B)h,h\rangle_{\mathcal{H}} \geq 0$ for all $h \in \mathcal{H}$.   Let $\{e_i\}_{i\in I}$ be an orthonormal basis of $\mathcal{H}$. The trace of $B$ is defined as $\mathrm{Tr}(B)=\sum_{i \in I} \langle Be_i,e_i \rangle_{\mathcal{H}}$. $A$ is a Hilbert Schmidt operator if $\sum_{i \in I} \|Ae_i\|^2_{\mathcal{H}}<\infty$, i.e., $\mathrm{Tr}(A^*A)<\infty$. $\|A\|_{HS}$ denotes its Hilbert–Schmidt norm that satisfies $\|A\|_{HS}^2=\mathrm{Tr}(A^*A)$. $A$ is said to be in the trace class if $\mathrm{Tr}(|A|)<\infty$ for $|A|=\sqrt{A^*A}$. Hilbert Schmidt operators and trace class operators are compact.

 For $d,N,M,L \in \mathbb{N}^+$, let $\mbf{w}=(\mbf{w}_{m,l,i})_{m,l,i=1}^{M,L,N},\mbf{z}=(\mbf{z}_{m,l,i})_{m,l,i=1}^{M,L,N}\in \mathbb{R}^{dNML}$ with  
$\mbf{w}_{m,l,i},\mbf{z}_{m,l,i} \in\mathbb{R}^d$, we define 
\begin{equation}\label{winnerp}
\langle \mbf{w}, \mbf{z}\rangle =\frac{1}{MLN}\sum_{m,l,i=1}^{M,L,N} \langle \mbf{w}_{m,l,i}, \mbf{z}_{m,l,i}\rangle,
\end{equation} where $\langle \mbf{w}_{m,l,i}, \mbf{z}_{m,l,i}\rangle$ is the canonical inner product on $\mathbb{R}^d$ (without normalization).

\paragraph{Preliminaries on RKHSs}  Let $\mathcal{D}$ be a compact subset of $\mathbb{R}^D$. We say that $K: \mathcal{D}\times \mathcal{D} \rightarrow \mathbb{R}$ is a Mercer Kernel if it is continuous, symmetric, and  positive semidefinite, i.e., for any finite set of distinct points $\{x_1,\cdots,x_M\} \subset \mathcal{D},$ the matrix $(K(x_i,x_j))_{i,j=1}^{M}$ is positive semidefinite.  For $x \in \mathbb{R}^D$, $K_{x}$ is a function defined on $\mathcal{D}$ such that $K_{x}(y)=K(x,y)$, $y \in \mathcal{D}$. The Moore–Aronszajn theorem proves that there is an  RKHS $\mH$ associated with  the kernel $K$, which is defined to be the closure of the linear span of the set of functions $\{K_x:x\in \mathcal{D}\}$ with respect to the inner product $\langle \cdot,\cdot\rangle_{\mH}$ satisfying $\langle K_x,K_y\rangle_{\mH}=K(x,y)$.

To ensure the  system \eqref{odes} has a unique solution for arbitrary initial conditions, we assume the true interaction kernel $\intkernel_{\mathrm{true}}$ lies in
 a suitable function space. 
\begin{assumption}\label{assumption1}
  $\phi_{\mathrm{true}}$ lies in a RKHS $\mH$ spanned by a Mercer Kernel $\mK$  defined on $[0,R]\times [0,R]$ for some $R>0$. In particular, $ \kappa^2=\sup_{r\in [0,R]} {\mK}(r,r) <\infty.$
\end{assumption}

Assumption \ref{assumption1} implies that functions in $\mH$ are continuous. Examples of RKHSs include common Sobolev spaces used in the differential equation literature. It is important to note that we only use this assumption in our theoretical analysis. In our numerical section, we use the Mat\'ern  kernel, and our interaction function is not necessarily compactly supported.

For the sake of conciseness, we will drop the subscript and use $\intkernel$ to represent the true interaction kernel. 

\section{GP Based algorithm for second-order systems with external forces}\label{sec:learningapproach}

In this section, we present the  algorithm for the second-order particle swarm model:  for $i=1,\cdots,N$,
\begin{equation}
 \begin{aligned}\label{2ndodes}
m_i\ddot\bx_i(t) = \mbf{F}_i(\bx_i(t), \dot\bx_i(t), \mbf{\alpha}_i) + \sum_{i'=1}^N \frac{1}{N} \Big[\intkernel (\norm{\bx_{i'}(t) - \bx_i(t)})(\bx_{i'}(t) - \bx_i(t)) \Big].
\end{aligned}
 \end{equation}

 We shall use the compact form of a second-order system as follows 

\begin{equation}
 \bZ(t) =\mbf{F}(\bY(t),\mbf{\alpha}) + \rhsfo_\intkernel(\bX(t)).
 \label{eq:2ndOrder_compact}
 \end{equation}
 
We summarize the notation in Table \ref{2ndorder}.

\begin{table}[!htb]
\caption{Notation for second-order systems}
\label{tab:2ndOrder_vecdef} 
\centering
\small{
\small{\begin{tabular}{ l l }
\toprule
Variable                    & Definition \\
\midrule
$\bX \in \R^{dN}$ &  vectorization of position vectors $(\bx_i)_{i=1}^{N}$\\
\midrule
$\bV \in \R^{dN}$ & vectorization of velocity vectors $(\bv_i)_{i=1}^{N} = (\dot{\bx}_i)_{i=1}^{N}$\\
\midrule
$\bY \in \R^{2dN} $ & $\bY = (\bX, \bV)^T$\\
\midrule
$\bZ \in \R^{dN} $ & vectorization of $(m_i\ddot{\bx}_i)_{i=1}^{N}$\\
\midrule
$\br^{\bX}_{ij}, \br^{\bX'}_{ij}\in \R^d$ & $\bx_j-\bx_i$, $\bx'_j-\bx'_i$  \\
\midrule
$r^{\bX}_{ij},r^{\bX'}_{ij}\in \R^+$ & $r^{\bX}_{ij}=\|\br^{\bX}_{ij}\|,  r^{\bX'}_{ij}=\|\br^{\bX'}_{ij}\|$ \\
\midrule
$\mbf{F}(\cdot,\balpha)$ & the non-collective force with parameter $\balpha$\\
\midrule
$\rhsfo_{\intkernele}$ & energy-based interaction force field\\
\midrule
$\mathrm{vec}(\{a_i\}_{i=1}^n) \in \mathbb{R}^n$ & $\mathrm{vec}(\{a_i\}_{i=1}^n) = (a_1,\dots,a_n)^T$, vectorization of the set $\{a_i\}_{i=1}^n$\\
\bottomrule
\end{tabular}}  
}\label{2ndorder}
\end{table}

Note that if $m_i= 0$ for all $i$, then the system \eqref{2ndodes} becomes a first-order system. We are interested in learning $\phi$ and $\mbf{\alpha}$ from data. By modeling  $\phi$ as a GP, the joint distribution of the acceleration field at any two time instances is still Gaussian, as shown in the following lemma:

\begin{lemma}
\label{lemma:prior}

Let $\intkernel$ be a Gaussian process with mean zero and covariance function $K_\theta : [0, R] \times [0, R] \to \mathbb{R}$, i.e., $\intkernel \sim\mathcal{GP}(0,K_\theta(r,r'))$, and $\bZ(t) =\force(\bY(t),\mbf{\alpha}) + \rhsfo_\intkernel(\bX(t))$ as defined in  \eqref{eq:2ndOrder_compact}. Then for any $t,t' \in [0,T]$, we have that,
 \begin{equation}
 \begin{bmatrix}
 \bZ(t)\\ \bZ(t')
 \end{bmatrix}
 \sim \mathcal{N} \left(
 \begin{bmatrix}
 \force(\bY(t),\mbf{\alpha})\\ \force(\bY(t'),\mbf{\alpha})
 \end{bmatrix}
 , K_{\rhsfo_{\intkernel}}(\bX(t),\bX(t'))\right),
\label{eqZdist2}
\end{equation}
where $K_{\rhsfo_{\intkernel}}(\bX(t),\bX(t')))$ is the covariance matrix $
 \cov(\rhsfo_{\intkernel}(\bX(t)),\rhsfo_{\intkernel}(\bX(t'))) $
with the $(i,j)$-th block
\begin{eqnarray}
\label{eqSigma2}
 \cov([\rhsfo_{\intkernel}(\bX)]_i,[\rhsfo_{\intkernel}(\bX')]_j)=\frac{1}{N^2}\sum_{k\neq i,k'\neq j} K_\theta(r^{\bX}_{ik},r^{\bX'}_{jk'})\br^{\bX}_{ik}{\br^{\bX'}_{jk'}}^T.
\end{eqnarray}
\end{lemma}

\begin{proof}
Since  $\intkernel \sim \mathcal{GP}(0,K_\theta(r,r'))$, for any $r,r' \in [0,R]$, we have that,
 \begin{eqnarray}
 \mathbb{E}[\intkernel(r)]&=&0,\\
 \cov[\intkernel(r),\intkernel(r')]&=&K_\theta(r,r').
\end{eqnarray}
Therefore, for any collection of states $\{r_i\}_{i=1}^n \subset [0,R]$, and $\{a_i\}_{i=1}^n, \{b_i\}_{i=1}^n \subset \mathbb{R}$, the linear operator on function values $\mathcal{L}(\{\intkernel(r_i)\}_{i=1}^n) : = (a_i \intkernel(r_i)+b_i)_{i=1}^n$ satisfies
\begin{equation}
    \mathcal{L}(\{\intkernel(r_i)\}_{i=1}^n) \sim \mathcal{N}(\mathrm{vec}(\{b_i\}_{i=1}^n), \Sigma_{\mathcal{L}(\intkernel)}),
\label{eq:phi_cov}
\end{equation}
where $\mathcal{N}$ denotes the Gaussian distribution, $\mathrm{vec}(\{b_i\}_{i=1}^n) \in \mathbb{R}^n$ is the vectorization of $\{b_i\}_{i=1}^n$, and the covariance matrix   $\Sigma_{\mathcal{L}(\intkernel)} = \{a_ia_jK_\theta(r_i,r_j)\}_{i,j=1}^{n} \in \mathbb{R}^{n\times n}$. 

Note that 
\begin{equation}
\label{eq: rhsfo}
    [\rhsfo_{\intkernel}(\bX(t))]_i = \sum_{i'=1}^N \frac{1}{N} \intkernel (\norm{\bx_{i'} - \bx_i})(\bx_{i'} - \bx_i),
\end{equation}
which is linear in $\intkernele$. So for any $t$, $t'$, using \eqref{eq:phi_cov}, we have that,
 \begin{equation}
 \begin{bmatrix}
 \rhsfo_{\intkernel}(\bX(t))\\
 \rhsfo_{\intkernel}(\bX(t'))
 \end{bmatrix}
 \sim \mathcal{N} (\bm{0}, K_{\rhsfo_{\intkernel}}(\bX(t),\bX(t'))),
\end{equation}
where $K_{\rhsfo_{\intkernel}}(\bX(t),\bX(t')))$ is the covariance matrix $
 \cov(\rhsfo_{\intkernel}(\bX(t)),\rhsfo_{\intkernel}(\bX(t')))$
with the $(i,j)$-th block
\begin{eqnarray}
 \cov([\rhsfo_{\intkernel}(\bX)]_i,[\rhsfo_{\intkernel}(\bX')]_j)=\frac{1}{N^2}\sum_{k\neq i,k'\neq j} K_\theta(r^{\bX}_{ik},r^{\bX'}_{jk'})\br^{\bX}_{ik}{\br^{\bX'}_{jk'}}^T.
\end{eqnarray}

Since $
 \bZ(t) =\force(\bY(t),\mbf{\alpha}) + \rhsfo_\intkernel(\bX(t))$, the observation $\bZ$ in the model follows the Gaussian distribution
 \begin{equation}
 \begin{bmatrix}
 \bZ(t)\\ \bZ(t')
 \end{bmatrix}
 \sim \mathcal{N} (
 \begin{bmatrix}
 \force(\bY(t),\mbf{\alpha})\\ \force(\bY(t'),\mbf{\alpha})
 \end{bmatrix}
 , K_{\rhsfo_{\intkernel}}(\bX(t),\bX(t'))).
\end{equation}
This completes the proof.
\end{proof}

\paragraph{Observation data regime} We fix $L$ time stamps with $0 = t_1 < t_2 <\cdots t_L = T$ on $[0, T]$ and obtain the trajectory data $\{\bY(t_l), \bZ_{\sigma^2}(t_l): 1\leq l\leq L\}$ as one training instance, where $\sigma^2$ denotes the unknown variance of additive Gaussian noise specified below. Furthermore, we hold the following two assumptions on training data of $M$ training instances:
\begin{enumerate}
    \item The $M$ initial conditions $\{\bY^{(m)}(0): 1\leq m \leq M\}$ are drawn randomly from a probability measure  $\mbf{\mu}_0=[\mu_0^{\bX},\mu_0^{\dot\bX} ]^T$ on $\R^{2dN}$.
    \item The accelerations $\{\bZ^{(m)}(t_l): 1\leq l\leq N, 1\leq m\leq M\}$ are observed with i.i.d additive Gaussian noise $\mbf{\epsilon}\sim \mathcal{N}(\mbf{0}, \sigma^2 I_{dN\times dN})$, so that the data is denoted by $\bZ^{(m)}_{\sigma^2}(t_l)$.
\end{enumerate}

\begin{remark} The  Gaussian assumptions on observational noise are necessary for us to derive the closed formulas of the estimators. In the actual algorithm, we can approximate the velocity and acceleration from the position data. The resulting estimators will be approximations of the estimators obtained in the ideal data regime. 
\end{remark}

Applying Lemma \ref{lemma:prior}, we  now derive the negative log marginal likelihood for training parameters $\balpha$, $\theta$, and $\sigma$, with given observational data as specified above.

\begin{proposition}
\label{prop: liklihood}
Denote $\bY^{(m,l)}=\bY^{(m)}(t_l)$  and $\bZ^{(m,l)}_{\sigma^2}=\bZ^{(m)}(t_l)+\epsilon^{(m,l)}$ with i.i.d noise $\mbf{\epsilon}^{(m,l)} \sim \mathcal{N}(0, \sigma^2 I_{dN\times dN})$. Suppose we are given the training data set $(\bbY_M,\bbZ_{\sigma^2,M}): = \{(\bY^{(m,l)},$\\$\bZ^{(m,l)}_{\sigma^2})\}_{m,l=1}^{M,L}$ for $M,L \in \mathbb{N}$, such that
\begin{equation}
    \bZ^{(m,l)}_{\sigma^2}  = \force(\bY^{(m,l)},\mbf{\alpha}) + \rhsfo_{\intkernel}(\bX^{(m,l)}) + \mbf{\epsilon}^{(m,l)},
\end{equation}
with $\force(\cdot,\mbf{\alpha})$, $\rhsfo_{\intkernel}$ defined in Table \ref{tab:2ndOrder_vecdef}. Then the {negative log} marginal likelihood of $\bbZ_{\sigma^2,M}$ given $\bbY_M$ and parameters $\balpha$, $\theta$, $\sigma$ satisfies
\begin{eqnarray}
    &-&\log p(\bbZ_{\sigma^2,M}|\bbY_M,\mbf{\alpha},\theta,\sigma^2) \\
    &=& \frac{1}{2} (\bbZ_{\sigma^2,M} - \force(\bbY_M,\mbf{\alpha}))^T(K_{\rhsfo_{\intkernel}}(\bbX_M,\bbX_M;\theta) + \sigma^2I)^{-1}(\bbZ_{\sigma^2,M} - \force(\bbY_M,\mbf{\alpha}))\notag\\ &\ & \qquad +\frac{1}{2}\log|K_{\rhsfo_{\intkernel}}(\bbX_M,\bbX_M;\theta)+\sigma^2I| + \frac{dNML}{2} \log 2\pi.
\label{apd eq:likelihood}
\end{eqnarray} 
where $K_{\rhsfo_\intkernel}(\bbX_M,\bbX_M;\theta)$ denotes the covariance matrix between $\rhsfo_{\intkernel}(\bbX_M)$ and $\rhsfo_{\intkernel}(\bbX_M)$, $I$ is the identity matrix of consistent size. 
\end{proposition}
 
\begin{proof} Using Lemma \ref{lemma:prior}, since $\epsilon^{(m,l)}$ is i.i.d Gaussian noise and is independent of the initial distributions, we have that
\begin{equation}
    \bbZ_{\sigma^2,M} \sim \mathcal{N}(\force(\bbY_M,\mbf{\alpha}), K_{\rhsfo_{\intkernel}}(\bbX_M,\bbX_M;\theta) + \sigma^2 I_{dNML}),
\end{equation}
where the mean vector $\force(\bbY_M,\mbf{\alpha}) = \mathrm{vec}((\force(\bY^{(m,l)},\mbf{\alpha}))_{m,l=1}^{M,L})\in \mathbb{R}^{dNML}$, and the covariance matrix $K_{\rhsfo_{\intkernel}}(\bbX_M,\bbX_M;\theta) = \big(\cov(\rhsfo_{\intkernel}(\bX^{(m,l)}),\rhsfo_{\intkernel}(\bX^{(m',l')})) \big)_{m,m',l,l'=1}^{M,M,L,L} $ can be computed componentwise using \eqref{eqSigma2}. According to the properties of the Gaussian distribution, given $\bbY$ and parameters $\balpha$, $\theta$, $\sigma$, we have the negative log marginal likelihood function as shown in \eqref{apd eq:likelihood}.
\end{proof}
 
 As mentioned earlier, we can apply the gradient-based method \citep{liu1989limited}, to minimize the negative log marginal likelihood and solve for the hyper-parameters $(\mbf{\alpha}, \theta, \sigma)$.  
\begin{proposition} 
\label{prop: derivs}
Let  $\bgamma = (K_{\rhsfo_{\intkernel}}(\bbX_M,\bbX_M;\theta)+ \sigma^2I)^{-1} (\bbZ_{\sigma^2,M} - \force(\bbY_M,\mbf{\alpha}))$. The  partial derivatives of the marginal likelihood w.r.t. the parameters $\mbf{\alpha},\theta,$ and $\sigma$ can be computed as follows:
\begin{align}
\frac{\partial}{\partial \mbf{\alpha}_i} \log p(\bbZ_{\sigma^2,M}|\bbY_M,\mbf{\alpha},\theta,\sigma^2) &= \bgamma^T \frac{\partial \force(\bbY_M,\mbf{\alpha})}{\partial \mbf{\alpha}_i}.
\label{eqalpha}\\
\frac{\partial}{\partial \theta_j} \log p(\bbZ_{\sigma^2,M}|\bbY_M,\mbf{\alpha},\theta,\sigma^2) &= \frac{1}{2} \mathrm{Tr}\left( (\bgamma \bgamma^T - (K_{\rhsfo_{\intkernel}}(\bbX_M,\bbX_M;\theta) + \sigma^2I)^{-1}) \frac{\partial K_{\rhsfo_{\intkernel}}(\bbX_M,\bbX_M;\theta)}{\partial \theta_j}\right).
\label{eqtheta}\\
\frac{\partial}{\partial \sigma} \log p(\bbZ_{\sigma^2,M}|\bbY_M,\mbf{\alpha},\theta,\sigma^2) &=  \mathrm{Tr}\left( (\bgamma \bgamma^T - (K_{\rhsfo_{\intkernel}}(\bbX_M,\bbX_M;\theta) + \sigma^2I)^{-1}) \right)\sigma.
\label{eqsigma}
\end{align}

\end{proposition}


 After optimization of the log likelihood using the computed partial derivatives, we obtain maximum likelihood estimators denoted by  $\hat \theta$, $\hat \balpha$, and $\hat \sigma$.

 Next, we show the detailed derivation of our estimators for the prediction of $\intkernel(r^*)$ at $r^\ast \in [0,R]$ if $\theta$, $\balpha$, and $ \sigma$ are known. 


\begin{theorem}
Suppose  the parameters $\theta$, $\balpha$, and $\sigma$ are known and we are given the training data set $(\bbY_M,\bbZ_{\sigma^2,M}): = \{(\bY^{(m,l)},\bZ^{(m,l)}_{\sigma^2})\}_{m,l=1}^{M,L}$ defined in Proposition \ref{prop: liklihood},  Then for any $r^\ast \in [0,R]$, $\intkernele(r^\ast)$ satisfies
\begin{equation}
    p(\intkernele(r^\ast)|\bbY_M,\bbZ_{\sigma^2,M}) \sim \mathcal{N}(\bar{\intkernel}^{\ast},\mathrm{Var}(\intkernele^\ast)),
\end{equation}
where
\begin{align}
    \bar{\intkernel}^{\ast} &= K_{\intkernele,\rhsfo_\intkernel}(r^\ast,\bbX_M)(K_{\rhsfo_{\intkernel}}(\bbX_M,\bbX_M) + \sigma^2I)^{-1}(\bbZ_{\sigma^2,M} - \force(\bbY_M,\mbf{\alpha})),
\label{eq:estimated phie}    \\
    \mathrm{Var}(\intkernele^\ast) &= K_\theta(r^\ast,r^\ast) - K_{\intkernele,\rhsfo_\intkernel}(r^\ast,\bbX_M)(K_{\rhsfo_{\intkernel}}(\bbX_M,\bbX_M) + \sigma^2I)^{-1}K_{\rhsfo_\intkernel,\intkernele}(\bbX_M,r^\ast).
    \label{eq:estimated var phie}
\end{align}
and $K_{\rhsfo_\intkernel,\intkernel}(\bbX_M, r^*) = K_{\intkernel,\rhsfo_\intkernel}(r^*,\bbX_M)^T$ denotes the covariance matrix between $\rhsfo_{\intkernel}(\bbX_M)$ and $\intkernel(r^*)$. 
\end{theorem}

\begin{proof}
Since $\rhsfo_{\intkernel}(\bbX_M)$ is defined componentwisely by \eqref{eq: rhsfo}, for any $r^\ast \in [0,R]$, we have that
  \begin{equation}
    \begin{bmatrix}
    \rhsfo_{\intkernel}(\bbX_M)\\
    \intkernele(r^\ast)
    \end{bmatrix}
    \sim \mathcal{N} \left( 0,
    \begin{bmatrix}
    K_{\rhsfo_{\intkernel}}(\bbX_M, \bbX_M) & K_{\rhsfo_\intkernel,\intkernele}(\bbX_M, r^\ast)\\
    K_{\intkernel,\rhsfo_\intkernel}(r^\ast, \bbX_M) & K_\theta(r^\ast,r^\ast)
    \end{bmatrix}
    \right),
\end{equation} 
where $K_{\rhsfo_\intkernel}(\bbX_M, \bbX_M)$ is the covariance matrix between $\rhsfo_{\intkernel}(\bbX_M)$ and $\rhsfo_{\intkernel}(\bbX_M)$ as we defined in Proposition \ref{prop: liklihood}, and $K_{\rhsfo_\intkernel,\intkernele}(\bbX_M, r^*) = K_{\intkernele,\rhsfo_\intkernel}(r^*,\bbX_M)^T$ is the covariance matrix between $\rhsfo_{\intkernel}(\bbX_M)$ and $\intkernele(r^*)$, i.e., $K_{\rhsfo_\intkernel,\intkernele}(\bbX_M, r^*) = (\cov(\rhsfo_{\intkernel}(\bX^{(m,l)}),\intkernel(r^\ast)))_{m,l=1}^{M,L}$ and the i-th component of $\cov(\rhsfo_{\intkernel}(\bX^{(m,l)}),\intkernel(r^\ast))$ is computed by
\begin{equation}
    \cov([\rhsfo_{\intkernel}(\bX^{(m,l)})]_i, \intkernel(r^\ast)) = \frac{1}{N} \sum_{k \neq i} K_\theta(r_{ik}^{\bX}, r^\ast) \br_{ik}^{\bX}.
\end{equation}
Note that $\bZ^{(m,l)}_{\sigma^2}  = \force(\bY^{(m,l)},\mbf{\alpha}) + \rhsfo_{\intkernel}(\bX^{(m,l)}) + \epsilon^{(m,l)}$ with i.i.d noise $\mbf{\epsilon}^{(m,l)} \sim \mathcal{N}(0, \sigma^2 I_{dN})$ for all $(m,l)$, so we have
  \begin{equation}
    \begin{bmatrix}
    \bbZ_{\sigma^2,M} - F(\bbY_M,\balpha)\\
    \intkernele(r^\ast)
    \end{bmatrix}
    \sim \mathcal{N} \left( 0,
    \begin{bmatrix}
    K_{\rhsfo_{\intkernel}}(\bbX_M, \bbX_M) + \sigma^2 I_{dNML} & K_{\rhsfo_\intkernel,\intkernele}(\bbX_M, r^\ast)\\
    K_{\intkernel,\rhsfo_\intkernel}(r^\ast, \bbX_M) & K_\theta(r^\ast,r^\ast)
    \end{bmatrix}
    \right),
\end{equation} 
Therefore, based on the properties of the joint Gaussian distribution (see 
Lemma \ref{lemma: conditioning Gaussian}), conditioning on $(\bbY_M, \bbZ_{\sigma^2,M})$, we have that
\begin{equation}
    p(\intkernele(r^\ast)|\bbY_M,\bbZ_{\sigma^2,M},r^\ast) \sim \mathcal{N}(\bar{\intkernel}^{\ast},var(\intkernele^\ast)),
\end{equation}
where $\bar{\intkernel}^{\ast}$ and $\mathrm{Var}(\intkernele^\ast)$ are defined as in \eqref{eq:estimated phie} and \eqref{eq:estimated var phie}.
\end{proof}
We would like point out  that in practice, we use $\hat \theta$, $\hat \balpha$, and $\hat \sigma$ learning from the training set (as mentioned above) in \eqref{eq:estimated phie} and \eqref{eq:estimated var phie} to predict $\phi(r^\ast)$. The kernel used is in fact $\widehat K =K_{\hat \theta}$.

Moreover, if we consider the case when $m_i \equiv 0$, and $F_i(\bx_i(t), \dot\bx_i(t), \mbf{\alpha}) = -\dot\bx_i(t)$ all for $i =1 \dots, N$ in \eqref{2ndodes}, then it becomes the first-order systems \eqref{odes}, and we can derive the following corollary as we have shown in  \eqref{firstorder:pos} and \eqref{firstorder:var}.

\begin{corollary}\label{first-order-estimator} Suppose the parameters $\theta$ and $\sigma$ are known, and we are given the training data set $\bbY_{\sigma^2, M} = \{\bbX_{M},  \bbV_{\sigma^2,M}\}$ from the first-order systems \eqref{odes}, then for any $r^\ast \in [0,R]$, $\intkernele(r^\ast)$ satisfies

\begin{equation}
    p(\intkernele(r^\ast)|\bbY_{\sigma^2, M}) \sim \mathcal{N}(\bar{\intkernel}^{\ast},\mathrm{Var}(\intkernele^\ast)),
\end{equation}

\begin{equation}
    \bar{\intkernel}^{\ast}= {K}_{\intkernel,\rhsfo_{\intkernel}}(r^\ast,\bbX_M)(K_{\rhsfo_\intkernel}(\bbX_M,\bbX_M) + \sigma^2I)^{-1}\bbV_{\sigma^2,M}, 
\end{equation} 
\begin{equation}
\mathrm{Var}(\intkernele^\ast) = {K}(r_\ast,r_\ast) - {K}_{\intkernele,\rhsfo_{\intkernele}}(r^\ast,\bbX_M)({K}_{\rhsfo_\intkernele}(\bbX_M,\bbX_M) + \sigma^2I)^{-1}{K}_{\rhsfo_{\intkernele},\intkernele}(\bbX_M,r^\ast). 
\end{equation} 

\end{corollary}

\section{Error analysis}\label{sec:learningtheory}

Numerical results in section \ref{sec:numericalresult} show that $\mbf{\alpha}$ and $\sigma^2$ were accurately recovered from small amounts of noisy data in the training step.  In this section, we shall focus on the prediction step of our GP-based learning approach: suppose the interaction kernel is the only unknown term in the governing equation, and our goal is to establish a rigorous  quantitative framework which analyzes the error of the posterior mean  \eqref{eq:estimated phie} that approximates $\intkernel$  and the marginal posterior variance when $L$ is fixed and $M\rightarrow \infty$.


\subsection{Preliminaries}

\begin{assumption}\label{assumption2}
   The  distribution of initial conditions $\mu_0$ is compactly supported on $\mathbb{R}^{dN}$. 
\end{assumption}

Recall that  $\mK$ is a Mercer kernel that is defined on $[0,R]\times [0,R]$ and $\mH$ is the RKHS associated to $\mK$.

\begin{lemma}\label{infbound}  
    Suppose $\kappa^2=\sup_{r\in [0,R]} {\mK}(r,r) <\infty$. Then we have that, for any $\intkernelvar \in \mH$, there holds $\|\intkernelvar\|_{\infty}\leq \kappa \|\intkernelvar\|_{\mH}.$
\end{lemma}
\begin{proof}By the reproducing property of $\mK$, we have that $$|\intkernelvar(r)|=|\langle \intkernelvar, \mK_{r}\rangle_{\mH}|\leq \|\intkernelvar\|_{\mH}\|\mK_{r}\|_{\mH}\leq \kappa \|\intkernelvar\|_{\mH}.$$ The conclusion follows.  
\end{proof}

\begin{remark} \label{smoothness}The reproducing property implies that functions in $\mH$ are continuous. In general, the smoothness of the Mercer kernel is closely related to the smoothness of functions in $\mH$. Let $C^{s}([0,R])$ be the space of all functions defined on $[0,R]$
whose partial derivatives up to order $s$ are continuous with the norm $\|f\|_{C^{s}}=\sum_{|\alpha| \leq s}\|D^{\alpha}f\|_{\infty}$, and $C^{s+\epsilon}([0,R])$ denotes the subspace of $C^{s}([0,R])$ of functions with these partial derivatives to be H\"older $\epsilon$ on $[0,R]$. In \citep{smale2007learning}, it has been shown that if 
$K \in C^{2s+\epsilon}([0,R]\times [0,R])$ with $0<\epsilon<2$, the inclusion $\mH \subset C^{2s+\frac{\epsilon}{2}}([0,R]\times [0,R])$ is well-defined, bounded and
$$\|\varphi\|_{C^{s}}\leq 4^s \|K\|_{C^{2s}}^{\frac{1}{2}}\|\varphi\|_{\mH}, \forall \varphi \in \mH.$$
\end{remark}

We introduce an important measure that will be crucial in our theoretical analysis. Note that the observational variables for $\intkernel$ consist of pairwise distances. In \citep{lu2019nonparametric}, a probability measure on $\mathbb{R}^+$ that encodes the information about the dynamics marginalized to pairwise distance was introduced as
\begin{align}
\rho_T^L (dr) &:= \frac{1}{\binom N2}\sum_{l=1}^{L}\bigg[\sum_{i,i'=1, i< i' }^N \E_{\mu_0}[\delta_{r_{ii'}(t_l)}(dr)] , \bigg],
\end{align}
where $\delta$ is the Dirac $\delta$ distribution and $r_{ii'}(t_l):=|\bx_{i}(t_l)-\bx_{i'}(t_l)|$, so that $\E_{\mu_0}[\delta_{r_{ii'}(t)}(dr)]$ is the distribution of the random variable $r_{ii'}(t)$ being the position of particle $i$ at time $t$. Note that it is on the support of $\rho_T^L$ that $\intkernel$ could be learned. The probability measure $\rho_T^L$ can be thought of as an ``occupancy'' measure, in the sense that for any interval $I\subset \mathbb{R}^+$, $\rho_T^{L}(I)$ is the probability of seeing a pair of agents at a distance between them equal to a value in $I$, averaged over the observation time. It measures how much regions of $\mathbb{R}^+$ on average (over the observed times and with respect to the distribution $\mu_0$ of the initial conditions) are explored by the dynamical system.

Without loss of generality, we assume that $\rho_T^L$ is non-degenerate on $[0,R]$\footnote{For example, we can choose $\mu_0:=\mathrm{Unif}[-\frac{R}{2}, \frac{R}{2}]^{dN}$. Then $\mathrm{Supp}(\rho_T^1)=[0,R]$ and $\mathrm{Supp}(\rho_T^1)\subset \mathrm{Supp}(\rho_T^L)$ for $L>1$. }. Due to the structure of the equation, we introduce a positive measure that appears naturally in estimating the error of estimators:
\begin{equation}
\tilde \rho_T^L(r)=r^2\rho_T^L(dr)\mres{[0,R]}, r\in \mathbb{R}^+.
\label{eq:tilderho}
\end{equation}

One can refer to Section 2.1 of \citep{lu2021learning} for the analytical study of measures.

\subsection{Learning as a statistical inverse problem}\label{subsec:operator}

 For easy presentation, we restrict our attention to first-order systems,  which is a special case of second-order systems by assuming the masses of the agents are zero:
\begin{align}\label{psoperator}
 \dot\bX(t)&=\rhsfo_{\intkernel}(\bX(t)).
 \end{align}  Our analysis can be  extended to second-order systems with (known) non-collective force terms with very slight modifications.  For first-order systems, we are given the noisy trajectory data 
\begin{align}
\bV^{(m,l)}_{\sigma^2}:=\rhsfo_{\intkernel}(\bX^{(m,l)})+\mbf{\epsilon}^{(m,l)}, \quad m=1,\cdots,M;\ l=1,\cdots,L, 
\end{align}
where $\bX^{(m,l)}=\bX^{(m)}(t_l)$ and $\mbf{\epsilon}^{(m,l)}$ is the additive Gaussian noise with variance $\sigma^2 I$ independent of $\mu_0$. The trajectory data is indeed of the type needed for the nonparametric regression of $\rhsfo_{\intkernel}$. One can construct an empirical quadratic risk functional 
\begin{align}\label{err1}
\frac{1}{ML}\sum_{m,l=1}^{M,L} \|\bV^{(m,l)}_{\sigma^2}-\mbf{f}(\bX^{(m,l)})\|^2
\end{align} to find the least square estimator of $\rhsfo_{\intkernel}$ over a hypothesis function space.

 In this paper, we are interested in the data regime: $L$ fixed, $M\rightarrow \infty$. In the case of $M= \infty$, the expectation of risk functional \eqref{err1} becomes 
\begin{align}\label{err2}
 \|\rhsfo_{\intkernel}(\bX)-\mbf{f}(\bX)\|_{L^2(\rho_{\bX})}^2
\end{align}  where the probability  measure $\rho_{\bX}$ is defined by
\begin{align}\label{rhox}
\rho_{\bX}:=\mathbb{E}_{\bX(0) \sim \mu_0}\bigg[\frac{1}{L}\sum_{l=1}^{L}\delta_{\bX(t_l)}\bigg];
\end{align}  $\delta$ is the Dirac $\delta$ distribution; $\bX(t_l) \in \mathbb{R}^{dN}$ is the position vector of all agents at time $t_l$. Therefore one can find an unbiased estimator of $\rhsfo_{\intkernel}$  if the regression function space is $L^2(\mathbb{R}^{dN};\mbf{\rho}_{\bX};\mathbb{R}^{dN})$. However, the classical nonparametric regression theory \citep{gyorfi2006distribution} implies that the optimal minimax convergence rate of least square estimators is cursed by the ambient dimension $dN$, which significantly restricts their usability as soon as, say, $dN\geq 10$. It is necessary to exploit the structure of the governing equation encoded in $\rhsfo$ and shift our regression target to $\phi$. This will become an inverse problem as shown below.

\paragraph{Operator representations of the learning problem} Below, we introduce an operator $A$ to represent the learning problem and specify function spaces on which $A$ is a bounded linear operator. 

\begin{proposition}\label{propertyA}
Let $A$ be an operator defined by
\begin{align}
A\intkernelvar=\rhsfo_{\intkernelvar}
\end{align} 
{where $\intkernelvar \in \mH$ and $\rhsfo_{\intkernelvar}$ is given in \eqref{psoperator}} specifying the interaction force. Then $A$ is a linear bounded operator that maps $\mathcal{H}_K$ to $L^2(\mathbb{R}^{dN};\rho_{\bX};\mathbb{R}^{dN})$ with $\|A\|\leq \kappa R$. The adjoint operator 
$A^{*}$ satisfies
\begin{align}\label{adjoint}
A^{*}g=\int_{\bX}\frac{1}{N^2}\sum_{i=1,i'\neq i}^{N}K_{r_{ii'}}\langle \br_{ii'},g_{i}(\bX) \rangle\, d\rho_{\bX},
\end{align} where
$g=[g_1^T,\cdots,g_{N}^T]^T$ with $g_i:\mathbb{R}^{dN} \rightarrow \mathbb{R}^d$. As a consequence, {the operator $B$, defined by}
\begin{align}\label{positive}
B\intkernelvar:=A^{*}A\intkernelvar=\frac{1}{N^3}\int_{\bX}\sum_{i,i',i''}K_{r_{ii'}}\langle \intkernelvar, K_{r_{ii''}}\rangle_{\mathcal{H}_K}\langle \br_{ii'}, \br_{ii''}\rangle \,d\rho_{\bX},
\end{align}
is a trace class operator mapping
$\mathcal{H}_K$ to $\mathcal{H}_K$. In addition, $B$ can be also viewed as a bounded linear operator from $L^2(\tilde{\rho}_T^L)$ to $L^2(\tilde{\rho}_T^L)$. 
\end{proposition}

To prove the Proposition above, we first state the following Lemma:

 \begin{lemma}\label{assumptionmeasure}
If $\mu_0$ is compactly supported, then for $1\leq i,i'\leq N$, we have $\br_{ii'}(\bX)=\bx_{i'}-\bx_i \in L^2(\mathbb{R}^{dN};\rho_{\bX};\mathbb{R}^d)=\{\mbf{f}: \mathbb{R}^{dN} \rightarrow \mathbb{R}^d | \int_{\mathbb{R}^{dN}} \|\mbf{f}(\bX)\|^2 d\rho_{\bX} < \infty \}$. 
\end{lemma}

 The proof of the above lemma is similar to 
 the proof of Proposition 2 in \citep{lu2020learning}. It utilizes the standard dynamical system techniques to show  the trajectory  starting from any $\bX(0)$ sampled from $\mu_0$ is inside a bounded region in $\mathbb{R}^{dN}$ within a finite time interval $[0, T]$.  Consequently, $\br_{ii'}$ is bounded and therefore lies in the $L^2$ space. One may generalize the argument to include distributions with a fast decay, such as the Gaussian distributions. We are now ready to prove Proposition \ref{propertyA}.

\begin{proof}[Proof of Proposition \ref{propertyA}]  Lemma \ref{infbound} implies that $\mH$ can be naturally embedded as a subspace of $L^2(\tilde \rho_T^L)$. Using Lemma \ref{lem1}, we have that 
\begin{align}\label{opinequality} \Rhoxnorm{A\intkernelvar}^2=\Rhoxnorm{\rhsfo_{\intkernelvar}}^2 \leq \frac{N-1}{N} \rhotnorm{\intkernelvar}^2 < R^2\|\intkernelvar\|_{\infty}^2 \leq \kappa^2R^2\|\intkernelvar\|_{\mathcal{H}_K}^2.
\end{align}

 This shows that $A$ is a bounded linear operator mapping $\mathcal{H}_K$ to $L^2(\mathbb{R}^{dN};\rho_{\bX};\mathbb{R}^{dN})$ and $\|A\|\leq \kappa R$.

Next, we prove \eqref{adjoint}. We first show that the map for each $(i,i')$, the map
$$\bX \rightarrow K_{r_{ii'}} \in \mathcal{H}_K$$ is continuous since $\|K_{r_{ii'}}-K_{r'_{ii'}}\|_{\mathcal{H}_K}^2=K(r_{ii'}, r_{ii'})+K(r'_{ii'}, r'_{ii'})-2K(r_{ii'},r'_{ii'})$ for all $r_{ii'} = \|\bx_i-\bx_{i'}\|$, $r'_{ii'} = \|\bx'_i-\bx'_{i'}\|$, and $\bX,\bX' \in \mathbb{R}^{dN}$, and both $\mK$ and $\|\cdot\|$ are continuous. Hence given a function $g \in L^2(\mathbb{R}^{dN};\rho_{\bX};\mathbb{R}^{dN})$, the map 
$$\bX \rightarrow \frac{1}{N^2}\sum_{i=1,i'\neq i}^{N}\mK_{r_{ii'}}\langle \br_{ii'},g_{i}(\bX)\rangle$$ is measurable from $\mathbb{R}^{dN}$ to $\mathcal{H}_K$. Moreover,
$$\|\frac{1}{N^2}\sum_{i=1,i'\neq i}^{N}\mK_{r_{ii'}}\langle \br_{ii'},g_{i}(\bX)\rangle\|_{\mathcal{H}_K} \leq \frac{\kappa}{N^2}\sum_{i=1,i'\neq i}^{N} |\langle \br_{ii'},g_{i}(\bX)\rangle|. $$

By Lemma \ref{assumptionmeasure}, we have that both $ \br_{ii'}, g_{i}(\bX) \in L^2(\mathbb{R}^{dN};\rho_{\bX};\mathbb{R}^d)$. By H\"older's inequality (or Cauchy-Schwartz inequality), $\langle\br_{ii'},g_{i}(\bX)\rangle$ is in $L^1(\mathbb{R}^{dN};\rho_{\bX};\mathbb{R})$, and hence $\frac{1}{N^2}\sum_{i=1,i'\neq i}^{N}\mK_{r_{ii'}}\langle \br_{ii'},g_{i}(\bX)\rangle$ is integrable as a vector-valued map.

Finally, for any
$\psi \in \mH$,
\begin{align*}
\Rhoxinnerp{ A\psi}{g} &=\frac{1}{N} \sum_{i=1}^{N}\int_{\bX} \langle [\rhsfo_{\psi}(\bX)]_i, g_i(\bX) \rangle \,d\rho_{\bX}(\bX)\\&= \frac{1}{N^2} \sum_{i=1}^{N} \sum_{i'=1}^{N}\int_{\bX} \psi(r_{ii'})\langle \br_{ii'}, g_i(\bX)\rangle \, d\rho_{\bX}(\bX)\\
&=\frac{1}{N^2} \sum_{i=1}^{N} \sum_{i'=1}^{N}\int_{\bX} \langle \psi, \mK_{r_{ii'}}\rangle_{\mH}\langle \br_{ii'}, g_i(\bX)\rangle \, d\rho_{\bX}(\bX)\\
&= \langle \psi, \frac{1}{N^2} \sum_{i=1}^{N} \sum_{i'=1}^{N} \int_{\bX} \mK_{r_{ii'}}\langle \br_{ii'}, g_i(\bX)\rangle \,d\rho_{\bX}(\bX)\rangle_{\mH} =\langle \psi, A^*g\rangle_{\mH},
\end{align*} so by the uniqueness of the integral, \eqref{adjoint} holds. Equation \eqref{positive} follows from \eqref{adjoint} by direct calculations and the fact that the integral commutes with the scalar product.

We now prove that $B$ is a trace class operator, i.e. to show that $ \mathrm{Tr}(|B|)<\infty,$ where $|B| =\sqrt{B^*B}$. Since $B$ is positive, we have $|B| =B$. Therefore it is equivalent to show $   \mathrm{Tr}(B) <\infty$. 
\begin{align*}
  \mathrm{Tr}(B)=\mathrm{Tr}(A^*A)&=\sum_n \langle A^*Ae_n,e_n\rangle_{\mH}= \sum_n \langle Ae_n,Ae_n\rangle_{L^2(\rho_{\bX})}  \\ &=\sum_n \|\rhsfo_{e_n}(\bX)\|^2_{L^2(\rho_{\bX})}< \sum_n \|e_n\|^2_{L^2(\tilde{\rho}_T^L)}\\
&\leq R^2 \sum_n\|e_n\|^2_{L^2(\rhoL)}= R^2 \int \langle \mK_{r},\mK_{r}\rangle_{\mH} \,d\rhoL(r) \leq \kappa^2R^2,
\end{align*} where we used  Lemma \ref{lem1} to show the inequality in the second line and
\begin{align*}
\langle \mK_{r},\mK_{r}\rangle_{\mH}=\langle \sum_n \langle \mK_{r},e_n\rangle_{\mH}e_n, \mK_r \rangle_{\mH}=\langle \sum_n \langle \mK_{r},e_n\rangle_{\mH}e_n, \mK_r \rangle_{\mH}=\sum_n e_n^2(r). 
\end{align*}

Lastly, we show $B$ can be viewed as a bounded operator on $L^2(\tilde \rho_{T}^L)$. Assume that $\intkernelvar \in L^2(\tilde\rho_T^L)$, we have the identity that  $B\intkernelvar(r)=\langle \rhsfo_{\intkernelvar}(\bX),  \rhsfo_{K_{r}}(\bX) \rangle_{L^2(\rho_{\bX})}$. We obtain that 
\begin{align}
|B\intkernelvar(r)| &\leq \|\rhsfo_{\intkernelvar}(\bX)\|_{L^2(\rho_{\bX})} \|\rhsfo_{K_r}(\bX)\|_{L^2(\rho_{\bX})} \notag\\
&\leq \frac{N-1}{N} \|\intkernelvar\|_{{L^2(\tilde\rho_T^L)}}\|K_r\|_{{L^2(\tilde\rho_T^L)}}\notag\\
&\leq  \frac{N-1}{N} \|\intkernelvar\|_{{L^2(\tilde\rho_T^L)}}R\|K_r\|_{{L^2(\rho_T^L)}}\notag\\
&\leq \frac{N-1}{N} \|\intkernelvar\|_{{L^2(\tilde\rho_T^L)}}R\|K_r\|_{\infty}\notag\\
&\leq \frac{N-1}{N} \|\intkernelvar\|_{{L^2(\tilde\rho_T^L)}}\kappa R\|K_r\|_{\mH}\notag\\
&\leq \frac{N-1}{N} \|\intkernelvar\|_{{L^2(\tilde\rho_T^L)}}\kappa^2 R.
\label{opinequality2}
\end{align}  where the last inequality follows from $\|K_r\|_{\mH}=\sqrt{K(r,r)} \leq \kappa$. 

As a result, $B\intkernelvar\in L^2(\tilde\rho_T^L)$, and $B$ can be viewed as a bounded linear operator from $L^2(\tilde\rho_T^L)$ to $L^2(\tilde\rho_T^L)$ with $\|B\|_{L^2(\tilde\rho_T^L)}\leq \kappa^2 R^2$.
\end{proof}

When $M=\infty$,   our learning problem is then equivalent  to solving a linear operator equation
\begin{align}\label{lineareq}
A\intkernelvar =\rhsfo_{\intkernel}.
\end{align} and it is, therefore, a linear inverse problem over possibly infinite dimensional space. In particular, when  $L=1$, our learning problem becomes a standard statistical inverse problem with a random and noisy observation scheme \citep{blanchard2018optimal}.

 In the case of finite data, i.e., $M< \infty$, we introduce an empirical version of $A$, denoted by $A_M$, see also in Table \ref{tab:empirical}, to represent the learning problem. 

\begin{table}
\caption{Notations in the empirical version}
\vspace{1em}
\label{tab:empirical} 
\centering
{
\small{\begin{tabular}{ l l }
\toprule
Notation          & Definition \\

\midrule
$\bbX_M \in \mathbb{R}^{dNML}$ & {vectorization} of $\{\bX^{(m,l)})\}_{m,l=1}^{M,L}$\\
\midrule
$A_M: \mH \rightarrow \mathbb{R}^{dNML} $ & $A_{M}\intkernelvar=\rhsfo_{\intkernelvar}(\bbX_M)$\\
\midrule
$A_{M}^* : \mathbb{R}^{dNML} \rightarrow \mH$ & adjoint operator of $A_M$\\
\midrule
$B_M: \mH \rightarrow \mH$ &  $B_M = A_M^*A_M$\\
\midrule
$\mE^{\lambda,M}(\cdot)$ &  the regularized empirical risk functional  (see \eqref{empiricalerrappendix})\\
\midrule
$\phi_{\mH}^{\lambda,M}$ & minimizer of $\mE^{\lambda,M}(\cdot)$ in $\mH$\\
\bottomrule
\end{tabular}}  
}
\end{table}

\begin{proposition}\label{eoperator} Given the empirical noisy trajectory data with the vectorized notation $\bbY_{\sigma^2,M}=\{\bbX_M,\bbV_{\sigma^2,M}\}$, we define the sampling operator
$A_{M}: \mH \rightarrow \mathbb{R}^{dNML}$ by
\begin{align}\label{finiterank}
A_{M}\intkernelvar=\rhsfo_{\intkernelvar}(\bbX_M):=\mathrm{Vec}(\{\rhsfo_{\intkernelvar}(\bX^{(m,l)})\}_{m,l=1}^{M,L}),
\end{align} where $\mathbb{R}^{dNML}$ is equipped with the inner product defined in \eqref{winnerp}. The adjoint operator $A_{M}^*$ is a finite rank operator. For any $\mathbb{W}$ in $\mathbb{R}^{dNML}$, let $\mathbb{W}_{m,l,i} \in \mathbb{R}^d$ denote the $i$-th component of {the} $(m,l)$-th block of $\mathbb{W}$.Then we have 
$$A^*_{M}\mathbb{W}=\frac{1}{LM}\sum_{l,m=1}^{L,M}\sum_{i=1,i'\neq i}^{N}\frac{1}{N^2}K_{r_{ii'}^{(m,l)}} \langle \br_{ii'}^{(m,l)}, \mathbb{W}_{m,l,i}\rangle.$$
For any function $\intkernelvar \in \mathcal{H}_K$, we have that 
$$B_{M}\intkernelvar:=A^*_{M}A_{M}\intkernelvar=\frac{1}{LM}\sum_{l,m=1}^{L,M}\left(\sum_{i=1,i', i'' \neq i}^{N}\frac{1}{N^3}K_{r_{ii'}^{(m,l)}} \langle \intkernelvar ,K_{r_{ii''}^{(m,l)}} \rangle_{\mH} \langle \br_{ii'}^{(m,l)},\br_{ii''}^{(m,l)}\rangle\right).$$

\end{proposition}

\begin{proof}[Proof of Proposition \ref{eoperator}] The formula of $A_M^*$ can be derived by using the identity $\langle A_M \intkernelvar, \mbf{w}\rangle=\langle \intkernelvar, A_M^*\mbf{w} \rangle_{\mH}$. The direct calculations of the composition of two operators yields $B_M$. 
\end{proof}

\subsection{Recoverablity: a coercivity condition}

Since $\phi \in \mH$, $\phi$ is always a solution to the linear operator equation \eqref{lineareq}. However, this inverse problem may still be ill-posed. {This} happens when the solution is not unique or does not depend continuously on $\rhsfo_{\intkernele}$.

The uniqueness of the solution is not obvious. As explained above, we only observe an additive functional of {$\phi$} induced by the structure of the governing equation: 
\begin{equation}
    \dot {\bx_i}(t) = \sum_{i'=1}^{N} {\phi} (\|\bx_{i'}(t)-\bx_{i}(t)\|)(\bx_{i'}(t)-\bx_{i}(t)), \quad i=1,\cdots,N.
\end{equation}
Given $\bX(t)$ and $\dot\bX(t)$, one may attempt to solve the values of $ \{{\phi}(\|\bx_{i'}(t)-\bx_{i}(t)\|)\}_{i, i'=1}^{N,N}$ from the constraints imposed by ODEs. However, we have $dN$ equations but with {only} $\binom{N}{2}$ unknowns. In our numerical examples, $d=1$ or 2, {so} as long as $N>5$, the linear system is underdetermined. Even in the overdetermined case, there {are} no guarantees on the exact recovery of {$\phi_{}$} on the pairwise distances.

\paragraph{A coercivity condition} To ensure the well-posedness, we require {$\phi_{}$} to be the unique solution to \eqref{lineareq}. So $A$ has to be injective. Now we introduce a sufficient condition to guarantee the injectivity of the operator $A$. {Using} Lemma \ref{infbound}, $\mH$ can be naturally embedded as a subspace of $L^2([0,R];{\tilde\rho_T^L}\mres[0,R];\mathbb{R})$.

\begin{definition}[Coercivity condition]
 We say that the system \eqref{firstorder:homogeneous} satisfies the \textbf{coercivity condition} on $\mH$,  if $ \forall \intkernelvar \in \mH$, there exists $ c_{\mH}>0$ such that
\begin{align}\label{coercivity}
\|A\intkernelvar\|^2_{L^2(\rho_{\bX})}=\|\rhsfo_{\intkernelvar}\|^2_{L^2(\rho_{\bX})}\geq c_{\mH}\|\intkernelvar\|^2_{L^2( \tilde\rho_T^L)}. \end{align}
We choose the largest $c_{\mH}$ that satisfies \eqref{coercivity} and refer to it as the \textit{coercivity constant}. 
\end{definition}

Then if $A\intkernelvar=0$ for $\intkernelvar \in \mH$, we conclude that $\intkernelvar=0$ everywhere on $[0,R]$ due to non-degeneracy of $\rho_T^L$ on $[0,R]$ and the function $\intkernelvar$ is continuous. Therefore, $A$ is injective. Below, we show the coercivity condition links our learning problem with {a} 1-dimensional kernel ridge regression problem in Problem \ref{krr}: they are equivalent inverse problems.

\begin{problem}\label{krr}
 Consider learning $\intkernel_{} \in \mH$ from i.i.d noisy samples: 
\begin{align}\label{1dregression}
y_m=\intkernele_{}(r_m)+\epsilon_m, r_m\sim \tilde\rho_T^L, \epsilon_m\sim \mathcal{N}(0,\sigma^2), m=1,\cdots,M. 
\end{align}
\end{problem}

One may want to find an estimator in the RKHS spanned by a Mercer kernel $K$. In the limiting case $M=\infty$, this learning problem can also be treated as an inverse problem, where one looks for the solution of the linear operator equation 
\begin{align}\label{ips}
J_{\tilde \rho_T^L}\varphi =\intkernel_{}
\end{align}
and the operator $J_{\tilde \rho_T^L}:\mH \rightarrow L^2([0,R];{\tilde\rho_T^L};\mathbb{R})$ is called the canonical inclusion map 
$$J_{\tilde \rho_T^L}(\varphi)(r)=\langle \varphi, K_{r}\rangle_{\mH}.$$

In general, this inverse problem is ill-posed, as $\phi_{}$ may not be in the closure of $ \mathrm{Im}(J_{\tilde \rho_T^L})$. One then looks for a solution to the least square problem
\begin{align}\label{inverseproblem}
\argmin{\varphi \in \mH}\|\varphi-\intkernel_{}\|^2_{L^2(\tilde \rho_T^L)}.
\end{align}
  
Let $P$ denote the projection  mapping $L^2([0,R];{\tilde\rho_T^L};\mathbb{R})$ onto the closure of $\mathrm{Im}({J_{\tilde \rho_T^L}})$. According to the theory of inverse problems, a sufficient condition for the existence and uniqueness of a minimal norm solution to the problem \eqref{inverseproblem} is $P(\phi) \in \mathrm{Im}(J_{\tilde \rho_T^L}).$ In fact, such a solution is exactly the Moore-Penrose (or generalized) solution to \eqref{inverseproblem}, denoted by $\phi_{\mH}^{+}$, satisfying 
\begin{align}\label{psedoinverse}
J_{\tilde \rho_T^L}^* J_{\tilde \rho_T^L}\phi_{\mH}^{+}=J_{\tilde \rho_T^L}^* \intkernel_{},
\end{align}
where the adjoint operator $J_{\tilde \rho_T^L}^*$ is an integral operator with respect to the kernel $\mK$, i.e., for $\intkernelvar \in L^2([0,R];{\tilde\rho_T^L};\mathbb{R})$ and $r \in [0,R]$,
 $$(J_{\tilde \rho_T^L}^*\intkernelvar)(r)=\int_{0}^{R} \mK(r,r')\intkernelvar(r')d\tilde\rho_T^L(r').$$
 
We know from the classical KRR learning theory \citep{smale2007learning} that  $J_{\tilde \rho_T^L}^*J_{\tilde \rho_T^L}: \mH\rightarrow \mH$ is a compact and positive operator, which ensured the well-posedness of \eqref{ips}. Below, we show  $A^*A$  is  equivalent  to $J_{\tilde \rho_T^L}^*J_{\tilde \rho_T^L}$ as an operator: their eigenvalues have the same asymptotic behaviours.

\begin{proposition}\label{wellpos}
Let $\lambda _{k}^{{\downarrow }}(A^{*}A)$ and $ \lambda _{k}^{{\downarrow }}(J_{\tilde \rho_T^L}^*J_{\tilde \rho_T^L})$   denote the $k$-th eigenvalue of $A^*A$  and $J_{\tilde \rho_T^L}^*J_{\tilde \rho_T^L}$  {respectively} in decreasing order. {If} the coercivity condition \eqref{coercivity} holds, then    $$ c_{\mH} \lambda _{k}^{{\downarrow }}(J_{\tilde \rho_T^L}^*J_{\tilde \rho_T^L}) \leq \lambda _{k}^{{\downarrow }}(A^*A) \leq  \lambda _{k}^{{\downarrow }}(J_{\tilde \rho_T^L}^*J_{\tilde \rho_T^L}). $$
\end{proposition}
 
Therefore, the coercivity condition bridges the study of our inverse problem with \eqref{ips}. To prove Proposition \ref{wellpos}, we first show the following Proposition and Theorem.
 
\begin{proposition}\label{equivalenceeig}
The coercivity condition \eqref{coercivity} implies that \begin{align}\label{positivity}
 c_{\mathcal{H}_K}J_{\tilde \rho_T^L}^*J_{\tilde \rho_T^L} \leq A^*A\leq  J_{\tilde \rho_T^L}^*J_{\tilde \rho_T^L}.
 \end{align}
\end{proposition}

\begin{proof} It suffices to show that, for any $\varphi \in \mH$, we have that
\begin{align*}
 c_{\mathcal{H}_K} \langle J_{\tilde \rho_T^L}^*J_{\tilde \rho_T^L} \varphi, \varphi\rangle_{\mH} \leq \langle A^*A \varphi, \varphi\rangle_{\mH} \leq  \langle J_{\tilde \rho_T^L}^*J_{\tilde \rho_T^L} \varphi. \varphi\rangle_{\mH}
\end{align*} The above inequality follows from the coercivity condition \eqref{coercivity} and the identities 
\begin{align}
\langle J_{\tilde \rho_T^L}^*J_{\tilde \rho_T^L}\intkernelvar, \intkernelvar \rangle_{\mH}=\|\intkernelvar\|^2_{L^2( \tilde\rho_T^L)} \text{ and } 
 \langle A^*A\intkernelvar, \intkernelvar \rangle_{\mH}=\|A\intkernelvar\|^2_{L^2(\rho_{\bX})}.
\end{align}
\end{proof}

\begin{theorem}[Courant–Fischer–Weyl min-max principle, see \citep{bhatia2013matrix}] \label{minmax}Let $U$ be a compact, self-adjoint, positive operator on a Hilbert space $\mathcal{H}$, whose eigenvalues are listed in decreasing order $\lambda_1\geq\lambda_2\cdots$. Let $S_k\subset \mathcal{H}$ be a $k$-dimensional subspace. Then:
\begin{align}\max _{{S_{k}}}\min _{{x\in S_{k},\|x\|=1}}\langle Ux,x\rangle_{\mathcal{H}}&=\lambda _{k}^{{\downarrow }}(U),\label{maximin}\\\min _{{S_{{k-1}}}}\max _{{x\in S_{{k-1}}^{{\perp }},\|x\|=1}}\langle Ux,x\rangle_{\mathcal{H}} &=\lambda _{k}^{{\downarrow }}(U).\end{align} 
\end{theorem}

Now we are ready to present the proof. 

\begin{proof}[Proof of Proposition \ref{wellpos}] Let $\lambda _{k}^{{\downarrow }}(A^{*}A)$ denote the $k$th eigenvalue of $A^*A$ in decreasing order. First, we recall that for two positive operators $A_1$ and $A_2$ on $\mathcal{H}$, $A_1\leq A_2$ means that $ \langle A_1x, x\rangle_{\mathcal{H}} \leq \langle A_2x, x\rangle_{\mathcal{H}} $ for all $x \in \mathcal{H}$. The inequality \eqref{positivity} in Proposition \ref{equivalenceeig} together with the equality  \eqref{maximin} yield that
   \begin{align*}
  c_{\mH} \lambda _{k}^{{\downarrow }}(J_{\tilde \rho_T^L}^*J_{\tilde \rho_T^L}) = \max_{S_k} \min_{x \in S_k, \|x\|=1}
  \langle c_{\mH} J_{\tilde \rho_T^L}^*J_{\tilde \rho_T^L}x, x\rangle_{\mH} & \leq \lambda _{k}^{{\downarrow }}(A^{*}A) =\max_{S_k} \min_{x \in S_k, \|x\|=1}\langle A^{*}Ax, x\rangle_{\mH} \\ &\leq \max_{S_k} \min_{x \in S_k, \|x\|=1}
   \langle  J_{\tilde \rho_T^L}^*J_{\tilde \rho_T^L}x, x\rangle_{\mathcal{H}}\\&=\lambda _{k}^{{\downarrow }}(J_{\tilde \rho_T^L}^*J_{\tilde \rho_T^L}).
   \end{align*}
    
Therefore, 
   $$ c_{\mH} \lambda _{k}^{{\downarrow }}(J_{\tilde \rho_T^L}^*J_{\tilde \rho_T^L}) \leq \lambda _{k}^{{\downarrow }}(A^*A) \leq  \lambda _{k}^{{\downarrow }}(J_{\tilde \rho_T^L}^*J_{\tilde \rho_T^L}). $$

\end{proof}

Since the coercivity condition implies the injectivity of  $A$,  $\phi_{}$ is the unique generalized solution to the equation 
$$A^*A\phi^+=A^* \rhsfo_{\phi_{}}.$$

However, this generalized solution may not depend continuously on the datum $\rhsfo_{\intkernel}$, so that finding $\intkernel$ is again an ill-posed problem when the datum $\rhsfo_{\intkernel}$ is contaminated by noise. In the literature of the inverse problem, one way to overcome this issue is to introduce the Tikhonov regularization technique and consider a risk functional with a possible regularization term determined by $\lambda
\geq 0$:
 \begin{align}\label{exp}
\mE^{\lambda,\infty}{(\intkernelvar)}:&=\Rhoxnorm{A\intkernelvar- \rhsfo_{\intkernel_{}}}^2+\lambda \|\intkernelvar\|_{\mH}^2.
\end{align}

When the data is finite and noisy, it is impossible to achieve the exact recovery of $\phi_{}$. Similar to the case of infinite data, one may consider solving 
\begin{align}\label{empiricalerrappendix}
\phi_{\mH}^{\lambda,M}: & =\argmin{\intkernelvar\in \mH}\mE^{\lambda,M}(\intkernelvar)
\\
\mE^{\lambda,M}(\intkernelvar):&=\|A_M\intkernelvar-\bbV_{\sigma^2,M}\|^2+\lambda \|\intkernelvar\|_{\mH}^2 \label{krrest}
\end{align}

\eqref{krrest} provides an  alternative approach to learn $\phi$ from data. When $A$ is the identity, \eqref{empiricalerrappendix} is called the KRR estimator. In classical nonparametric regression problems such as Problem \ref{krr}, one can also model $\phi$ as a GP with a suitable prior and then approximate $\phi$ by the posterior mean estimator. There is a well-known connection between the posterior mean estimator of {the} GP approach with the KRR estimator. In our paper, we shall generalize this classical fact to our setting: we show that the posterior mean estimator \eqref{firstorder:pos} with a suitable prior coincides with $\phi_{\mH}^{\lambda,M}$.

The connection between our posterior mean estimator with $\phi_{\mH}^{\lambda,M}$ allows us to use the operator algebra framework to derive quantitative error analysis for our approach, since  $\phi_{\mH}^{\lambda,M}$ admits an operator representation.

\begin{proposition} Consider the expected risk $\mE^{\lambda,\infty}(\cdot)$ in \eqref{exp} as well as its empirical version $\mE^{\lambda,M}(\cdot)$ in \eqref{empiricalerrappendix}. Let $\phi_{\mH}^{\lambda,\infty}$ and $\phi_{\mH}^{\lambda,M}$ be their minimizers respectively.

\begin{itemize}
\item Case $\lambda=0$. The minimizer $\intkernel_{\mH}^{0,\infty}$ always exists and satisfies 
$$B\intkernel_{\mH}^{0,\infty}=A^{*}\rhsfo_{\intkernele_{}}, B=A^*A.$$ 
\item Case $\lambda>0$. Then $\phi_{\mH}^{\lambda,\infty}$ and $\phi_{\mH}^{\lambda,M}$ are unique minimizers and they are given by 
\begin{align}
\intkernel_{\mH}^{\lambda,\infty}:&=(B+\lambda)^{-1}A^{*}\rhsfo_{\intkernele_{}}.\\
\intkernel_{\mH}^{\lambda,M}:&=(B_M+\lambda)^{-1}A_{M}^{*}\bbV_{\sigma^2,M}, B_M=A_M^*A_M.\label{em10} 
\end{align}
\end{itemize}
\end{proposition}

The proof of {this} Proposition follows from solving the norm equation of the corresponding regularized least squares.

Below, we derive a Representer theorem for $\phi_{\mH}^{\lambda, M}$, which is key to establish the connection. It shows that $\phi_{\mH}^{\lambda, M}$ is, in fact, a linear combination of the kernel function $K_r$, where $r$ ranges in pairwise distances of agents coming from the observational data. 

\begin{theorem}[Representer theorem]\label{representerthm}
If $\lambda>0$, the minimizer of the regularized empirical risk functional $\mE^{\lambda,M}(\cdot)$ (see \eqref{krrest})
has the form
\begin{equation}
 \phi_{\mH}^{\lambda,M}= \sum_{r \in r_{\bbX_M}} \hat c_r K_{r},
\end{equation}
where $r_{\bbX_M} \in \mathbb{R}^{MLN^2}$ is the set which contains all the pair{wise} distances in $\bbX_{M}$, i.e. 
\begin{equation}
r_{\bbX_M} = \begin{bmatrix}r_{11}^{(1,1)},\dots,r_{1N}^{(1,1)},\dots,r_{N1}^{(1,1)},\dots,r_{NN}^{(1,1)}, \dots, r_{11}^{(M,L)},\dots,r_{1N}^{(M,L)},\dots, r_{N1}^{(M,L)},\dots,r_{NN}^{(M,L)}\end{bmatrix}^T.
\end{equation}
Moreover, denote by $\mathbf{\hat{c}}$ the vectorization of $\hat c_r$ for $r$ in $r_{\bbX_M}$, we have that
\begin{equation}\label{solution1}
 \mathbf{\hat c}= \frac{1}{N}\br_{\bbX_M}^T \cdot (K_{\rhsfo_{\intkernele}}(\bbX_M,\bbX_M) + \lambda N MLI)^{-1}\bbV_{\sigma^2,M},
\end{equation}
where the block-diagonal matrix $\br_{\bbX_M} = \mathrm{diag}(\br_{\bX^{(m,l)}}) \in \mathbb{R}^{MLdN \times MLN^2}$ and $\br_{\bX^{(m,l)}} \in \mathbb{R}^{dN\times N^2}$ defined by
\begin{equation}
\br_{\bX^{(m,l)}} = 
 \begin{bmatrix}
  \br_{11}^{(m,l)}, \dots, \br_{1N}^{(m,l)} & \mbf{0} & \cdots & \mbf{0}\\
  \mbf{0} & \br_{21}^{(m,l)}, \dots, \br_{2N}^{(m,l)} & \cdots & \mbf{0}\\
  \vdots & \vdots & \ddots & \vdots\\
 \mbf{ 0} & \mbf{0} & \cdots & \br_{N1}^{(m,l)}, \dots, \br_{NN}^{(m,l)}
 \end{bmatrix}
 \ .
\end{equation}

\end{theorem}

\begin{proof}[Proof of Theorem \ref{representerthm}] The proof is based on the operator representations of minimizers which allow us to use tools from the spectral theory of operator algebra. 

Let $\mathcal{H}_{K,M}$ be the subspace of $\mH$ spanned by the set of functions $\{K_{r}: r\in r_{\bbX_M}\}$. By Proposition \ref{eoperator}, we know that $B_M(\mathcal{H}_{K,M}) \subset \mathcal{H}_{K,M}$. Since $B_M$ is self-adjoint and compact, by {the} spectral theory of self-adjoint compact operators (see \citep{blank2008hilbert}), $\mathcal{H}_{K,M}$ is also an invariant subspace for the operator $(B_M+\lambda I)^{-1}$. Then by \eqref{em10}, there exists a vector 
$\hat c$ such that 
\begin{equation}\label{brep}
  \phiH^{\lambda,M} = \sum_{r \in r_{\bbX_M}} \hat c_r K_{r}.
\end{equation}


Then, multiplying $(B_M+\lambda I)$ on both sides of \eqref{em10} and plugging  \eqref{brep} into the identity, we can obtain 
\begin{align}\label{mtxid}
\big(\br_{\bbX_M}^T\br_{\bbX_M}K(r_{\bbX_M}, {r_{\bbX_M}})+\lambda N^3ML I \big)\mathbf{\hat c} =N\br_{\bbX_M}^T \bbV_{\sigma^2,M},
\end{align}
where we used the matrix representation of $(B_M+\lambda I)$ with respect to the spanning set $\{K_{r}: r\in r_{\bbX_M}\}$.

Recall that we have $K(r_{\bbX_M}, {r_{\bbX_M}})=(K(r_{ij}, r_{i'j'}))_{r_{ij},r_{i'j'} \in r_{\bbX_M}}$ and $K_{\rhsfo_\intkernele}(\bbX_M,\bbX_M)=$\\$\mathrm{Cov}(\rhsfo_{\intkernele}(\bbX_M), 
\rhsfo_{\intkernele}(\bbX_M))$. By the identity 
\begin{align}\label{id}
\br_{\bbX_M}K(r_{\bbX_M}, {r_{\bbX_M}}) \br_{\bbX_M}^T=N^2K_{\rhsfo_\intkernele}(\bbX_M,\bbX_M)
\end{align}
and the fact that the matrix in the left hand side of \eqref{mtxid} is invertible, one can verify that 
\begin{equation}
  \mathbf{\hat c} = \frac{1}{N}\br_{\bbX_M}^T \cdot (K_{\rhsfo_{\intkernele}}(\bbX_M,\bbX_M) +\lambda NML I)^{-1}\bbV_{\sigma^2,M}
\end{equation} 
is the solution. 
\end{proof}

 \subsection{Operator representations of posterior mean estimators and marginal variances}

Leveraging Theorem \ref{representerthm}, we derive operator representations for posterior mean estimators and marginal variances. Note that this result does not require the coercivity condition.

\begin{theorem}\label{maingp}  Suppose $\intkernel_{} \sim \mathcal{GP} (0, \tilde K)$ with $ \tilde K=\frac{\sigma^2 K} {MNL\lambda}$ for some $\lambda>0$.
\begin{itemize}
\item {The} posterior mean estimator $\bar\phi_M$ in \eqref{firstorder:pos} has an operator representation
\begin{align}
\bar\phi_M :&=(A_M^*A_M+\lambda)^{-1}A_{M}^{*}\bbV_{\sigma^2,M}\label{em} 
\end{align}   

\item {The} marginal posterior variance  \eqref{firstorder:var} can be written as 
{\small
\begin{align}
\mathrm{Var}(\phi_M(r_*) | \bbY_{\sigma^2,M})=\frac{\sigma^2}{ML\lambda N}[K_{r_*}(r_*)-K_{r_*}^{\lambda,M}(r_*)],
\end{align} 
}
where the function $K_{r_*}(\cdot):= K(r_*,\cdot)$, and $
K_{r_*}^{\lambda,M}:=(A_M^*A_M+\lambda)^{-1}A_{M}^{*}\rhsfo_{K_{r_*}}(\bbX_M)$.
\end{itemize}
\end{theorem}

\begin{proof} [Proof of Theorem \ref{maingp}]
Let $\tilde K=\frac{\sigma^2 K}{MNL\lambda}$. 

\begin{itemize}
\item Since $\intkernele \sim \mathcal{GP}(0,\tilde K)$, the posterior mean estimator \eqref{firstorder:pos} becomes 
\begin{align*}
\bar{\intkernel}_M(r^{\ast})&= \tilde K_{\intkernel,\rhsfo_{\intkernel}}(r^\ast,\bbX_M)(\tilde K_{\rhsfo_\intkernel}(\bbX_M,\bbX_M) + \sigma^2I)^{-1}\bbV_{\sigma^2,M}\\
&=\frac{1}{N}\tilde K_{r_{\bbX_M}^T}(r^{\ast})\br_{\bbX_M}^T(\tilde K_{\rhsfo_\intkernele}(\bbX_M,\bbX_M) + \sigma^2I)^{-1}\bbV_{\sigma^2,M}\\
&=\frac{1}{N} K_{r_{\bbX_M}^T}(r^{\ast})\br_{\bbX_M}^T( K_{\rhsfo_\intkernele}(\bbX_M,\bbX_M) + NML\lambda I)^{-1}\bbV_{\sigma^2,M}\\
&= K_{\intkernele,\rhsfo_{\intkernele}}(r^\ast,\bbX_M)( K_{\rhsfo_\intkernele}(\bbX_M,\bbX_M) + NML\lambda I)^{-1}\bbV_{\sigma^2,M}\\
&= \sum_{r \in r_{\bbX_M}} \hat c_r K_{r}, 
\end{align*} where $\hat c$ is defined in \eqref{solution1} and we use the identity 
$ K_{\intkernele,\rhsfo_{\intkernele}}(r^\ast,\bbX_M)=\frac{1}{N} K_{r_{\bbX_M}^T}(r^{\ast})\br_{\bbX_M}^T$ (also for $\tilde K$) in the proof.

\item  We replace the regression target $\intkernel$ with the function $K_{r^*}$, and use the same analysis to develop a representer theorem similar to \eqref{representerthm} for the empirical regularized risk functional \eqref{krrest}. Specifically, we have that 
$$K_{r_*}^{\lambda,M}(\cdot)=K_{\intkernele,\rhsfo_{\intkernele}}(\cdot,\bbX_M)(K_{\rhsfo_\intkernele}(\bbX_M,\bbX_M) + ML\lambda NI)^{-1}K_{\rhsfo_{\intkernele},\intkernele}(\bbX_M,r^\ast).$$
Since $\intkernele \sim \mathcal{GP}(0,\tilde K)$,  the marginal posterior  variance  in $\eqref{firstorder:var}$ will then become
{\small
\begin{align*}
&\mathrm{Var}(\phi_M(r_*) | \bbY_{\sigma^2,M})\\
&=\tilde K_{r^\ast} (r^\ast) - \tilde K_{\intkernele,\rhsfo_{\intkernele}}(r^\ast,\bbX_M)(\tilde K_{\rhsfo_\intkernele}(\bbX_M,\bbX_M) + \sigma^2I)^{-1}\tilde K_{\rhsfo_{\intkernele},\intkernele}(\bbX_M,r^\ast)\\
&= \frac{\sigma^2}{ML\lambda N} \bigg(K_{r^\ast}(r^\ast)- K_{\intkernele,\rhsfo_{\intkernele}}(r^\ast,\bbX_M)(\frac{\sigma^2}{ML\lambda N}  K_{\rhsfo_\intkernele}(\bbX_M,\bbX_M)+\sigma^2 I)^{-1}\frac{\sigma^2}{ML\lambda N} K_{\rhsfo_{\intkernele},\intkernele}(\bbX_M,r^\ast)\bigg)\\
&=\frac{\sigma^2}{ML\lambda N}[K(r^\ast,r^\ast)- K_{r_*}^{\lambda,M}(r^\ast)]
\end{align*}
}

\end{itemize}

\end{proof}

{Applying Theorem \ref{maingp}, the analysis of reconstruction error for our posterior mean estimator and marginal posterior variance can be performed equivalently on $\phi_{\mH}^{\lambda, M}$ and $K_{r_*}(r_*)-K_{r_*}^{\lambda,M}(r_*)$. We shall next develop estimates of the error $\phi_{\mH}^{\lambda, M}-\phi_{}$ and the marginal posterior variance can be analyzed similarly by replacing $\phi_{}$ with $K_{r^*}$.  }

\subsection{Finite sample analysis of errors}

In nonparametric regression and inverse problems, one of {the} fundamental problems to address is the convergence of estimators obtained from finite data.  Without constraints on the target function, we can always find a solution with convergence guarantees  but the convergence rates can be arbitrarily slow. This is called the ``no free lunch theorem'' in learning theory \citep{devroye2013probabilistic} and a similar kind of phenomenon occurs in the regularization of ill-posed inverse problems \citep{engle1996regularization}.


\paragraph{Source condition} 

In solving Problem \ref{krr}, a standard way to impose restrictions on target functions is to describe a prior on $\phi$ determined by smoothness conditions. One typically assumes \citep{smale2007learning}
\begin{align}\label{cond1}
\phi \in \tilde{\Omega}_{\alpha,S}=\{ & \varphi \in L^2(\tilde \rho_T^L):  \varphi = (J_{\tilde \rho_T^L} J_{\tilde \rho_T^L}^*)^{\alpha} \psi, \|\psi\|_{L^2(\tilde \rho_T^L)} \leq S\},
\end{align}
and $\alpha$ typically ranges from 0 to 1. When $\alpha=0$, this condition is equivalent to $\intkernel \in  L^2(\tilde \rho_T^L)$; as $\alpha$ increases, $\phi$ becomes more smooth. For example, we have $\phi \in \mH$ as long as $\alpha \geq \frac{1}{2}$. The value $S$ measures the complexity of $\varphi$. A function $\varphi$ with many oscillations will force $S$ to be large.

As noted \citep{de2005learning,caponnetto2005fast}, \eqref{cond1} corresponds to what is called source conditions in the context of solving linear  inverse problems \eqref{ips}. When $\phi \in \tilde{\Omega}_{\alpha,S}$ with $ \alpha >\frac{1}{2}$,  it is equivalent to consider the H\"older type source condition
$$\phi \in \Omega_{\gamma,S}=\{\varphi \in \mH: \varphi = (J_{\tilde \rho_T^L}^* J_{\tilde \rho_T^L})^{\gamma}\psi, \|\psi\|_{\mH}^2\leq S \},$$
where $\gamma=\alpha-\frac{1}{2}$.  One can refer to section 2.3 of \citep{bauer2007regularization} for more details.  Following inverse problem literature and the connection between $(J_{\tilde \rho_T^L}^* J_{\tilde \rho_T^L})$ and $A^*A$ established in Proposition \ref{wellpos}, we shall consider {the} standard H\"older type source condition for our inverse problem.
\begin{assumption} \label{sourceconditionappendix}
 $\phi \in \mathrm{Im}(B^{\gamma})$ with $\gamma \in (0,\frac{1}{2}]$, where $B=A^*A$.
\end{assumption}

\subsubsection{Decomposition of the reconstruction error}\label{appendix:finitesample}

Using the operator representations, we perform the  decomposition of  the reconstruction error  as the sum of  two types of errors: \begin{align*} 
& \phi^{\lambda,M}_{\mH}-\intkernele =\phi^{\lambda,M}_{\mH}-\phi^{\lambda,\infty}_{\mH}+\phi^{\lambda,\infty}_{\mH}-\phi\\
&= \underbrace{ (B_M+\lambda)^{-1}A_{M}^{*}\bbV_{\sigma^2,M}-(B+\lambda)^{-1}A^{*}\rhsfo_{\intkernel}}_{\text{Sample error}}+\underbrace{(B+\lambda)^{-1}A^{*}\rhsfo_{\intkernel}-\intkernel}_{\text{Approximation error}}.
\end{align*}

The sample error comes from two sources: one is from the randomness in the initial conditions of observed trajectories, and the second one is  the randomness in the noise term. We further decouple the sample error into the noise part and noise-free part; we have that: 
\begin{align*}
\phi^{\lambda,M}_{\mH}-\phi^{\lambda,\infty}_{\mH}&=(B_M+\lambda)^{-1}A_{M}^{*}\bbV_{\sigma^2,M}-\phi^{\lambda,\infty}_{\mH}\\
&=\underbrace{(B_M+\lambda)^{-1}B_M\intkernele-(B+\lambda)^{-1}B\intkernele}_{\tilde\intkernele_{\mH}^{\lambda,M}-\phi^{\lambda,\infty}_{\mH}}+\underbrace{(B_M+\lambda)^{-1}A_{M}^{*}\mathbb{W}_M}_{\text{Noise term}}
\end{align*} where $\tilde\intkernel_{\mH}^{\lambda,M}$ is the empirical minimizer of $\mE^{\lambda,M}(\cdot)$ for noise-free observations and $\mathbb{W}_M$ denotes the noise vector.

One of our key technical contributions is to provide  a detailed analysis of the operators $A$ ($B=A^*A$) and $A_{M}$ ($B_M=A_M^*A_M$), and prove the concentration inequalities for operators.

\paragraph{Analysis of sample error $\|\phi_{\mH}^{\lambda,M}-\phi_{\mH}^{\lambda,\infty}\|_{\mH}$.} We first provide non-asymptotic analysis of the sample error $$\|(B_M+\lambda)^{-1}B_M\intkernelvar-(B+\lambda)^{-1}B\intkernelvar\|_{\mH}$$ for any $\intkernelvar\in \mH$. Then we apply the bound to $\intkernele$ and obtain an error estimate of $\|\tilde\intkernele_{\mH}^{\lambda,M}-\phi^{\lambda,M}_{\mH}\|_{\mH}.$ We shall need the following lemmas. 

\begin{lemma}\label{varinequality} For any bounded function $\intkernelvar \in L^2(\tilde\rho_T^L)$ and any positive integer $M$, we have that
\begin{align}
\|B_{M}\intkernelvar\|_{\mH}&\leq \kappa R^2\|\intkernelvar\|_{\infty}, a.s.,\\
\E\|B_{M}\intkernelvar\|_{\mH}^2&\leq \|\intkernelvar\|_{L^2(\tilde\rho_{T}^L)}^2\kappa^2R^2.
\end{align}
\end{lemma}


\begin{lemma} \label{decomp}For a bounded function $\intkernelvar \in L^2(\tilde \rho_T^L)$ and $0< \delta <1$, with probability at least $1-\delta$, there holds 
\begin{align}
\|B_M\intkernelvar-B\intkernelvar\|_{\mH} \leq \frac{4\kappa R^2 \|\intkernelvar\|_{\infty} \log(2/\delta)}{M}+ \kappa R \rhotnorm{\intkernelvar}
\sqrt{ \frac{2\log(2/\delta)}{M}}.
\end{align}
\end{lemma}

\begin{lemma}\label{sampleerror}For a bounded function $\intkernelvar \in L^2(\tilde\rho_T^L)$ and $0< \delta <1$, with probability at least $1-\delta$, there holds 
$$\|(B_M+\lambda)^{-1}B_M\intkernelvar-(B+\lambda)^{-1}B\intkernelvar\|_{\mH} \leq \frac{\kappa R^2\|\intkernelvar\|_{\infty}\sqrt{2\log(4/\delta)}}{\sqrt{M}\lambda} \left(C_{\kappa,{\mH}}+\frac{C_{\kappa,R,\lambda}\sqrt{2\log(4/\delta)}}{\sqrt{M\lambda}} \right),$$
where $C_{\kappa,{\mH}}=(\kappa+1)\sqrt{\frac{2}{c_{\mH}}}$ and $C_{\kappa,R,\lambda}=\kappa R +\sqrt{\lambda}$.
\end{lemma}

Finally, we can analyze the perturbation $(B_M+\lambda)^{-1}A_{M}^{*}\mathbb{W}_M$ caused by the noise and obtain a bound for the sample error.

\begin{theorem}[Sample error bound]\label{hbound} For any $\delta \in (0,1)$, it holds with  probability at least $1-\delta$ that 
$$\|\intkernel_{\mH}^{\lambda,M}-\intkernel_{\mH}^{\lambda,\infty}\|_{\mH} \lesssim \frac{\kappa R^2\|\intkernele\|_{\infty}\sqrt{2\log(8/\delta)}}{\sqrt{M}\lambda}\left(C_{\kappa,{\mH}}+\frac{C_{\kappa,R,\lambda}\sqrt{2\log(8/\delta)}}{\sqrt{M\lambda}}\right) + \frac{2\kappa R\sigma \log(8/\delta)}{\sqrt{c}\lambda d \sqrt{MLN}},$$ where $c$ is an absolute constant appearing in the Hanson-Wright inequality (Theorem \ref{HAnson}), $C_{\kappa,{\mH}}=(\kappa+1)\sqrt{\frac{2}{c_{\mH}}}$ and $C_{\kappa,R,\lambda}=\kappa R +\sqrt{\lambda}$.
\end{theorem}

The detailed proofs of the above lemmas and theorem are shown in Appendix \ref{proofsappendix}.

\paragraph{ Analysis of approximation error $\| \intkernel_{\mH}^{\lambda,\infty}-\intkernele\|_{\mH}.$} 

Under the standard source condition (Assumption \ref{sourceconditionappendix}), the analysis of $\| \intkernel_{\mH}^{\lambda,\infty}-\intkernele\|_{\mH}$ follows the routine in the literature of Tikhonov regularization (see section 5 in \citep{caponnetto2005fast}). For the sake of being self-contained, we still present the analysis here.

Recall that $B=A^*A$ is a positive compact operator. Let $B=\sum_{n=1}^{N}\lambda_n\langle \cdot, e_n\rangle e_n$ (possibly $N=\infty$) be the spectral decomposition of $B$ with $0<\lambda_{n+1}<\lambda_{n}$ and $\{e_n\}_{n=1}^{N}$ be an orthonormal basis of $\mH$. 
 Then
\begin{align}\label{eq: approx error}
\| \intkernel_{\mH}^{\lambda,\infty}-\intkernele\|_{\mH}^2&=\|(B+\lambda)^{-1}B\intkernele-\intkernele\|_{\mH}^2 \nonumber \\& =\|\lambda (B+\lambda)^{-1}\intkernele\|_{\mH}^2 \nonumber \\&=\sum_{n=1}^{N}(\frac{\lambda}{\lambda_n+\lambda})^2|\langle \intkernele, e_n\rangle_{\mH}|^2. 
\end{align}
Assume now that $\intkernele \in \mathrm{Im}(B^{\gamma})$ with $0< \gamma \leq \frac{1}{2}$. Since the function $x^\gamma$ is concave on $[0,\infty]
$, $\frac{\lambda}{\lambda_n+\lambda}\leq \frac{\lambda^\gamma}{\lambda_n^\gamma}$. 
Then we have $\| \intkernel_{\mH}^{\lambda,\infty}-\intkernele\|_{\mH} \leq \lambda^{\gamma}\|B^{-\gamma}\intkernele\|_{\mH}$ where $B^{-\gamma}\intkernele$ represents the pre-image of $\intkernele$.

\begin{theorem}[\text{Convergence rate for posterior mean estimator}] \label{Maintheorem}Suppose that $\phi \in \mathrm{Im}(B^{\gamma})$ for some $\gamma \in (0,\frac{1}{2}]$. If we choose $\lambda \asymp M^{-\frac{1}{2\gamma+2}}$, then for any $\delta \in (0,1)$, it holds with probability at least $1-\delta$ that 
$$
\|\intkernel_{\mH}^{\lambda,M}-\intkernel\|_{\mH} \lesssim C( \intkernel, \kappa, R, c_{\mH}, \sigma)\log(\frac{8}{\delta}) M^{-\frac{\gamma}{2\gamma+2}},
 $$
 where $C =\max\{\frac{\kappa R^2 \|\intkernel\|_{\infty}}{\sqrt{c_{\mH}}}, \frac{2\kappa R\sigma}{\sqrt{LN}d}, \|g \|_{\mH}\}$, with $g$ satisfying $(A^*A)^{\gamma}(g)=\intkernel.$
\label{convergence}
\end{theorem}

\begin{proof} Without loss of generality, let $\lambda=M^{-\frac{1}{2\gamma+2}}$. By Theorem \ref{hbound} and approximation error \eqref{eq: approx error}, with a probability at least $1-\delta$, we have that 
\begin{align*}
 &\|\intkernel_{\mH}^{\lambda,M}-\intkernele\|_{\mH} \leq \|\intkernel_{\mH}^{\lambda,M}- \intkernel_{\mH}^{\lambda,\infty}\|_{\mH}+ \|\intkernel_{\mH}^{\lambda,\infty}-\intkernele\|_{\mH}\\
 \leq &\frac{\kappa R^2\|\intkernele\|_{\infty}\sqrt{2\log(8/\delta)}}{\sqrt{M}\lambda}\left(C_{\kappa,{\mH}}+\frac{C_{\kappa,R,\lambda}\sqrt{2\log(8/\delta)}}{\sqrt{M\lambda}} \right) + \frac{2\kappa R\sigma \log(8/\delta)}{\sqrt{c}\lambda d \sqrt{MLN}}+\lambda^{\gamma}\|B^{-\gamma}\intkernele\|_{\mH}\\
 \lesssim& C_1 M^{-\frac{\gamma}{2\gamma+2}}\sqrt{\log(8/\delta)} +C_2 M^{-\frac{\gamma}{2\gamma+2}} M^{-\frac{1+2\gamma}{4+4\gamma}}{\log(8/\delta)}+C_3 M^{-\frac{\gamma}{2\gamma+2}}\\
\leq &C\log(\frac{8}{\delta}) M^{-\frac{\gamma}{2\gamma+2}},
 \end{align*}where $C=\max\{\frac{\kappa^2 R^2 \|\intkernele\|_{\infty}}{\sqrt{c_{\mH}}}, \frac{2\kappa R\sigma}{\sqrt{cLN}d},\|B^{-\gamma}\intkernele \|_{\mH}\}$, and the symbol $\lesssim$ means that the inequality holds up to a multiplicative constant that  is independent of the listed parameters.
\end{proof}

Finally, we provide  an \text{ $L^{\infty}$ error analysis for marginal posterior variance \eqref{firstorder:var}}.

\begin{theorem}\label{marginalpos} For any $\delta \in (0,1)$, it holds with probability at least $1-\delta$ that 
\begin{align*}
|\mathrm{Var}(\bar\intkernele(r_*)|\mathbb{Y}_{M})|& \leq   \frac{\kappa \sigma^2}{ML\lambda N} \bigg(\kappa+\frac{\kappa R^2\|K_{r^\ast}\|_{\infty}\sqrt{2\log(4/\delta)}}{\sqrt{M}\lambda}\bigg(C_{\kappa,{\mH}}+\frac{C_{\kappa,R,\lambda}\sqrt{2\log(4/\delta)}}{\sqrt{M\lambda}} \bigg)\bigg),
\end{align*} where $C_{\kappa,{\mH}}=(\kappa+1)\sqrt{\frac{2}{c_{\mH}}}$ and $C_{\kappa,R,\lambda}=\kappa R +\sqrt{\lambda}$.
\end{theorem}
\begin{proof} Note that $K_{r^\ast}^{\lambda,M}=(B_M+\lambda)^{-1}B_MK_{r^\ast} $. Then 
\begin{align*}
K_{r^\ast}^{\lambda,M}-K_{r^\ast}&=(B_M+\lambda)^{-1}B_MK_{r^\ast} -(B+\lambda)^{-1}BK_{r^\ast}+(B+\lambda)^{-1}BK_{r^\ast}-K_{r^\ast} \\
&=(B_M+\lambda)^{-1}B_MK_{r^\ast} -(B+\lambda)^{-1}BK_{r^\ast}+\lambda (B+\lambda)^{-1}K_{r^\ast}.
\end{align*}
Applying Theorem \ref{sampleerror} to $K_{r^\ast}$, we know that,  for any $0< \delta <1$, with probability at least $1-\delta$, there holds 
{\small
$$\|(B_M+\lambda)^{-1}B_MK_{r^\ast} -(B+\lambda)^{-1}BK_{r^\ast}\|_{\mH} \leq \frac{\kappa R^2\|K_{r^\ast}\|_{\infty}\sqrt{2\log(4/\delta)}}{\sqrt{M}\lambda}\left(C_{\kappa,{\mH}}+\frac{C_{\kappa,R,\lambda}\sqrt{2\log(4/\delta)}}{\sqrt{M\lambda}} \right).$$
}
On the other hand,
$$\|\lambda (B+\lambda)^{-1}K_{r^\ast}\|_{\mH} \leq \|K_{r^\ast}\|_{\mH}.$$

Therefore,  for any $0< \delta <1$, with probability at least $1-\delta$, 
\begin{align*}
|\mathrm{Var}(\bar\intkernele(r_{*})|\mathbb{Y}_{M})|& \leq \frac{\sigma^2}{ML\lambda N}\|K_{r^\ast}^{\lambda,M}-K_{r^\ast}\|_{\infty} \\&\leq \frac{\kappa \sigma^2}{ML\lambda N}  \|K_{r^\ast}^{\lambda,M}-K_{r^\ast}\|_{\mH} 
\\&\leq   \frac{\kappa \sigma^2}{ML\lambda N} \bigg(\kappa+\frac{\kappa R^2\|K_{r^\ast}\|_{\infty}\sqrt{2\log(4/\delta)}}{\sqrt{M}\lambda}\bigg(C_{\kappa,{\mH}}+\frac{C_{\kappa,R,\lambda}\sqrt{2\log(4/\delta)}}{\sqrt{M\lambda}} \bigg)\bigg).
\end{align*} The conclusion follows.

\end{proof}

\textbf{Discussions} 
\begin{itemize}

\item The coercivity constant $c_{\mH}$ in fact depends on $N, L$, $\mu_0$ and $\mH$. In our paper, we consider both $N$ and $L$ fixed. We can prove that $c_{\mH}\geq \frac{N-1}{N^2}$ and can even be independent of $N$ for the case of $L=1$ for certain initial distributions (see section \ref{coedissucsion}). If $L$ changes, the measure $\tilde \rho_T^L$ will also change. It is not clear if $c_{\mH}$ would increase as $L$ increases. We defer more detailed discussions to section \ref{coedissucsion}. In addition, our theoretical framework suggests that it is possible to use a part of equations $(N_1 << N)$ for learning, as long as a form of coercivity condition is satisfied. We leave it for future investigation. 

\item We show that a parametric learning rate in terms of $M$ for marginal posterior variance can be obtained. In particular, from our analysis of the coercivity condition, it is possible to show the bound is proportional to $\frac{1}{N}$ (number of particles) for the case of $L=1$, when $c_{\mathcal{H}_K}$ is independent of $N$ (see an example  in section \ref{coedissucsion}) due to the independence of noise. For higher $L > 1$, the dependence of $c_{\mathcal{H}_K}$ on $N$ is unclear. For the reconstruction error, the current rate only sees $M$ as the effective sample size, i.e. number of random samples. Obtaining
this rate is satisfactory because we do not observe the values of $\intkernel$ and the pairwise distances are in general correlated. The convergence rate in $M$ coincides with the optimal minimax rate achieved in the classical 1-dimensional KRR problem, Problem \ref{krr}, for the set of functions $\Omega_{\gamma,S}$, see the summary in the second column of Table 1 in \citep{blanchard2018optimal} and their associated references. Our convergence rate is done for all Mercer kernels, where one can refer to $s=0$ (reconstruction error) and $b \rightarrow 1^{+}$  in the main result of \citep{blanchard2018optimal}. Using our framework as the bridge, we believe we can obtain more refined rates and bounds if we know, for example, the decay of eigenvalues of $A$. This opens many future questions to investigate.

\item Recall in Remark \ref{smoothness}, if 
$K \in C^{2s+\epsilon}([0,R]\times [0,R])$ with $0<\epsilon<2$, we have 
$\|\varphi\|_{C^{s}}\leq 4^s \|K\|_{C^{2s}}^{\frac{1}{2}}\|\varphi\|_{\mH}, \forall \varphi \in \mH.$
Therefore, we obtain the convergence rate in terms of $C^s$ norm, which is a stronger norm than the previous $L^2$ convergence \citep{lu2019nonparametric}.

\end{itemize}

\subsection{Discussion on the coercivity constant}\label{coedissucsion}
The coercivity condition was proposed in \citep{lu2019nonparametric}, where a least square approach was proposed to learn $\intkernele$ over a suitably chosen hypothesis function space with complexity adaptive to data. The coercivity condition \eqref{coercivity} in this paper can be viewed as a special instance when the hypothesis space is set to be $\mH$. We review below briefly the recent study on the coercivity condition.

When the initial distributions of the agents are exchangeable, the coercivity condition is closely related to the positiveness of integral operators that arise in \eqref{coercivity}:
\begin{align*}
\kappa_{\mH}\|\intkernelvar\|^2_{L^2( \tilde\rho_T^L)} \! & \leq 
\frac{1}{L}\sum_{l=1}^{L} \E_{\bX(0)\sim\mu_0}[\varphi(|\br_{12}(t_l)|)\varphi(|\br_{13}(t_l)|)\innerp{\br_{12}(t_l)}{\br_{13}(t_l)}] \\
& =\int_0^\infty \int_0^\infty \varphi(r)\varphi(s) \overline{K}(r,s)drds, \forall \intkernelvar \in \mH
\end{align*}
where the integral kernel $\overline{K}:\R^+\times \R^+ \to \R$ is defined as
 \begin{equation} \label{kernelK}
\overline{K}(r,s) := (rs)^{d} \int_{S^{d-1}}\int_{S^{d-1}}\innerp{\xi}{\eta}\frac{1}{L}\sum_{l=1}^{L} p_{t_l}(r\xi,s\eta) d\xi d\eta,
 \end{equation} with $p_{t_l}(u,v)$ denoting the joint density function of the random vector $(\br_{12}(t_l), \br_{13}(t_l))$ and $\mathbb{S}^{d-1}$ denoting the unit sphere in $\R^d$. The coercivity constant satisfies that $$c_{\mH}=\frac{N-1}{N^2}+\frac{(N-1)(N-2)}{N^2}\kappa_{\mH}.$$ Therefore, if the integral kernel $\overline{K}$ is positive definite, i.e., $\kappa_{\hypspace}\geq 0$, then the coercivity condition holds on $\mH$ with $c_{\mH}\geq \frac{N-1}{N^2}$.
 
In the case of $L=1$ where the initial distributions of the agents are exchangeable Gaussian, it is proven in \citep{lu2021learning} that $\kappa_{\mH}>0$ provided $\mH$ can be compactly embedded into $L^2([0,R];\tilde\rho_T^1;\mathbb{R})$. Back to our setting, the previous results indicate that the coercivity condition \eqref{coercivity} is satisfied with the coercivity constant independent of $N$ if

 \begin{itemize}
 \item $\mH$ is finite dimensional
 \item $\mH$ is a subspace of Sobolev space $W_2^{s}([0,R])$ for $s>\frac{1}{2}$ (see def in \eqref{sobolev}). 
 \item the kernel $\mK$ is a $C^{\infty}$ Mercer kernel. 
 \end{itemize}
 
The space $\mH$ in the last two examples can be  embedded compactly into $C([0,R])$\citep{Cucker02onthe} and therefore into the space $L^2([0,R];\tilde\rho_T^1;\mathbb{R})$. In these scenarios, the condition number of the inverse problem is uniformly bounded below and is independent of the number of agents in the system.

 We also prove that $\overline{K}(r,s)$ is positive definite if the initial distribution of each agent is drawn i.i.d according to a probability measure on $\mathbb{R}^d$. This indicates that the coercivity constant $c_{\mathcal{H}_K} \geq \frac{N-1}{N^2}$ as {long} as $\mH$ is a subspace of $L^2([0,R];\tilde \rho_T^1;\mathbb{R})$. The generalization for the case $L>1$ is difficult for deterministic systems, due to the implicit solutions to the systems and the richness of collective behaviors which cause grand challenges to analyzing distributions in a unified way.

 However, as the numerical results and relevant discussions in \citep{lu2019nonparametric,lu2021learning}, we believe that the coercivity condition is ``generally'' satisfied for various systems and initial distributions for the case $L>1$. We also refer the reader to \citep{lu2020learning} for the study of coercivity conditions in stochastic systems.

\subsection{Analysis of computational complexity}\label{computation}

 To compute the posterior mean and variance, the direct construction of the covariance matrix $K_{\phi}$ requires $\mathcal{O} (N^4M^2L^2d)$ operations and the  direct inversion of the covariance matrix requires $\mathcal{O}((NMLd)^3)$ operations.  Theoretically, we prove the scalability in $M$: for example, when $\gamma=\frac{1}{2}$, for the accuracy $\epsilon$, it is of the order $\mathcal{O}((\frac{1}{\epsilon})^6)$,  independent of the ambient dimension $dN$. \textcolor{black}{ In our numerical experiments, we used direct inversion of kernel matrices as we focused on the scarce and noisy data regime. In our numerical sections, we test our approach on systems with dimensions ranging from 10 to 60. }

 The most expensive computational part of the full GP model is on constructing the covariance matrix $K_{\phi}$ and inverting it. This is a well-known limitation of the GP approach. There are many possible ways of overcoming the computational bottleneck. Currently, we are investigating the  
 use of a sparse conjugate gradient method (CG) to solve the linear system which does not require {assembling} the covariance matrix and uses an iterative method to get the estimator.  Our Representer theorem implies that the covariance matrix $K_{\phi}$ has a special sparse structure depending on the covariance kernel we use, which allows us to efficiently compute the matrix-vector multiplication used in each iteration of CG, similar to the ideas used in the Kalman filter. The total cost can be reduced to $\mathcal{O}(N^2MLdp)$ where $p$ is the number of total iterations (usually a few hundred steps). We refer the reader to \citep{gu2022scalable} for the preliminary  investigation  in first-order systems.

\section{Numerical Examples}
\label{sec:numericalresult}

\paragraph{Numerical setup.} We simulate the trajectory data on the time interval $[0,T]$ with given i.i.d initial conditions generated  from the probability measures specified for each system. For the training data sets, we generate $M$ trajectories and observe each trajectory at $L$ equidistant times $0 = t_1 < t_2 < \cdots < t_L = T$.  All ODE systems are evolved using \textrm{ode$15$s} in MATLAB\textsuperscript{\textregistered}2020a with a relative tolerance at $10^{-5}$ and absolute tolerance at $10^{-6}$. We apply the \textit{minimize} function in the GPML package\footnote{Carl Edward Rasmussen \& Hannes Nickisch (http://gaussianprocess.org/gpml/code)} to train the parameters using conjugate gradient optimization with the partial derivatives shown in Proposition \ref{prop: derivs}, and set the maximum number of function evaluations to 600.

\paragraph{Choice of the covariance function.} We choose the Mat\'{e}rn covariance function restricted on $[0,R] \times [0,R]$ for the Gaussian process priors in our numerical experiments, i.e.,
\begin{equation}
    K_\theta(r,r')=s_\intkernel^2 \frac{2^{1-\nu}}{\Gamma(\nu)}(\frac{\sqrt{2\nu}|r-r'|}{\omega_{\intkernel}})^\nu B_\nu(\frac{\sqrt{2\nu}|r-r'|}{\omega_{\intkernel}}),
\end{equation}
where the parameter $\nu > 0$ determines the smoothness; $\Gamma(\nu)$ is the Gamma function; $B_\nu$ is the modified Bessel function of {the} second kind; and the hyper-parameters $\theta = \{s_\intkernel^2, \omega_{\intkernel}\}$ quantify the amplitude and scales. 

The Reproducing Kernel Hilbert Space (RKHS), $\mathcal{H}_{\mathrm{Mat\acute{e}rn}}$, associated with this Mat\'{e}rn kernel is norm-equivalent to the Sobolev space $W_2^{\nu + 1/2}([0,R])$ defined by 
\begin{equation}\label{sobolev}
    W_2^{\nu + 1/2}([0,R]) : = \Big\{ f \in L^2([0,R]): \|f\|^2_{W_2^{\nu + 1/2} } := \sum_{\beta \in \mathbb{N}_0^1:|\beta|\leq \nu + 1/2} \|D^\beta f\|_{L_2}^2 < \infty  \Big\}.
\end{equation}
That is to say,  $\mathcal{H}_{\mathrm{Mat\acute{e}rn}} = W_2^s([0,R]) $ as a set of functions, and there exist constants $c_1,c_2>0$ such that
\begin{equation}
    c_1\|f\|_{W_2^{\nu+\frac{1}{2}}} \leq \|f\|_{\mathcal{H}_{\mathrm{Mat\acute{e}rn}}} \leq c_2||f||_{W_2^{\nu+\frac{1}{2}}}, \quad \forall f \in \mathcal{H}_{\mathrm{Mat\acute{e}rn}}.
\end{equation}
In other words, $\mathcal{H}_{\mathrm{Mat\acute{e}rn}}$ consists of functions that are differentiable up to order $\nu$ and weak differentiable up to order $s = \nu + \frac
{1}{2}$.

\paragraph{Baseline comparisons}

We perform comparisons with approaches that learn the right-hand side function of \eqref{firstorder:force} directly from trajectory data: the first one is SINDy  \citep{brunton2016discovering}, which aims at finding a sparse representation for each row of governing equations in a (typically large) dictionary; the second one is regression using Feed-Forward Neural networks (FNN), for which we use the MATLAB\textsuperscript{\textregistered} 2021a Deep Learning Toolbox\textsuperscript{\texttrademark}. To evaluate the performance, we compare the trajectory prediction errors of the estimators. We also perform a comparison with the previous least square approach for learning $\phi$, see Table \ref{tab:FM_C_traj}.

\begin{table}
\caption{Short Notations}
\vspace{1em}
\label{tab:2ndOrder_notationdef} 
\centering
{
\small{\begin{tabular}{ l l }
\toprule
Notation          & Definition \\

\midrule
GPs & Gaussian Processes\\
\midrule
GPR & Gaussian Process Regression\\
\midrule
RKHS & Reproducing Kernel Hilbert Space\\
\midrule
KRR & Kernel Ridge Regression\\
\midrule
IC & Initial Condition\\
\bottomrule
\end{tabular}}  
}
\end{table}

\subsection*{Overview of the numerical results}
\begin{itemize}
\item The proposed algorithm performs \textit{simultaneous} accurate estimations of $\mbf{\alpha}$ in the non-collective force functions and $\phi$ from a \textit{small} amount of \textit{noisy} trajectory data, even in the cases where $\mbf{F}$ depends \textit{nonlinearly} on $\mbf{\alpha}$ (See Example \ref{ex: OD}). Although we are dealing with  non-convex optimization in our training step, our algorithm works well in tuning the hyper-parameters when the  initialization is within an appropriate range.  We find that learning $\phi$ is more challenging, since finding $\alpha$ is a parameter estimation problem with a small number of unknowns (2 to 5 unknowns) while finding $\phi$  is a nonparametric inference problem and suffers from the possible ill-posedness of the inverse problem. 
We show in Example \ref{ex: OD} (Table \ref{tab:OD_appendix})  the existence of  outliers in learning $\phi$ from noise-free data,  while we do not observe this phenomenon for  $\mbf{\alpha}$ and the cases using noisy training data. It suggests that regularization is needed in the noise-free case.

  \item Our numerical results show that the estimation errors for both $\mbf{\alpha}$ and $\intkernel$ decrease as the size of training data increases. It remains elusive regarding the role of $N$, $M$, and $L$ in determining the size of ``effective'' samples, which serves as the core challenge in learning complex systems.  In addition, we would like to point out that the hyper-parameters for the Mat\'{e}rn  kernel are known to be unidentifiable \citep{zhang2004inconsistent, tang2021identifiability}. Our learning theory treats  $M$ as the effective sample size and $N$, $L$ fixed. We address how to choose the prior as $M\rightarrow \infty$ so as to achieve the optimal convergence of estimators.  We leave all other regimes for future work.

\item The marginal posterior variances obtained in our learning approach quantify the reliability of estimated kernels and are fairly small in the region well-explored by the training data. We also observe that the estimators can extrapolate well in the regions which are not explored by the training data. We impute this extrapolation property to the powerful training procedure of GPR, which learns a covariance kernel function that achieves an automatic trade-off between data-fit and model complexity. For the fixed sample size, the width of the uncertainty band increases as the noise level increases. We remark that this can also serve as a sign of the model mismatch error if the system is applied to fit a real dataset. 

 \item Predicting 
 long term behavior of complex systems is known to be very challenging. Our estimators are shown to have good performance in prediction and generalization.  The occasional large prediction errors that happened in a larger time interval may be caused by the propagation of estimation errors. We still think the performance is satisfactory since we only have very limited and noisy training data. Even in cases where the prediction errors are relatively large, the estimators can predict remarkably accurate collective behaviors of the agents, e.g. the consensus in the opinion dynamics, the flocking behavior in the Cuker-Smale dynamics, and the milling pattern in the fish milling dynamics.
 
 \item Besides the SINDy and FNN models, we also conduct a comparison with the learning approach proposed in previous work \citep{lu2019nonparametric,zhong2020data} when $\phi$ is the only unknown term in the governing equation (see Example \ref{ex: FM}.)
\end{itemize}

\subsection{Example 1: Opinion dynamics (OD) with stubborn agents}
\label{ex: OD}

We consider the Taylor model \citep{taylor1968towards}, which models the collective dynamics of continuous opinion exchange in the presence of stubborn agents. It is a first-order system of $N$ interacting agents, and each agent $i$ is characterized by a continuous opinion variable $x_i \in \mathbb{R}$. The dynamics of opinion exchange are governed by the following first-order equation,
\begin{eqnarray}
  \dot\bx_i &=&\force_i(\bx_i, \mbf{\alpha}) + \sum_{i'=1}^N \frac{1}{N} \intkernele (\norm{\bx_{i'} - \bx_i})(\bx_{i'} - \bx_i),
  \label{eq:1storderOD}
\end{eqnarray}
where
\begin{equation}
  \intkernele(r) = 
  \begin{cases}
  2.5 r & \textrm{if } 0 \leq r < 0.4\\
  1 & \textrm{if } 0.4 \leq r < 0.6\\
  2.5 - 2.5r & \textrm{if } 0.6 \leq r < 1\\
  0 & \textrm{if } r \geq 1
  \end{cases}
\label{eq:ODS_phi}
\end{equation}
and
\begin{equation}
  \force_i(\bx_i, \mbf{\alpha}) = 
  \begin{cases}
  - \kappa(\bx_i - P_i) & \textrm{if agent $i$ is stubborn with bias $P_i$}\\
  0 & \textrm{otherwise}
  \end{cases}
\end{equation}
The interaction kernel $\intkernele$ encodes the non-repulsive interactions between agents: all agents aim to align their opinions to their connected neighbors according to distanced-based attractive influences. The non-collective force $\force(\bx_i, \mbf{\alpha})$ describes the additional influence induced by the stubbornness: the stubborn agents have strong desires to follow their biases $P_i$, and $\kappa$ controls the rate of convergence towards {their} biases. The stubborn agents may cause a major effect on the collective opinion formation process. If $\kappa=0$, then stubborn agents do not follow their biases and behave as regular agents. 

We are interested in learning the parameters $\mbf{\alpha} = (P_1,P_2,P_3,\kappa)$ and interaction kernel $\intkernele$ from trajectory data. 
Note that this first-order system is a special case of the second-order system \eqref{2ndodes} 
with $m_i = 0$ for all $i$, and $\force_i(\bx_i, \dot\bx_i, \mbf{\alpha}) = -\dot\bx_i + \force_i(\bx_i, \mbf{\alpha})$. In this example,  the unknown scalar parameters in $\balpha$ are nonlinear with respect to the non-collective force function and the interaction kernel $\intkernel \notin \mathcal{H}_{\mathrm{Mat\acute{e}rn}}$.

The training data $(\bbX_M, \bbV_{\sigma^2,M})$ is generated with parameters shown in Table \ref{tab:ex_ODS_info}, and the observations are made {in} the time interval $[0,15]$ with different size of observations $\{M,L\}$, and different noise level $\sigma$. 

\begin{table}[!htb]
\caption{System parameters in the opinion dynamics}
\label{tab:ex_ODS_info} 
\centering
\small{
\small{\begin{tabular}{ ccccc}
\toprule
 $d$ & $N$ & $[0; T; T_f]$ & $\alpha = (P_1,P_2,P_3,\kappa)$ & $\mu_0$\\
\midrule
1 & 10 & $[0,15,20]$ & $(1,0,-1,10)$ & $\mathrm{Unif}([-1,1])$\\
\bottomrule
\end{tabular}}  
}
\end{table}

 We initialize the parameters $ (\sigma,P_1,P_2,P_3,\kappa) = (1/2,1/2,1/2,1/2,1/2)$. Table \ref{tab:OD_appendix} shows the errors of the estimations for $\balpha$ and $\intkernel(r)$ in 10 independent trials of experiments. The results demonstrate that our algorithm can produce an accurate estimation of the parameters $(\sigma, P_1,P_2,P_3,\kappa)$ from both noise-free and noisy training data. For the estimation of $\intkernel(r)$, {even though} $\intkernel$ is not in the RKHS generated by the Mat\'{e}rn kernel, our algorithm still provides us with faithful prediction in the region  the training data covers, see Figure \ref{fig:ex_OD} (a). At the region around $r = 0$, we see the approximation is not as good as in other regions. We impute this phenomenon to the fact that $\intkernel(r)$ is weighted by $\mbf{r}$ in the model \eqref{eq:1storderOD}, thus we lose the information of $\intkernel$ when $\mbf{r}$ is close to zero. However, we expect that our estimators will produce accurate trajectories since they are generated by $\intkernel(r)\mbf{r}$, and Figure \ref{fig:ex_OD} (b)  supports this intuition.



\begin{table}[!htb]
\begin{center}
\caption{Means and standard deviations of the  errors of $\hat\balpha$ (including $\hat \sigma$ when noise exists) and $\hat{\intkernel}$ for different settings} \label{tab:OD_appendix}
\vspace{0.1in}
\begin{tabular}{ccc}
\toprule[.05cm]
$\{N,M,L,\sigma\}$ & $\| \hat{\balpha} - \balpha\|_{\infty}$  & $\| \hat{\intkernel} - \intkernel\|_{\infty}$ \\
\cmidrule(lr){1-1}\cmidrule(lr){2-3}
$\{10,3,4,0\}$ \footnotemark[1] & \textbf{$5.1\cdot 10^{-3} \pm 8.2\cdot 10^{-3}$ } & $1.0\cdot 10^{-1}\pm 2.2\cdot 10^{-2}$  \\
 & & (\textcolor{blue}{$6.0\cdot 10^{-1}\pm 2.2\cdot 10^{-2}$}) \\
\cmidrule(lr){1-1}\cmidrule(lr){2-3}
$\{10,3,8,0\}$ & $4.2 \cdot 10^{-4}\pm 5.4\cdot 10^{-3}$ & $7.5\cdot 10^{-2}\pm 2.5\cdot 10^{-2}$ \\
\cmidrule(lr){1-1}\cmidrule(lr){2-3}
$\{10,6,4,0\}$  \footnotemark[2] & $2.8 \cdot 10^{-3} \pm 9.8\cdot 10^{-3}$  & $7.6\cdot 10^{-2}\pm 2.2\cdot 10^{-2}$\\
 &  & (\textcolor{blue}{$1.7\cdot 10^{-1}\pm 3.0\cdot 10^{-1}$}) \\
\cmidrule(lr){1-1}\cmidrule(lr){2-3}
$\{10,6,4,0.01\}$ & $1.6\cdot 10^{-3} \pm 4.7\cdot 10^{-3}$ & $5.9\cdot 10^{-2}\pm 2.8\cdot 10^{-2}$ \\
\cmidrule(lr){1-1}\cmidrule(lr){2-3}
$\{10,6,4,0.03\}$ & $ 4.1\cdot 10^{-3} \pm 1.1\cdot 10^{-2}$ & $1.4\cdot 10^{-1}\pm 8.3\cdot 10^{-2}$\\
\cmidrule(lr){1-1}\cmidrule(lr){2-3}
$\{10,6,4,0.05\}$ & $7.2\cdot 10^{-3} \pm 2.2\cdot 10^{-2}$ & $1.9\cdot 10^{-1}\pm 1.3\cdot 10^{-1}$ \\
\bottomrule
\end{tabular} 
\end{center}
\end{table}

\footnotetext[1]{omit 2 trails in the 10 independent learning trails for errors in $\intkernel$ and corresponding trajectories, the result with all 10 trials are shown in the brackets below}
\footnotetext[2]{omit 1 trail in the 10 independent learning trails for errors in $\intkernel$ and corresponding trajectories, the result with all 10 trials are shown in the brackets below}

\begin{figure}[!htb]
\centering
\subfigure[OD:  vs learned kernel]{
\includegraphics[width=0.45\linewidth]{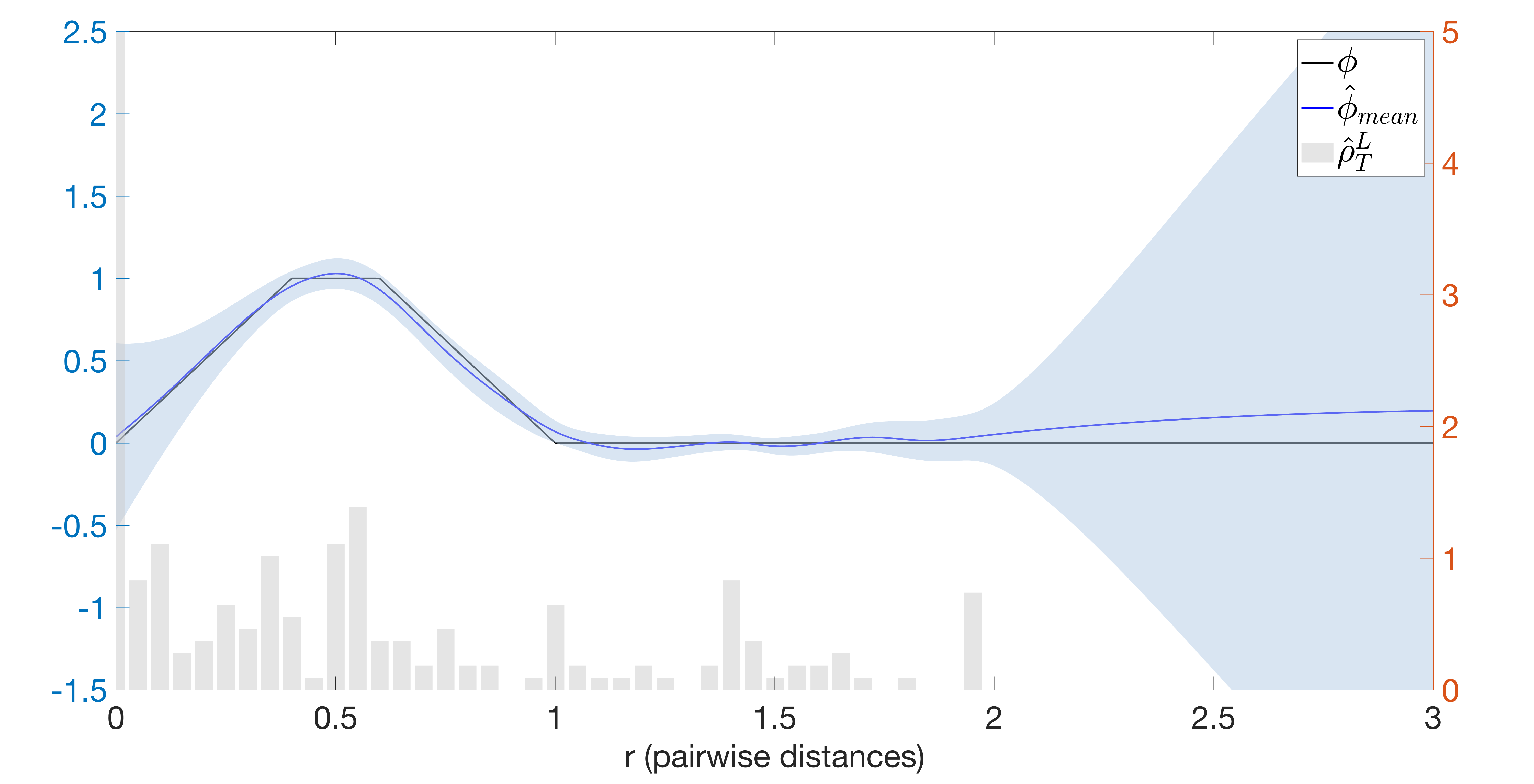}
}
\subfigure[ versus predicted model]{
\includegraphics[width=0.47\linewidth]{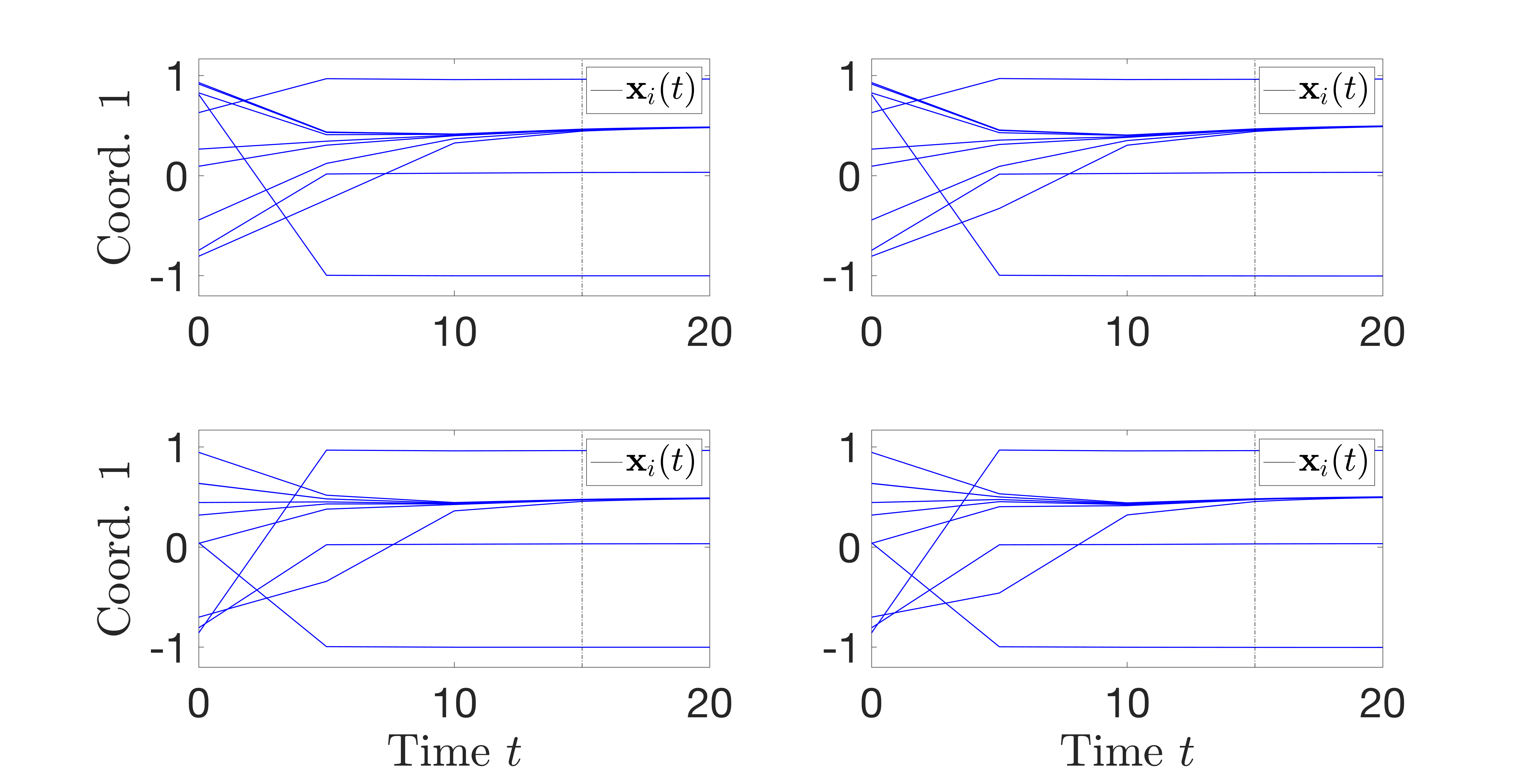}
}
\caption{Learning OD ($\{N,M,L,\sigma\} = \{10,6,4,0.05\}$) using the Matérn kernel. (a): predictive mean $\hat{\intkernel}$ of the  kernel, and two-standard-deviation band (light blue color) around the mean. The grey bars represent the empirical density of the $\rho_T^L$. (b): the {true} (left)  versus predicted (right) trajectories using $\hat{\balpha}$ and $\hat{\intkernel}$  with initial conditions of training data (top) and testing data (bottom).  }
\label{fig:ex_OD}
\end{figure}

The comparison between the trajectories generated by the parameters $\balpha$ { and} interaction kernel $\intkernel$, and the estimated parameters $\hat{\balpha}$ {and} interaction kernel $\hat{\intkernel}$, is shown in Table \ref{tab:OD_traj_appendix}. We can see that in both {the} training time interval $[0,15]$ and future time interval $[15,20]$, the estimators can produce accurate approximations of the  trajectories and the performance becomes better when we increase the size of training data ($M$ or $L$).


\begin{table}[!htb]
\begin{center}
\caption{The trajectory prediction errors for different settings.} \label{tab:OD_traj_appendix}
\vspace{0.1in}
\resizebox{\linewidth}{!}{%
\begin{tabular}{ccccc}
\toprule[.05cm]
$\{N,M,L,\sigma\}$ & Training IC $[0,15]$ & Training IC $[15,20]$ & New IC $[0,15]$ & New IC $[15,20]$\\
\cmidrule(lr){1-1}\cmidrule(lr){2-5}
$\{10,3,4,0\}$ & $1.2\cdot 10^{-2}\pm 9.3\cdot 10^{-3}$ & $2.3\cdot 10^{-2}\pm 4.2\cdot 10^{-2}$ & $1.4\cdot 10^{-2}\pm 7.1\cdot 10^{-3}$ & $2.1\cdot 10^{-2}\pm 3.2\cdot 10^{-2}$\\
\cmidrule(lr){1-1}\cmidrule(lr){2-5}
$\{10,3,8,0\}$ & $8.1\cdot 10^{-3}\pm 3.4\cdot 10^{-3}$ & $7.0\cdot 10^{-3}\pm 5.4\cdot 10^{-3}$ & $8.3\cdot 10^{-3}\pm 5.3\cdot 10^{-3}$ & $5.9\cdot 10^{-3}\pm 4.4\cdot 10^{-3}$\\
\cmidrule(lr){1-1}\cmidrule(lr){2-5}
$\{10,6,4,0\}$ & $1.1\cdot 10^{-2}\pm 6.2\cdot 10^{-3}$ & $8.0\cdot 10^{-3}\pm 7.0\cdot 10^{-3}$ & $1.8\cdot 10^{-2}\pm 1.8\cdot 10^{-2}$ & $6.9\cdot 10^{-3}\pm 4.4\cdot 10^{-3}$ \\
\cmidrule(lr){1-1}\cmidrule(lr){2-5}
$\{10,6,4,0.01\}$ & $3.4\cdot 10^{-2}\pm 2.0\cdot 10^{-2}$ & $2.7\cdot 10^{-2}\pm 2.0\cdot 10^{-2}$ & $4.2\cdot 10^{-2}\pm 2.0\cdot 10^{-2}$ & $4.1\cdot 10^{-2}\pm 4.4\cdot 10^{-2}$\\
\cmidrule(lr){1-1}\cmidrule(lr){2-5}
$\{10,6,4,0.03\}$ & $6.6\cdot 10^{-2}\pm 3.3\cdot 10^{-2}$ & $5.7\cdot 10^{-2}\pm 4.2\cdot 10^{-2}$ & $7.1\cdot 10^{-2}\pm 3.1\cdot 10^{-2}$ & $3.9\cdot 10^{-2}\pm 1.7\cdot 10^{-2}$\\
\cmidrule(lr){1-1}\cmidrule(lr){2-5}
$\{10,6,4,0.05\}$ & $1.1\cdot 10^{-1}\pm 4.2\cdot 10^{-2}$ & $8.0\cdot 10^{-2}\pm 6.4\cdot 10^{-2}$ & $1.2\cdot 10^{-1}\pm 5.9\cdot 10^{-2}$ & $6.9\cdot 10^{-2}\pm 4.1\cdot 10^{-2}$\\
\bottomrule
\end{tabular} 
}
\end{center}
\end{table}

\subsection{Example 2: Fish-Milling (FM) dynamics with
friction force}  \label{ex: FM}
We consider the D'orsogma model \citep{d2006self,chuang2007state} which describes the motion of $N$ self-propelled particles powered by biological or mechanical motors under frictional forces: for $i=1,\cdots,N$,
\begin{align}
\label{motionequation}
m_i\ddot\bx_i =\force_i(\bx_i, \dot\bx_i, \mbf{\alpha}) + \sum_{i'=1}^N \frac{1}{N} \intkernel(\norm{\bx_{i'} - \bx_i})(\bx_{i'} - \bx_i),
\end{align}

\begin{align}
\label{eq:Fishmilling}
\force_i(\bx_i, \dot\bx_i, \mbf{\alpha}) =(\gamma-\beta |\dot\bx_i|^2)\dot\bx_i.
\end{align}

The form of \eqref{motionequation} is derived using Newton’s law  with the right hand side of \eqref{motionequation} describing the three forces acting on each agent: self-propulsion with strength $\gamma$, nonlinear drag with strength $\beta$, and social interactions determined by $\intkernel$. This system can produce a rich variety of collective patterns: in our numerical example, we consider the interaction kernel that is derived from the Morse-type potential
\begin{equation}
\intkernel(r) = \frac{1}{r}\bigg[-\frac{C_{rp}}{l_{rp} }e^{-\frac{r}{l_{rp}}}+\frac{C_a}{l_a}e^{-\frac{r}{l_a}}\bigg],
\end{equation} where $l_a, l_{rp}$ represent the attractive and repulsive potential ranges and $C_a, C_{rp}$ represent the respective amplitudes. Since this kernel is singular at $r=0$, we truncate it at $r_0=0.05$ with a function of the form $ae^{-br}$ to ensure that the new function has a continuous derivative. We assume that we {do not have} knowledge of the parametric form of $\intkernel$, $\gamma$, and $\beta$, and our goal is to learn them from the trajectory data.


As mentioned above, the training data $(\bbY_M, \bbZ_{\sigma^2,M})$ is generated with different numbers of agents $N$ and the parameters shown in Table \ref{tab:ex_FM_info}, and the observations are made in the time interval $[0,5]$ with different sizes $\{M, L\}$ and different amounts of additive noise $\sigma$.

\begin{table}[!htb]
\caption{System parameters in the fish milling dynamics}
\label{tab:ex_FM_info} 
\centering
\vspace{0.1in}
\small{
\small{\begin{tabular}{cccccccc}
\toprule
 $d$ & $m_i$ & $[0; T; T_f]$ & $\mbf{\alpha} = (\gamma, \beta)$ & $(C_{rp},l_{rp})$ & $(C_a,l_a)$ & $\mu_0^\bx$ & $\mu_0^\bv$ \\
\midrule
2 & 1& $[0;5;10]$ & $(1.5, 0.5)$ & (0.5,0.5) & (4,4) & $\mathcal{U}([-0.5,0.5]^2)$ & $(0,0)$\\
\bottomrule
\end{tabular}}  
}
\end{table}

We initialize the parameters $(\gamma, \beta) = (1,1)$, and $\sigma = 1$ for the cases with noisy data. The errors of the estimations for $\mbf{\alpha}$ after our training procedure and the learned $\intkernel$ are shown in Table \ref{tab:FM}. In this model, $\intkernel$ is in the RKHS generated by the chosen Mat\'{e}rn kernel. We can see that our estimators produced faithful approximations to the  kernel based on the results, see Figure \ref{fig:ex_FM_traj} (a). We also compare the discrepancy between the  trajectories (evolved using $\alpha$, $\intkernel$) and predicted trajectories (evolved using $\hat{\alpha}$, $\hat{\intkernel}$) on both the training time interval $[0, T]$ and on the future time interval $[T, T_f]$, over two different sets of initial conditions (IC) – one taken from the training data, and one consisting of new samples from the same initial distribution, see the results of different cases in Table \ref{tab:FM_traj} and Figure \ref{fig:ex_FM_traj} (b)-(c).

\begin{table}[!htb]
\begin{center}
\caption{Means and standard deviations of the  errors of $\hat\balpha$ (including $\hat \sigma$ when noise exists) and $\hat{\intkernel}$ for different settings
} \label{tab:FM}
\vspace{0.1in}
\begin{tabular}{ccc}
\toprule[.05cm]
$\{N,M,L,\sigma\}$ & $\| \hat{\balpha} - \balpha\|_{\infty}$  & $\| \hat{\intkernel} - \intkernel\|_{\infty}$ \\
\cmidrule(lr){1-1}\cmidrule(lr){2-3}
$\{10,1,3,0\}$ & $3.1\cdot 10^{-4}\pm 1.5\cdot 10^{-4}$ & $ 2.6\cdot 10^{-2} \pm 4.3\cdot 10^{-3}$ \\
\cmidrule(lr){1-1}\cmidrule(lr){2-3}
$\{10,1,9,0\}$ & $1.2\cdot 10^{-4} \pm 1.8\cdot 10^{-4}$ & $ 2.3\cdot 10^{-2} \pm 3.1\cdot 10^{-3}$ \\
\cmidrule(lr){1-1}\cmidrule(lr){2-3}
$\{10,3,3,0\}$ & $2.3\cdot 10^{-4} \pm 1.7\cdot 10^{-4}$  & $ 2.5\cdot 10^{-2} \pm 3.4\cdot 10^{-3}$\\
\cmidrule(lr){1-1}\cmidrule(lr){2-3}
$\{5,3,3,0\}$ & $3.0\cdot 10^{-4} \pm 1.6\cdot 10^{-4}$ & $ 2.4\cdot 10^{-2} \pm 3.3\cdot 10^{-3}$\\
\cmidrule(lr){1-1}\cmidrule(lr){2-3}
$\{10,3,3,0.01\}$ & $8.5\cdot 10^{-4} \pm 2.6\cdot 10^{-3}$  & $ 2.5\cdot 10^{-2} \pm 5.0\cdot 10^{-3}$ \\
\cmidrule(lr){1-1}\cmidrule(lr){2-3}
$\{10,3,3,0.05\}$ & $3.0 \cdot 10^{-3} \pm 1.3\cdot 10^{-2}$  & $1.7\cdot 10^{-2} \pm 8.7\cdot 10^{-3}$\\
\cmidrule(lr){1-1}\cmidrule(lr){2-3}
$\{10,3,3,0.1\}$ & $5.7\cdot 10^{-3} \pm 2.7\cdot 10^{-2}$  & $3.0 \cdot 10^{-2} \pm 1.7\cdot 10^{-2}$\\
\bottomrule
\end{tabular} 
\end{center}
\end{table}

Even if the trajectory prediction errors can go up to $O(10^{-1})$ with the presence of a relatively large noise for the systems with $N=10$, our estimators provided faithful predictions to most of the agents in the system, and the milling pattern as shown in Figure \ref{fig:ex_FM_traj}(c). 




\begin{table}[!htb]
\begin{center}
\caption{The trajectory prediction errors for different settings of FM dynamics.} \label{tab:FM_traj}
\vspace{0.1in}
\resizebox{\linewidth}{!}{%
\begin{tabular}{ccccc}
\toprule[.05cm]
$\{N,M,L,\sigma\}$ & Training IC $[0,5]$ & Training IC $[5,10]$ & New IC $[0,5]$ & New IC $[5,10]$\\
\cmidrule(lr){1-1}\cmidrule(lr){2-5}
$\{10,1,3,0\}$ & $2.6\cdot 10^{-2} \pm 9.7 \cdot 10^{-3}$ &  $6.9\cdot 10^{-2}\pm 3.3\cdot 10^{-2}$  & $2.7\cdot 10^{-2} \pm 7.7 \cdot 10^{-3}$ & $1.1\cdot 10^{-1}\pm 1.0\cdot 10^{-1}$\\
\cmidrule(lr){1-1}\cmidrule(lr){2-5}
$\{10,1,9,0\}$ & $1.6\cdot 10^{-2} \pm 8.1\cdot 10^{-3}$ & $4.2\cdot 10^{-2} \pm 2.1\cdot 10^{-2}$  &  $1.4\cdot 10^{-2} \pm 3.7\cdot 10^{-3}$ & $3.7\cdot 10^{-2} \pm 2.2\cdot 10^{-2}$\\
\cmidrule(lr){1-1}\cmidrule(lr){2-5}
$\{10,3,3,0\}$ & $1.4 \cdot 10^{-2} \pm 9.1 \cdot 10^{-3}$ & $4.4\cdot 10^{-2} \pm 3.5\cdot 10^{-2}$ & $1.3\cdot 10^{-2}\pm 9.4\cdot 10^{-3}$ & $4.8\cdot 10^{-2} \pm 3.2\cdot 10^{-2}$\\
\cmidrule(lr){1-1}\cmidrule(lr){2-5}
$\{5,5,6,0\}$ & $1.8 \cdot 10^{-3} \pm 6.1 \cdot 10^{-3}$ & $4.3 \cdot 10^{-2} \pm 3.0\cdot 10^{-1}$ & $1.5 \cdot 10^{-3} \pm 2.8 \cdot 10^{-3}$ & $2.3 \cdot 10^{-2} \pm 5.8\cdot 10^{-1}$\\
\cmidrule(lr){1-1}\cmidrule(lr){2-5}
$\{5,3,3,0\}$ & $2.7\cdot 10^{-3} \pm 2.4\cdot 10^{-3}$ & $2.3\cdot 10^{-2} \pm 2.5\cdot 10^{-2}$ & $2.5\cdot 10^{-3} \pm 1.5\cdot 10^{-3}$ & $9.7\cdot 10^{-2} \pm 1.5\cdot 10^{-1}$\\
\cmidrule(lr){1-1}\cmidrule(lr){2-5}
$\{10,3,3,0.01\}$ & $2.6\cdot 10^{-2} \pm 8.5\cdot 10^{-3}$ & $7.2\cdot 10^{-2} \pm 3.7\cdot 10^{-2}$ & $2.7\cdot 10^{-2} \pm 1.1\cdot 10^{-2}$ & $7.9\cdot 10^{-2} \pm 3.8\cdot 10^{-2}$ \\
\cmidrule(lr){1-1}\cmidrule(lr){2-5}
$\{10,3,3,0.05\}$ & $1.3\cdot 10^{-1} \pm 4.3\cdot 10^{-2}$ & $3.4\cdot 10^{-1} \pm 1.8\cdot 10^{-1}$ & $1.2\cdot 10^{-1} \pm 4.6\cdot 10^{-2}$ & $3.2\cdot 10^{-1} \pm 1.1\cdot 10^{-1}$\\
\cmidrule(lr){1-1}\cmidrule(lr){2-5}
$\{10,3,3,0.1\}$ & $2.6\cdot 10^{-1} \pm 1.0\cdot 10^{-1}$ & $7.0\cdot 10^{-1} \pm 3.7\cdot 10^{-1}$ & $2.2\cdot 10^{-1} \pm 1.1\cdot 10^{-1}$ & $5.8\cdot 10^{-1} \pm 2.4\cdot 10^{-1}$\\
\bottomrule
\end{tabular} 
}
\end{center}
\end{table}


\begin{figure}[!hbt]
\centering
\subfigure[FM:  vs learned kernel]{
\includegraphics[width=0.5\linewidth]{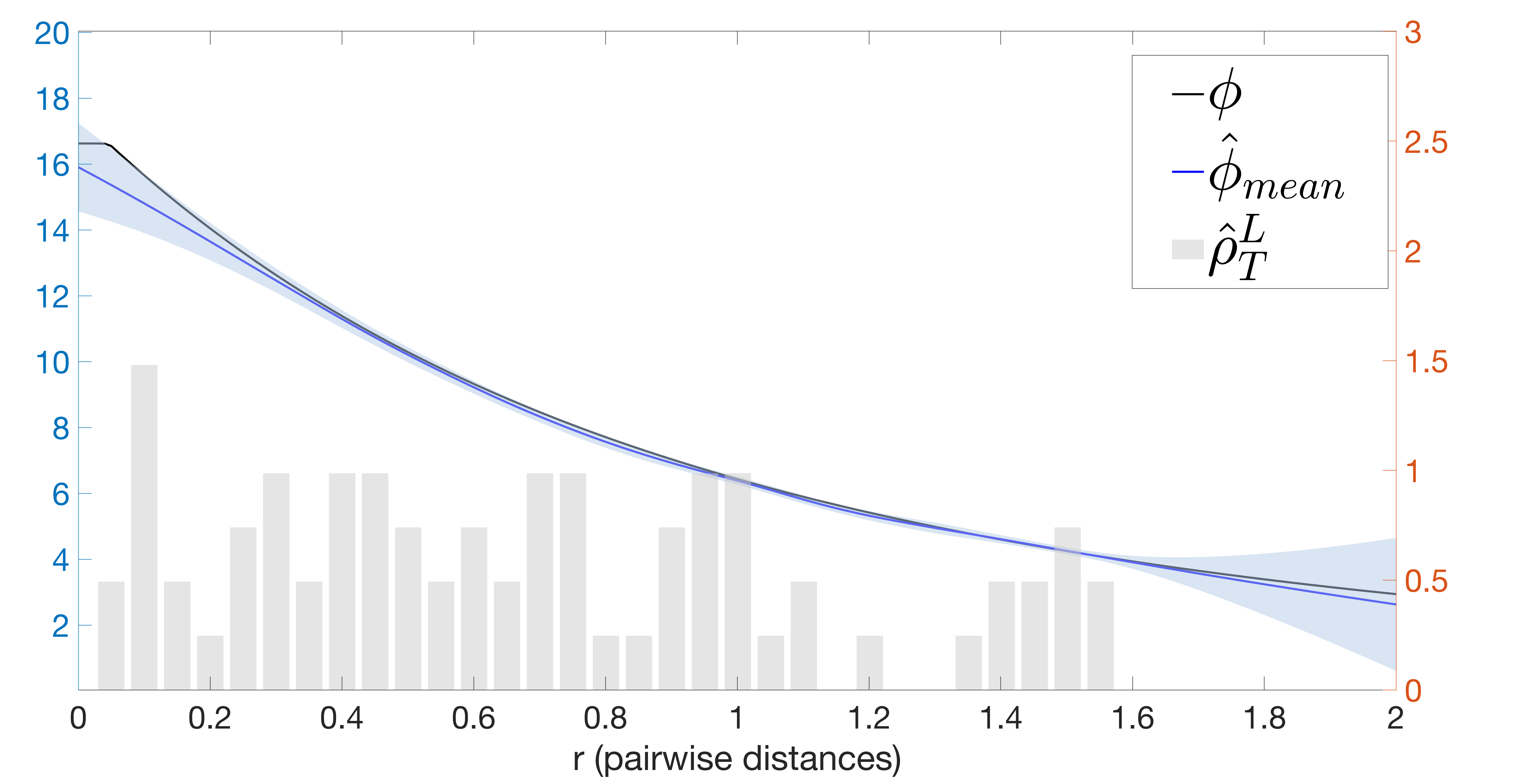}
}
\subfigure[$\{N,M,L,\sigma\} = \{5,5,6,0\}$]{
\includegraphics[width=0.47\linewidth]{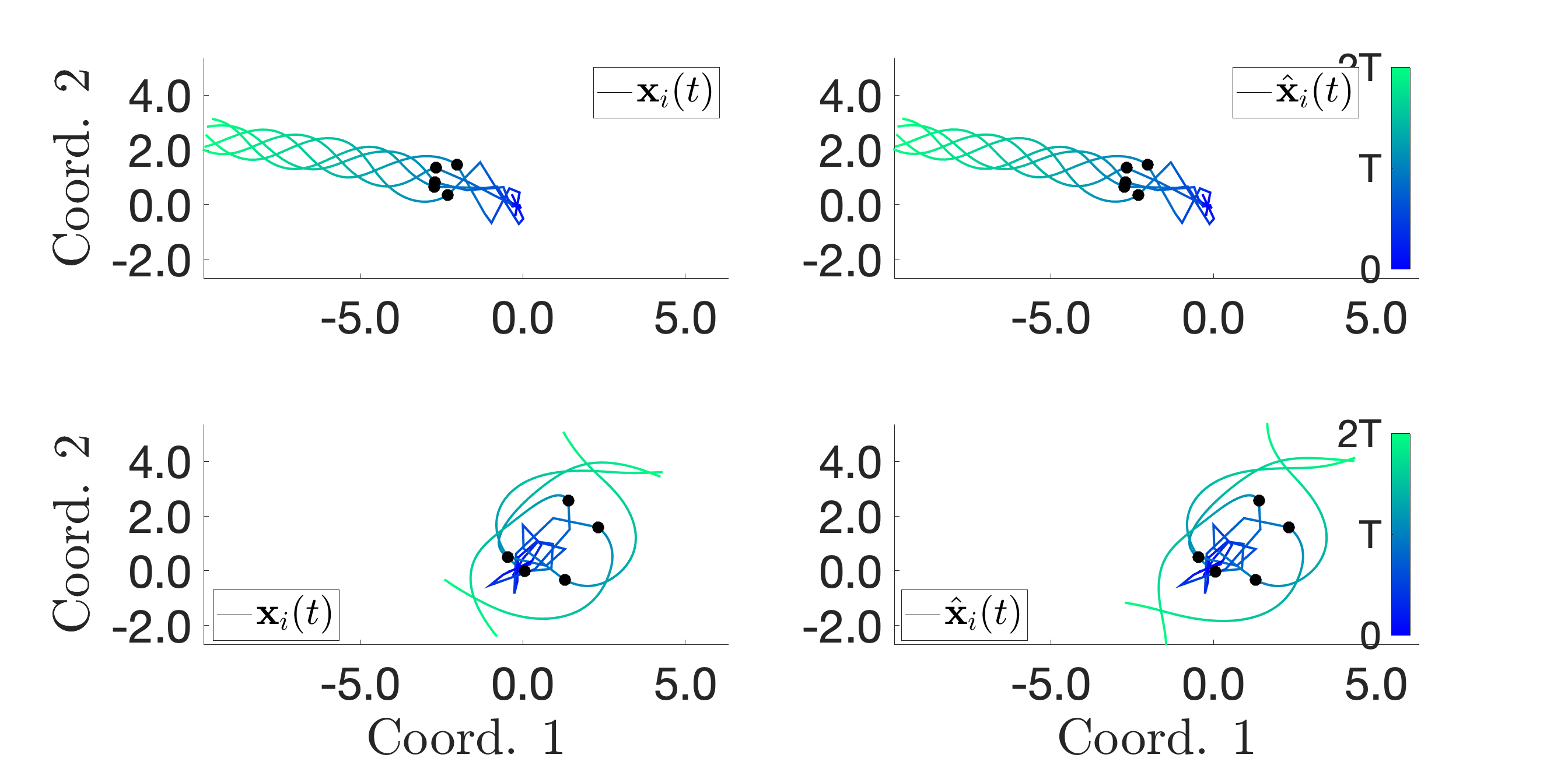}
}
\subfigure[$\{N,M,L,\sigma\} = \{10,3,3,0.1\}$]{
\includegraphics[width=0.47\linewidth]{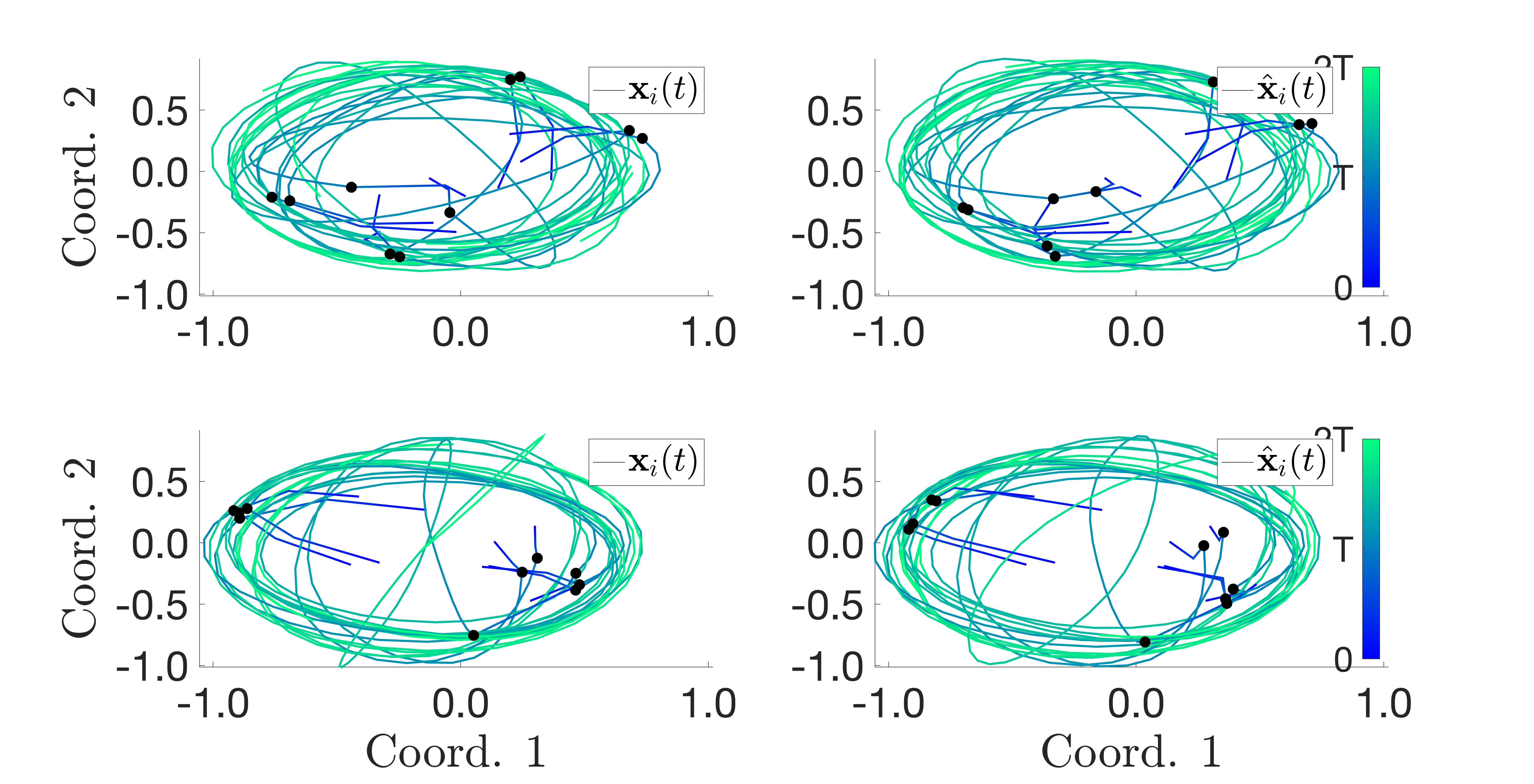}
}
\caption{Learning Fish Milling (FM) using the Matérn kernel. (a): $\phi$ versus the posterior mean $\hat{\intkernel}_{mean}$, with a two-standard-deviation band (light blue color) around the mean. The grey bars represent the histogram of pairwise distances, with density value shown in the orange axis. (b),(c): The true (left) versus predicted (right) trajectories using $\hat{\balpha}$ and $\hat{\intkernel}$  with initial conditions of training data (top) and testing data (bottom).}
\label{fig:ex_FM_traj}
\end{figure}

\paragraph{Baseline Comparisons}
For the case where $\{N,M,L,\sigma\}=\{5,1,9,0\}$, we compare our results with the SINDy model and the FNN models: for the SINDy model, we apply a reasonably large dictionary of monomials up to order 2 and sines and cosines of frequencies $\{k\}_{k=1}^{10}$, and fit the system $\dot{\bY} = [\bV, \bZ]^T = \tilde{\rhsfo}_{\intkernel}(\bY)$ with the same training data $\{\bbY_M, \bbZ_{\sigma^2,M}\}$; for the FNN model, we consider a three-layer FNN with $[40, 40, 20]$ hidden units. The results are shown in Figure \ref{fig:trajs_appendix} (a),(b) and Table \ref{tab:FM_traj_compare}. Since we only train the models with a small amount of data (9 observations, 0:0.625:5 in the training time interval [0,5]), both SINDy and FNN fail to provide accurate trajectory predictions on the training time interval and perform even worse in the testing time interval [5,10], since they do not include the physical information of the fish milling system in contrast to our method.

\begin{figure*}[t]
\centering
\subfigure[True versus Our Model in FM dynamics]{
\includegraphics[width=0.45\textwidth]{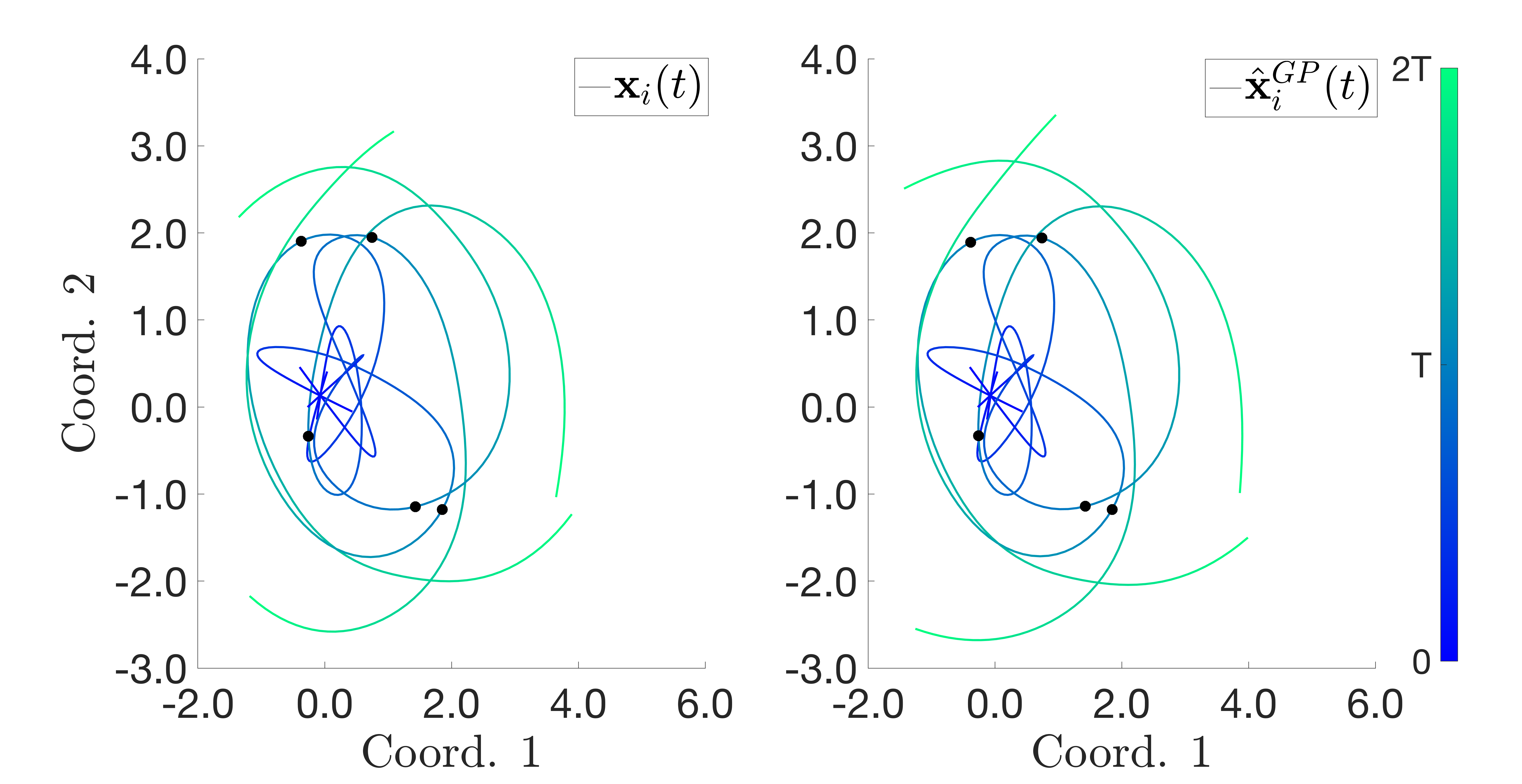}
}
\subfigure[Comparison with SINDy and FNN in FM dynamics]{
\includegraphics[width=0.45\textwidth]{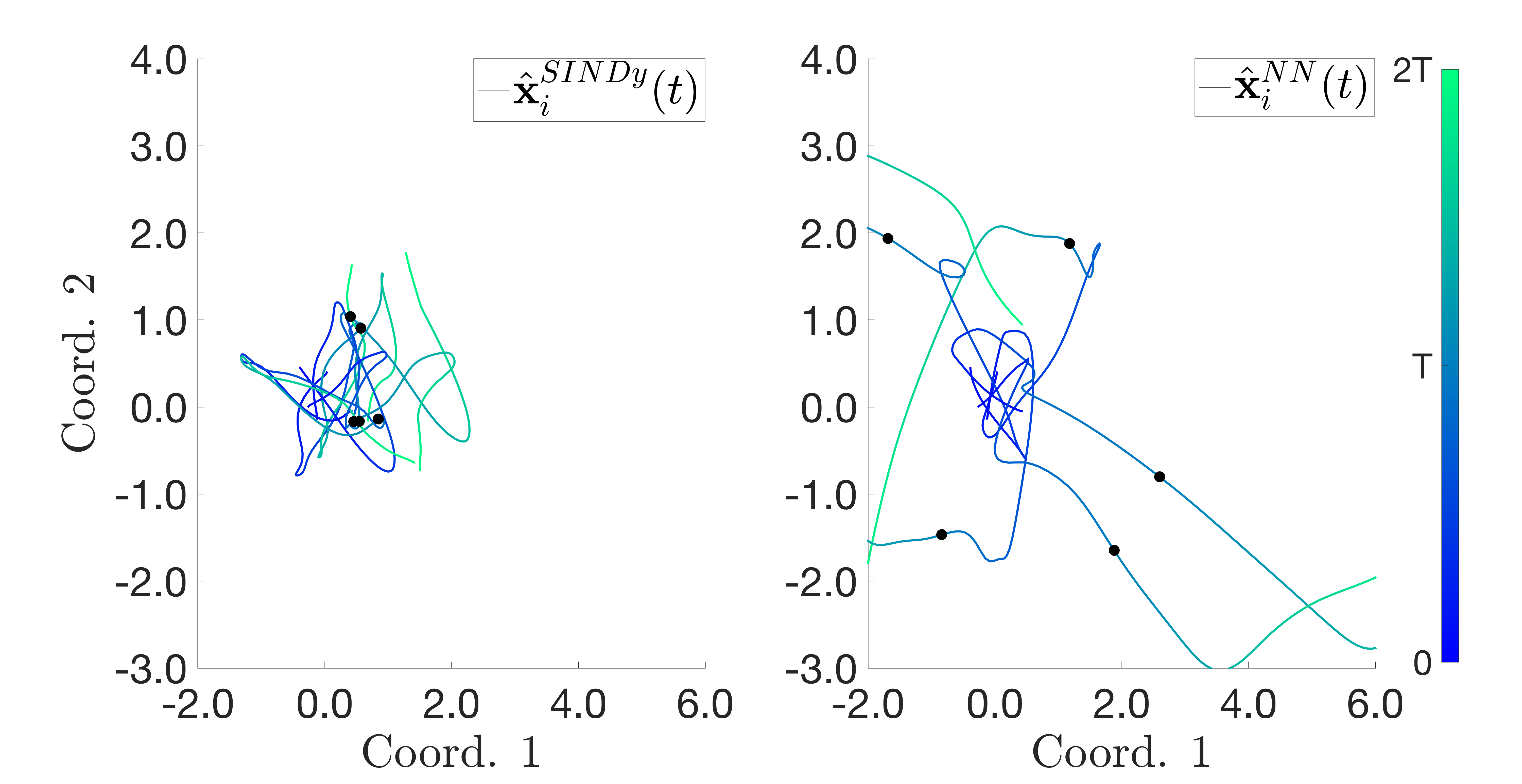}
}
\caption{Top: Learning FM dynamics from training data $\{N,M,L,\sigma\} = \{5,1,9,0.1\}$. (a): the true trajectory versus the prediction from our GP model. (b): the SINDy model and the FNN model trajectories.}
\label{fig:trajs_appendix}
\end{figure*}

\begin{table}[!htb]
\begin{center}
\caption{Baseline comparison. The relative trajectory prediction errors.} \label{tab:FM_traj_compare}
\vspace{0.1in}
\begin{tabular}{c c c}
\toprule
Approach & Training time interval &  Testing time interval\\
\cmidrule(lr){1-1}\cmidrule(lr){2-3}
GPs  & $\mathbf{3.6\cdot 10^{-3} \pm 2.5\cdot 10^{-3}}$ & $\mathbf{2.4\cdot 10^{-1} \pm 3.1\cdot 10^{-1}}$ \\
\cmidrule(lr){1-1}\cmidrule(lr){2-3}
SINDy  & $9.4\cdot 10^{-1} \pm 3.8\cdot 10^{-1}$ & $1.2\cdot 10^{0} \pm 4.7\cdot 10^{-1}$\\
\cmidrule(lr){1-1}\cmidrule(lr){2-3}
FNN  & $2.2\cdot 10^{0} \pm 1.3\cdot 10^{0}$ & $3.1\cdot 10^{0} \pm 1.7\cdot 10^{0}$\\
\bottomrule
\end{tabular} 
\end{center}
\end{table}

\paragraph{Comparison with the previous methods}  Here, we also provide another comparison with one recently proposed method on the data-driven discovery of interacting particle system \citep{lu2019nonparametric,zhong2020data}, where they also incorporate the information of the physical structures of the model but only assume $\intkernel$ is unknown in the system, and $\balpha$ is given. Therefore, to compare with this previous method, we fixed the parameters $(\balpha, \theta)$ in our model with the true parameters $\balpha$ and a guessed parameter for $\theta$, and we predict $\intkernel$ with the fixed parameters $(\balpha, \theta)$ without training procedure.

 We consider the case where $\{N,M,L\} = \{5,5,6\}$ and $\sigma = 0$ or  $0.01$. Using the previous method, we apply piecewise linear polynomials with $n = 18$ basis functions to approximate $\intkernel$ on the support $[0,R] = [0,4.66]$ by solving a least square problem, while we consider the function space $\mathcal{H}_{\mathrm{Mat\acute{e}rn}}$ with $\theta = (10^2,0.1)$ using our new method. By Theorem 14, the posterior mean estimator obtained is  approximately the least-square solution when the noise is very small. Noise in fact serves a role of regularization in our method. We compare the errors of $\hat{\intkernel}$ using the {supremum} norm (post-smoothing techniques are applied to piecewise linear estimators) and present the relative trajectory prediction errors for both estimators. 

\begin{table}[!htb]
\begin{center}
\caption{The trajectory prediction errors for different models of FM dynamics $\{N,M,L\} = \{5,5,6\}$.} \label{tab:FM_C_traj}
\vspace{0.1in}
\resizebox{\linewidth}{!}{%
\begin{tabular}{ccccccc}
\toprule[.05cm]
Method & $\sigma$ & $\| \hat{\intkernel} - \intkernel\|_{\infty}$ & Training IC $[0,5]$ & Training IC $[5,10]$ & New IC $[0,5]$ & New IC $[5,10]$\\
\cmidrule(lr){1-2}\cmidrule(lr){3-7}
previous & 0 & $1.9\cdot 10^{-1} \pm 2.9\cdot 10^{-2}$ & $1.3\cdot 10^{-1} \pm 2.8\cdot 10^{-2}$ & $2.7\cdot 10^{-1} \pm 1.1\cdot 10^{-1}$  & $1.4\cdot 10^{-1} \pm 3.8\cdot 10^{-2}$ & $3.5\cdot 10^{-1} \pm 1.5\cdot 10^{-1}$\\
\cmidrule(lr){1-2}\cmidrule(lr){3-7}
now & 0 & $\mathbf{3.4\cdot 10^{-2} \pm 5.5 \cdot 10^{-3}}$ & $\mathbf{2.4\cdot 10^{-3} \pm 2.5\cdot 10^{-3}}$ &  $\mathbf{1.5\cdot 10^{-1} \pm 5.6\cdot 10^{-1}}$ & $\mathbf{2.0\cdot 10^{-3} \pm 1.4\cdot 10^{-3}}$ & $\mathbf{1.9\cdot 10^{-1} \pm 4.9\cdot 10^{-1}}$\\
\cmidrule(lr){1-2}\cmidrule(lr){3-7}
previous & 0.01 & $1.9\cdot 10^{-1} \pm 3.5\cdot 10^{-2}$ & $1.5\cdot 10^{-1} \pm 3.6\cdot 10^{-2}$ & $2.8\cdot 10^{-1} \pm 9.4\cdot 10^{-2}$ & $1.3\cdot 10^{-1} \pm 3.4\cdot 10^{-2}$ & $3.1\cdot 10^{-1} \pm 1.5\cdot 10^{-1}$\\
\cmidrule(lr){1-2}\cmidrule(lr){3-7}
now & 0.01 & $\mathbf{2.8\cdot 10^{-2} \pm 1.6\cdot 10^{-2}}$ & $\mathbf{6.9\cdot 10^{-3} \pm 2.5\cdot 10^{-2}}$ & $\mathbf{1.1\cdot 10^{-1} \pm 8.3\cdot 10^{-1}}$ & $\mathbf{3.6\cdot 10^{-3} \pm 2.6\cdot 10^{-3}}$ & $\mathbf{1.5\cdot 10^{-1} \pm 7.2\cdot 10^{-1}}$\\
\bottomrule
\end{tabular} 
}
\end{center}
\end{table}

From the results shown in Table \ref{tab:FM_C_traj}, we can see that our new approach has better performance than the previous approach, given the limited size of training data. One drawback of the previous approach lies in selecting the optimal number of bases to minimize the error. We have tried different $n$ (up to $100$) for the previous approach, but none of them can provide a significantly better result than what we have shown here.  In contrast, the training step of our approach automatically chooses a basis by updating the prior. For noisy data, our approach  is equivalent to regularized least squares. The numerical results also show that regularization improves prediction accuracy.

\paragraph{Other collective patterns.} The FM system can also display other collective patterns such as double ring and symmetric escape dynamics. Below, we also display the learning results to show our approach can faithfully learn and predict the ground truth from a small set of noisy data in different scenarios and for systems with larger dimensions.


\begin{figure}[!hbt]
\centering
\subfigure[$\{N,M,L,\sigma\} = \{30,2,6,0.1\}$. Double Ring Pattern.]{
\includegraphics[width=0.45\linewidth]{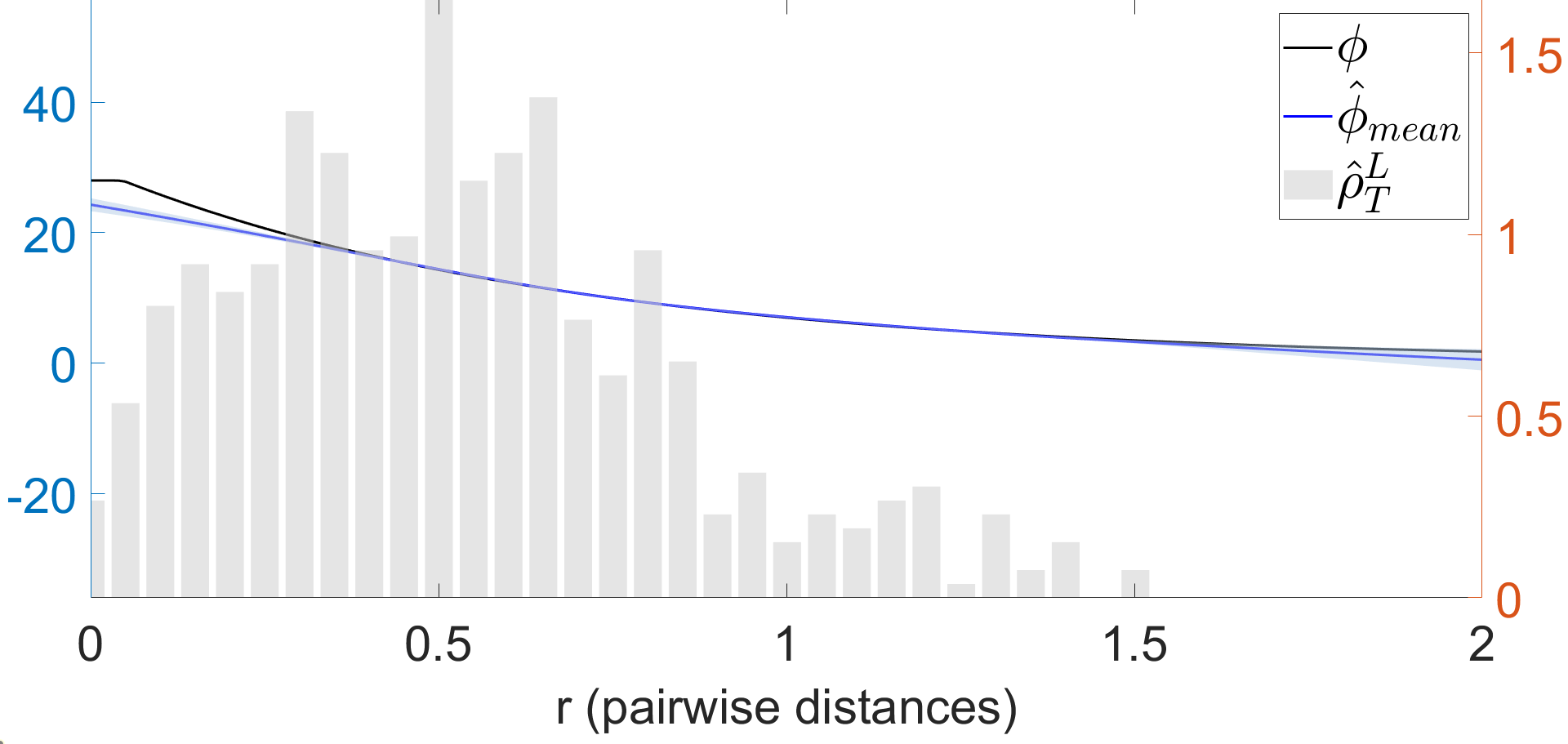}
\hspace{0.3cm}
\includegraphics[width=0.47\linewidth]{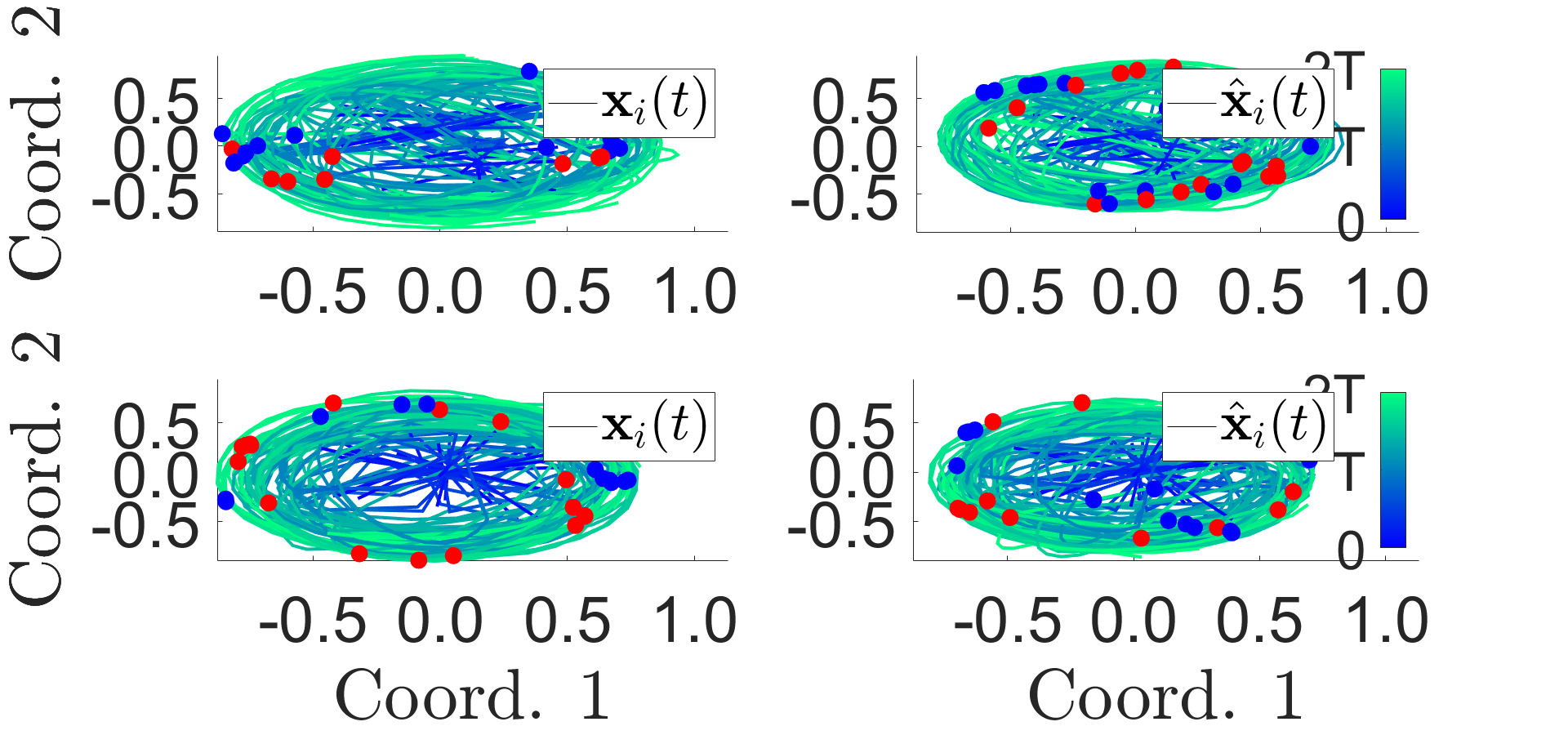}
}

\subfigure[$\{N,M,L,\sigma\} = \{30,2,6,0.1\}$. Symmetric Escape Pattern.]{
\includegraphics[width=0.45\linewidth]{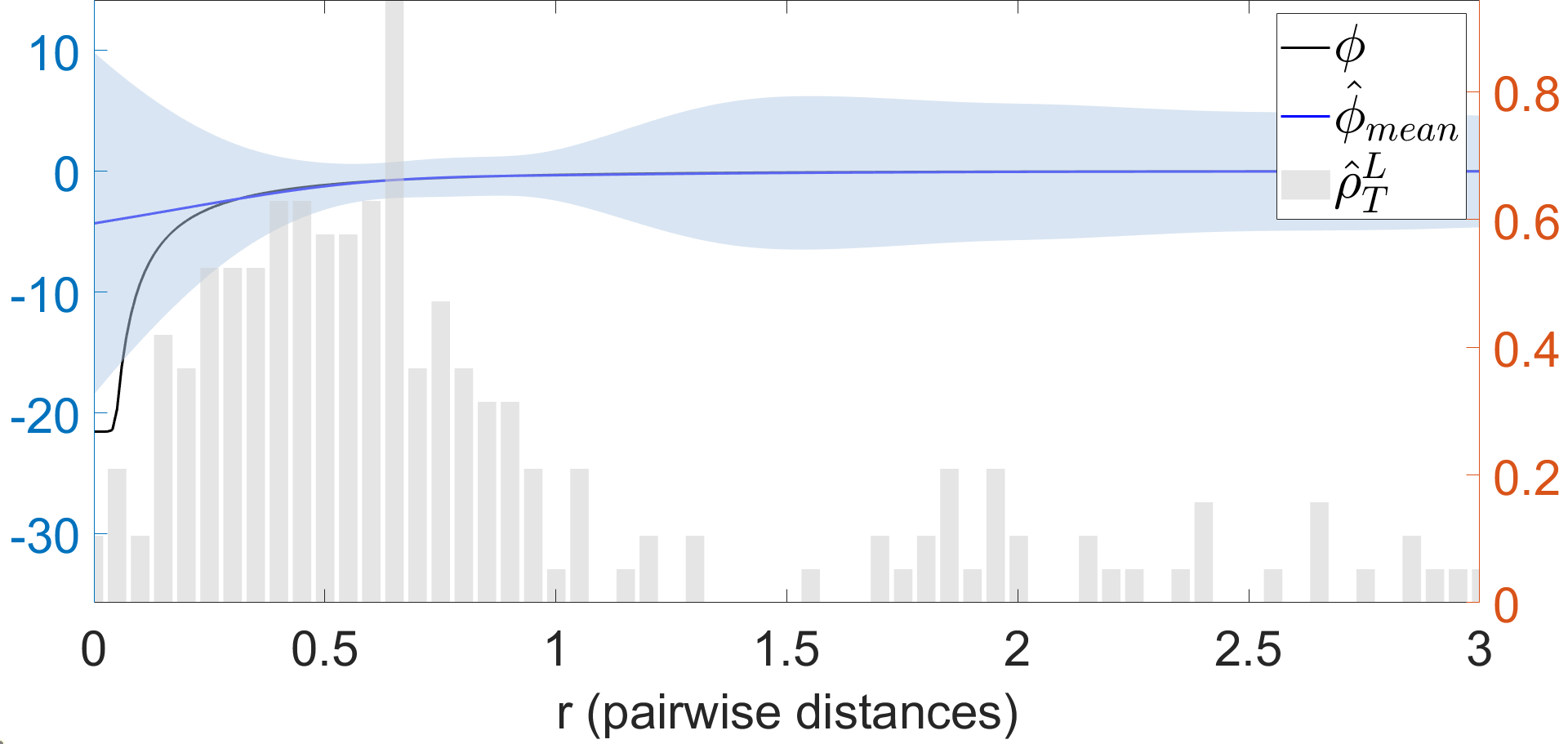}
\hspace{0.3cm}
\includegraphics[width=0.47\linewidth]{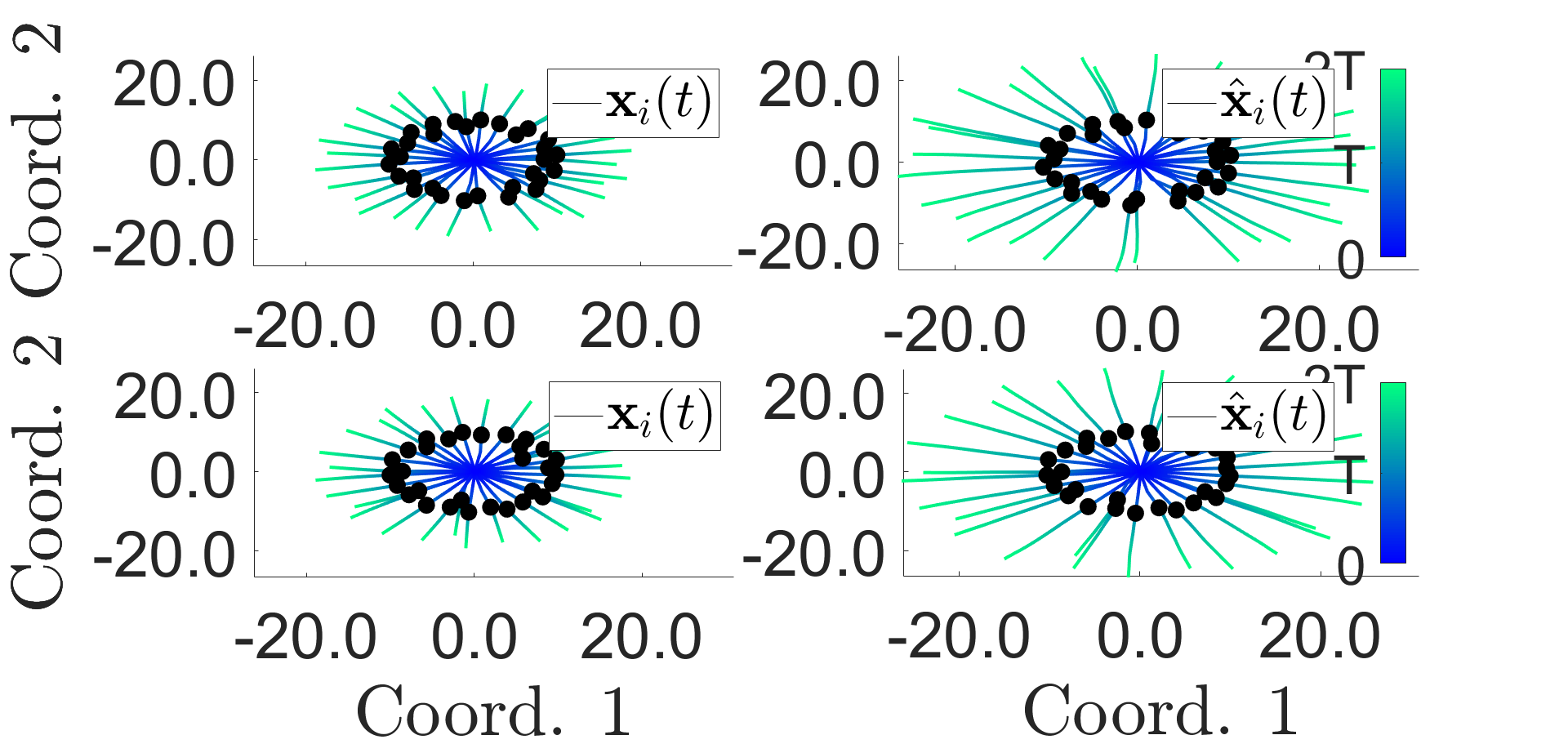}
}


\caption{Learning Fish Milling (FM) using the Matérn kernel for systems with $N = 30$. The true versus predicted kernel (left), and the true versus predicted trajectories (right). Each shows unique collective behavior.}
\label{fig:FM_traj_escape_1}
\end{figure}

Our method faithfully captures the behavior of the system and is capable of robust prediction. While we expect the error in our approximation of the true interaction force, our predictions reflect well the general dynamics and preserve the critical topological properties of each pattern. The learning errors are summarized in Table \ref{tab:FMHD}.

In the double ring pattern, we color counterclockwise orbiting agents as red and clockwise orbiting agents as blue. We can see the mixture of directions of orbit characteristic of the double ring pattern. While prediction errors occur in the exact position of the agents, our method has great success in faithfully predicting the orbit type. With very small amounts of data ($M = 2$), our method predicts only one orbit direction incorrectly among all predictions.

In the symmetric escape pattern, we have a repulsive force under which the agents escape outward with straight trajectories. Our method captures this behavior with very little error, despite the vanishing of learning information as $r$ becomes close to $0$.

\begin{table}[!htb]
\begin{center}
\caption{Means and standard deviations of the error of $\hat\balpha$ and $\hat\phi$ for above patterns} \label{tab:FMHD}
\vspace{0.1in}
\begin{tabular}{cccc}
\toprule[.05cm]
Pattern & $\{C_{rp},l_{rp}, C_a,l_a\}$  & $\| \hat{\balpha} - \balpha\|_{\infty}$  & $\| \hat{\intkernel} - \intkernel\|_{\infty}$ \\
\cmidrule(lr){1-2}\cmidrule(lr){3-4}
Double Ring & $\{0.5,0.5,1,1\}$ & $3.17 \cdot 10^{-1} \pm 4.7 \cdot 10^{-1}$ & $2.30 \cdot 10^{-1} \pm 8.8 \cdot 10^{-2}$ \\
\cmidrule(lr){1-2}\cmidrule(lr){3-4}
Symmetric Escape & $\{2,0.9,1,1\}$ & $8.50 \cdot 10^{-1} \pm 8.1 \cdot 10^{-2}$ & $ 7.15 \cdot 10^{-1} \pm 3.3 \cdot 10^{-2}$ \\
\bottomrule
\end{tabular} 
\end{center}
\end{table}

The above shows the kernel learning errors. For this section, we only use $20$ iterations of the maximum likelihood method, converging very quickly to a close approximation. This exhibits the power of even very few iterations of this methodology.

\section{Conclusion}

We have considered the inverse problem in 
a particular yet widely used set of interacting agent systems.  We provide a GP based approach that converges optimally and avoids the
curse of dimensionality, by exploiting the multiple symmetries in these systems. Extensions of our learning approach to more general systems, such as heterogeneous systems with multiple types of agents and external potentials and stochastic systems may be considered. Another direction is to consider the collective inference problem when only distributions of trajectory data are provided. 

   

\appendix

\section{Detailed proofs for lemmas and theorems}\label{proofsappendix}

\begin{proof}[Proof of Lemma \ref{varinequality}]
Note that $\big\|K_{r}\big\|_{\mH} \leq \kappa$ for any $r\in [0,R]$, we have that  
\begin{align*}
\|B_{M}\intkernelvar\|_{\mH} &\leq \frac{1}{LM}\sum_{l,m=1}^{L,M}\sum_{i=1,i', i'' \neq i}^{N}\frac{1}{N^3}\big\|K_{r_{ii'}^{(m,l)}}\big\|_{\mH} \|\intkernelvar\|_{\infty} R^2\\&\leq \kappa \|\intkernelvar\|_{\infty} R^2, a.s.
\end{align*}
For the second inequality, we have that 
\begin{align*}
\E\|B_{M}\intkernelvar\|_{\mH}^2&=\E \langle A^*_{M}A_{M}\intkernelvar, A^*_{M}A_{M}\intkernelvar\rangle_{\mH}= \langle A^*A\intkernelvar, B\intkernelvar\rangle_{\mH}=\langle A\intkernelvar, AB\intkernelvar\rangle_{L^2(\mbf{\rho}_X)}\\
&\leq \| A\intkernelvar\|_{L^2(\mbf{\rho}_X)}\| AB\intkernelvar\|_{L^2(\mbf{\rho}_X)}\leq \| \intkernelvar\|_{L^2(\tilde \rho_T^L)} \| B\intkernelvar\|_{L^2(\tilde \rho_T^L)}\\
&\leq \|B\|_{L^2(\tilde \rho_T^L)}\| \intkernelvar\|_{L^2(\tilde \rho_T^L)} ^2 \leq \kappa^2 R^2\| \intkernelvar\|_{L^2(\tilde \rho_T^L)} ^2,
\end{align*} where we use Lemma \ref{lem1} and Equation \eqref{opinequality2}. 
\end{proof}

\begin{proof}[Proof of Lemma \ref{decomp} ]Define the $\mH$-valued random variable $$\xi^{(m)}=\frac{1}{L}\sum_{l=1}^{L}\sum_{i=1,i', i'' \neq i}^{N}\frac{1}{N^3}K_{r_{ii'}^{(m,l)}} \langle \intkernelvar ,K_{r_{ii''}^{(m,l)}} \rangle_{\mH} \langle \br_{ii'}^{(m,l)},\br_{ii''}^{(m,l)}\rangle.$$ Then the random variables $\{\xi^{(m)}\}_{m=1}^{M}$ are i.i.d. According to Lemma \ref{varinequality}, we have that 
\begin{align*}
\|\xi^{(m)}\|_{\mH}&\leq \kappa R^2 \|\intkernelvar\|_{\infty},\\
\E\|\xi^{(m)}\|_{\mH}^2&\leq \kappa^2 R^2\rhotnorm{\intkernelvar}.
\end{align*} Note that $B_M\intkernelvar-B\intkernelvar=\frac{1}{M}\sum_{m=1}^M (\xi^{(m)}-\E(\xi^{(m)})).$ The conclusion follows by applying Lemma \ref{coninequality2} to $\{\xi^{(m)}\}_{m=1}^{M}$. 
\end{proof}

\begin{proof} [Proof of Lemma \ref{sampleerror}] We introduce an intermediate quantity $(B_M+\lambda)^{-1}B\intkernelvar$ and decompose 
\begin{align*}
&(B_M+\lambda)^{-1}B_M\intkernelvar-(B+\lambda)^{-1}B\intkernelvar \\
=&(B_M+\lambda)^{-1}B_M\intkernelvar-(B_M+\lambda)^{-1}B\intkernelvar+(B_M+\lambda)^{-1}B\intkernelvar- (B+\lambda)^{-1}B\intkernelvar.
\end{align*}
Since $\|(B_M+\lambda)^{-1}\|_{\mH}\leq \frac{1}{\lambda}$, we have that 
\begin{align*}
\|(B_M+\lambda)^{-1}B_M\intkernelvar-(B_M+\lambda)^{-1}B\intkernelvar\|_{\mH}\leq \frac{1}{\lambda} \|B_M\intkernelvar-B\intkernelvar\|_{\mH}.
\end{align*}

Applying Lemma \ref{decomp} to $B_M\intkernelvar-B\intkernelvar$, we obtain with probability at least $1-\delta/2$
\begin{align*}
\frac{1}{\lambda}\|B_M\intkernelvar-B\intkernelvar\|_{\mH} &\leq \frac{4\kappa R^2 \|\intkernelvar\|_{\infty} \log(4/\delta)}{\lambda M}+ \kappa R \rhotnorm{\intkernelvar} \sqrt{ \frac{2\log(4/\delta)}{\lambda^2 M}}\\ &\leq \frac{4\kappa R^2 \|\intkernelvar\|_{\infty} \log(4/\delta)}{\lambda M}+ \kappa R^2 \|\intkernelvar\|_{\infty} \sqrt{ \frac{2\log(4/\delta)}{\lambda^2 M}}.
\end{align*}

On the other hand, we have
\begin{align*}
\|(B_M+\lambda)^{-1}B\intkernelvar- (B+\lambda)^{-1}B\intkernelvar\|_{\mH} &= \|(B_M+\lambda)^{-1}(B-B_M)(B+\lambda)^{-1}B\intkernelvar\|_{\mH}\\ 
&\leq \frac{1}{\lambda}\|(B-B_M)(B+\lambda)^{-1}B\intkernelvar\|_{\mH}.
\end{align*}

Since $\intkernelvar^{\lambda,\infty}_{\mH}=(B+\lambda)^{-1}B\intkernelvar$ is the unique minimizer of the expected risk functional $ \mE(\psi)=\|A{\psi}-A{\intkernelvar}\|^2_{L^2(\rho_{\bX})}+\lambda\|\psi\|_{\mH}^2,$ plugging $\psi=0$, we obtain that 
$$\|A{\intkernelvar^{\lambda,\infty}_{\mH}}-A{\intkernelvar}\|^2_{L^2(\rho_{\bX})}+\lambda\|\intkernelvar^{\lambda,\infty}_{\mH}\|_{\mH}^2 <\| A{\intkernelvar}\|^2_{L^2(\rho_{\bX})},$$ which implies that 
\begin{align}
\label{eq1}\|\intkernelvar^{\lambda,\infty}_{\mH}\|_{\mH}&\leq \frac{1}{\sqrt{\lambda}}\| A{\intkernelvar}\|_{L^2(\rho_{\bX})},\\
\|A{\intkernelvar^{\lambda,\infty}_{\mH}}\|^2_{L^2(\rho_{\bX})} &\leq 2\|A{\intkernelvar}\|^2_{L^2(\rho_{\bX})}.
\end{align}

By Lemma \ref{infbound} and \eqref{eq1}, it follows that
\begin{equation}
\label{eq2}
  \infnorm{\intkernelvar^{\lambda,\infty}_{\mH}}\leq \kappa \|\intkernelvar^{\lambda,\infty}_{\mH}\|_{\mH} \leq \frac{\kappa}{\sqrt{\lambda}}\|A\intkernelvar\|_{L^2(\rho_{\bX})}.
\end{equation}

Suppose the coercivity condition \eqref{coercivity} holds. We have that \begin{equation}
\label{eq3}
  \rhotnorm{\intkernelvar^{\lambda,\infty}_{\mH}}^2\leq \frac{1}{c_{\mH}}\|A{\intkernelvar^{\lambda,\infty}_{\mH}}\|_{L^2(\rho_{\bX})}^2 \leq \frac{2}{c_{\mH}}\|A{\intkernelvar}\|_{L^2(\rho_{\bX})}^2.
\end{equation}
Note that $ \Rhoxnorm{A\intkernelvar}^2< R^2\|\intkernelvar\|_{\infty}^2$ (see \eqref{opinequality}).
Applying Lemma \ref{decomp} to $\intkernelvar^{\lambda,\infty}_{\mH}=(B+\lambda)^{-1}B\intkernelvar$, and using \eqref{eq2} and \eqref{eq3} , we obtain that with probability at least $1-\delta/2$,
\begin{align*}
\frac{1}{\lambda}\|(B-B_M)(B+\lambda)^{-1}B\intkernelvar\|_{\mH} & \leq \frac{4\kappa R^2 \|\intkernelvar^{\lambda,\infty}_{\mH}\|_{\infty} \log(4/\delta)}{\lambda M}+ \kappa R \rhotnorm{\intkernelvar^{\lambda,\infty}_{\mH}}
\sqrt{ \frac{2\log(4/\delta)}{\lambda^2 M}}\\
&\leq \frac{4\kappa^2 R^3 \|\intkernelvar\|_{\infty} \log(4/\delta)}{\lambda^{\frac{3}{2}} M}+ \frac{\sqrt{2}}{\sqrt{c_{\mH}}}\kappa^2 R \rhotnorm{\intkernelvar} \sqrt{ \frac{2\log(4/\delta)}{\lambda^2 M}}\\ &\leq \frac{4\kappa^2 R^3 \|\intkernelvar\|_{\infty} \log(4/\delta)}{\lambda^{\frac{3}{2}} M}+ \frac{\sqrt{2}}{\sqrt{c_{\mH}}} \kappa^2 R^2 \|\intkernelvar\|_{\infty} \sqrt{ \frac{2\log(4/\delta)}{\lambda^2 M}}.
\end{align*}

Finally, by combining two bounds, we obtain that with a probability at least $1-\delta$
\begin{align*}
&\|(B_M+\lambda)^{-1}B_M\intkernelvar-(B_M+\lambda)^{-1}B\intkernelvar\|_{\mH} \\ &\leq \frac{\kappa R^2\|\intkernelvar\|_{\infty}\sqrt{2\log(4/\delta)}}{\sqrt{M}\lambda}\bigg[ (\kappa+1)\sqrt{\frac{2}{c_{\mH}}}+ \frac{(\kappa R+\sqrt{\lambda})\sqrt{2\log(4/\delta)}}{\sqrt{M\lambda}} \bigg]\\&\leq \frac{\kappa R^2\|\intkernelvar\|_{\infty}\sqrt{2\log(4/\delta)}}{\sqrt{M}\lambda} (C_{\kappa,{\mH}}+\frac{C_{\kappa,R,\lambda}\sqrt{2\log(4/\delta)}}{\sqrt{M\lambda}} ).
\end{align*} where $C_{\kappa,{\mH}}=(\kappa+1)\sqrt{\frac{2}{c_{\mH}}}$ and $C_{\kappa,R,\lambda}=\kappa R +\sqrt{\lambda}$.
\end{proof}

\begin{proof}[Proof of Theorem \ref{hbound}] We decompose $\intkernel_{\mH}^{\lambda,M}-\intkernel_{\mH}^{\lambda,\infty}=\intkernel_{\mH}^{\lambda,M}-\tilde\intkernel_{\mH}^{\lambda,M}+\tilde\intkernel_{\mH}^{\lambda,M}-\intkernel_{\mH}^{\lambda,\infty}$ where $\tilde\intkernel_{\mH}^{\lambda,M}$ is the empirical minimizer for noise-free observations. Then applying Lemma \ref{decomp} to the term $\tilde\intkernel_{\mH}^{\lambda,M}-\intkernel_{\mH}^{\lambda,\infty}$, we obtain that with probability at least $1-\delta$,
\begin{align}\label{trian1}
\|\tilde\intkernel_{\mH}^{\lambda,M}-\intkernel_{\mH}^{\lambda,\infty}\|_{\mH} \leq \frac{\kappa R^2\|\intkernele\|_{\infty}\sqrt{2\log(4/\delta)}}{\sqrt{M}\lambda}(C_{\kappa,{\mH}}+\frac{C_{\kappa,R,\lambda}\sqrt{2\log(4/\delta)}}{\sqrt{M\lambda}} ).
\end{align}
We now just need to estimate the ``noise part'' $\intkernel_{\mH}^{\lambda,M}-\tilde\intkernel_{\mH}^{\lambda,M}$. According to \eqref{em}, 
\begin{align}\label{em1}
\tilde\intkernel_{\mH}^{\lambda,M}-\intkernel_{\mH}^{\lambda,M}=(B_M+\lambda)^{-1}A_{M}^{*}\mathbb{W}_M,
\end{align} where the noise vector $\mathbb{W}_M$ follows a multivariate Gaussian distribution with zero mean and variance $\sigma^2I_{dNML}$. Note that
\begin{align*}
\|\tilde\intkernel_{\mH}^{\lambda,M}-\intkernel_{\mH}^{\lambda,M}\|_{\mH}^2 &= \langle \mathbb{W}_M, A_M(B_M+\lambda)^{-2}A_M^*\mathbb{W}_M\rangle\\
&= \mathbb{W}_M^T\Sigma_M \mathbb{W}_M,
\end{align*} where the matrix $$\Sigma_M= (\mK_{\rhsfo_\intkernele}(\bbX_M,\bbX_M) +\lambda NdML I)^{-1}\mK_{\rhsfo_\intkernele}(\bbX_M,\bbX_M) (\mK_{\rhsfo_\intkernele}(\bbX_M,\bbX_M) +\lambda dNML I)^{-1}.$$  

The matrix $\Sigma_M$ is the matrix form of the operator $A_M(B_M+\lambda)^{-2}A_M^*$, as is derived from \eqref{em}, \eqref{solution1} and \eqref{id}. We have that
\begin{align*}
\mathrm{Tr}(\Sigma_M)&\leq \frac{1}{\lambda^2(MLNd)^2}\mathrm{Tr}(\mK_{\rhsfo_\intkernele}(\bbX_M,\bbX_M) )\\
&= \frac{1}{\lambda^2(MLNd)^2} \sum_{m=1,l=1,i=1}^{M,L,N} \frac{1}{N^2}\sum_{k\neq i,k'\neq i}\mK(r_{ik}^{(m,l)},r_{ik'}^{(m,l)})(\br_{ik'}^{(m,l)})^T \br_{ik}^{(m,l)}\\
&\leq \frac{1}{\lambda^2d^2MLN}\kappa^2R^2, a.s.
\end{align*}
and 
\begin{align*}
\mathrm{Tr}(\Sigma_M^2)&\leq \frac{1}{\lambda^4(MLNd)^4}\mathrm{Tr}(\mK_{\rhsfo_\intkernele}(\bbX_M,\bbX_M)^2 )\\
&= \frac{1}{\lambda^4(MLNd)^4} \sum_{m,m'=1,l,l'=1,i,i'=1}^{M,L,N} \bigg\| \frac{1}{N^2}\sum_{k\neq i,k'\neq i'}\mK(r_{ik}^{(m,l)},r_{i'k'}^{(m',l')}) \br_{ik}^{(m,l)}(\br_{i'k'}^{(m',l')})^T\bigg\|_F^2\\
&\leq \frac{\kappa^4R^4}{\lambda^4d^4(MLN)^2}, a.s.
\end{align*}

Now we apply the Hanson-Wright inequality (Theorem \ref{HAnson}) for the Gaussian random vector $\mathbb{W}_M$ with $S_0=\sigma^2$. Note that for any $\epsilon>0$, 
 \begin{align*}
 \min \bigg\{ \frac{\epsilon^2}{\sigma^4\|\Sigma_M\|_{\mathrm{HS}}^2}, \frac{\epsilon}{\sigma^2\|\Sigma_M\|}\bigg\} &\geq \min \bigg\{ \frac{\epsilon^2}{\sigma^4\mathrm{Tr}(\Sigma_M^2)}, \frac{\epsilon}{\sigma^2 \mathrm{Tr}(\Sigma_M)}\bigg\},
\end{align*} 
we obtain that, with a probability at least $1-e^{-t^2}$,  
 \begin{align*}
\mathbb{W}_M^T\Sigma_M \mathbb{W}_M &\leq \frac{1}{c}\sigma^2\max\{\mathrm{Tr}(\Sigma_M),\sqrt{\mathrm{Tr}(\Sigma_M^2)}\}(1+2t+t^2)\\
 &\leq \frac{\kappa^2R^2\sigma^2}{c\lambda^2 d^2 {MLN} }(1+2t+t^2)
 \end{align*} for any $t>0$, where $c$ is an absolute positive constant appearing in the Hanson-Wright inequality. Therefore, with a probability at least $1-\delta$, there holds 
 \begin{align}\label{trian2}
 \|\tilde\intkernel_{\mH}^{\lambda,M}-\intkernel_{\mH}^{\lambda,M}\|_{\mH} \leq \frac{\kappa R\sigma(\log(1/\delta)+1)}{\sqrt{c}\lambda d \sqrt{MLN}}<\frac{2\kappa R\sigma \log(4/\delta)}{\sqrt{c}\lambda d \sqrt{MLN}}.
 \end{align} 
 
 Now combining \eqref{trian1} and \eqref{trian2}, we obtain that with probability at least $1-\delta$,
 $$\|\intkernel_{\mH}^{\lambda,M}-\intkernel_{\mH}^{\lambda,\infty}\|_{\mH} \leq  \frac{\kappa R^2\|\intkernele\|_{\infty}\sqrt{2\log(8/\delta)}}{\sqrt{M}\lambda}(C_{\kappa,{\mH}}+\frac{C_{\kappa,R,\lambda}\sqrt{2\log(8/\delta)}}{\sqrt{M\lambda}} ) + \frac{2\kappa R\sigma \log(8/\delta)}{\sqrt{c}\lambda d \sqrt{MLN}}.$$
\end{proof}

\section{Auxiliary lemmas and theorems}

\begin{lemma}\label{lemma: conditioning Gaussian}
Let $\bx$ and $\by$ be jointly Gaussian random vectors
 \begin{equation}
 \begin{bmatrix}
 \bx\\ \by
 \end{bmatrix}
 \sim \mathcal{N} (
 \begin{bmatrix}
 \mu_{\bx}\\ \mu_{\by}
 \end{bmatrix}
 , 
 \begin{bmatrix}
 A & C\\
 C^T & B
 \end{bmatrix}
 ),
\end{equation}
then the marginal distribution of $\bx$ and the conditional distribution of $\bx$ given $\by$ are
\begin{equation}
 \bx \sim \mathcal{N}(\mu_{\bx},A), \quad \textrm{and } \bx|\by \sim \mathcal{N}(\mu_{\bx} + CB^{-1}(\by - \mu_{\by}), A - CB^{-1}C^T).
\end{equation}
\end{lemma}
\begin{proof}
See, e.g. \citep{williams2006gaussian}, Appendix A.
\end{proof}

 \begin{lemma}\label{lem1} For any function $\intkernelvar \in L^2(\tilde \rho_T^L)$, we have that 
\begin{align} 
\Rhoxnorm{\rhsfo_{\intkernelvar}}^2 \leq \frac{N-1}{N} \rhotnorm{\intkernelvar}^2.
\end{align}
\end{lemma}

\begin{proof} See the proof of Proposition 16 in \citep{lu2021learning} by taking $K=1$. 
\end{proof}


\begin{lemma}[Lemma 8 in \citep{de2005learning}]\label{coninequality2}
Let $\mathcal{H}$ be a Hilbert space and $\xi$ be a random variable on $(Z,\rho)$ with values in $\mathcal{H}$. Suppose that, $\|\xi\|_{\mathcal{H}}\leq S < \infty$ almost surely. Let $z_m$ be i.i.d drawn from $\rho$. For any $0<\delta<1$, with confidence $1-\delta$,
$$\bigg\| \frac{1}{M}\sum_{m=1}^{M}(\xi(z_m)-\E(\xi))\bigg\| \leq \frac{4S\log(2/\delta)}{M}+\sqrt{\frac{2\E(\|\xi\|_{H}^2)\log(2/\delta)}{M}}.$$
\end{lemma}

The original version of Lemma \ref{coninequality2} is presented in \citep{yurinsky1995sums}.

\begin{theorem}[Hanson-Wright inequality \citep{rudelson2013hanson}] Let $X=(X_1,\cdots,X_n) \in \mathbb{R}^n$ be a random vector with independent components $X_i$ which satisfy 
$\E X_i=0$ and $\|X_i\|_{\psi_2} \leq S_0$, where $\|\cdot\|_{\psi_2}$ is the subGaussian norm. Let $A$ be an $n \times n$ matrix and $\|A\|_{HS}$ denote the Hilbert-Schmidt norm. Then, for every $\epsilon \geq 0$

$$\mathbb{P}\bigg\{\bigg\| X^TAX-\E X^TAX \bigg\| \geq \epsilon \bigg\} \leq 2\exp \bigg\{ -c \min \bigg\{ \frac{\epsilon^2}{S_0^4\|A\|_{HS}^2}, \frac{\epsilon}{S_0^2\|A\|}\bigg\} \bigg\},$$ where $c$ is an absolute positive constant. 
\label{HAnson}
\end{theorem}

\bibliography{main.bib}		

\end{document}